\documentclass[11pt]{article}
\usepackage[top=1in,bottom=1in,right=1in,left=1in]{geometry}
\usepackage{amssymb,amsmath,amsthm}
\usepackage{pgfplots}
\usepackage{pgfplotstable}
\usepackage{arydshln}
\usepackage{verbatim}
\usepackage{titlesec}
\usepackage{subfigure}
\usepackage[linesnumbered,ruled]{algorithm2e}
\usepackage{bm}
\usepackage{pdfpages}
\usepackage{enumerate} 
\usepackage{float}

\newcommand\blfootnote[1]{%
  \begingroup
  \renewcommand\thefootnote{}\footnote{#1}%
  \addtocounter{footnote}{-1}%
  \endgroup
}

\definecolor{forestgreen}{rgb}{0.13, 0.55, 0.13}
\usepackage[colorlinks, linkcolor=blue,citecolor=blue,urlcolor = black, bookmarks=true]{hyperref}
\usepackage[T1]{fontenc}
 
\usepackage[nameinlink, noabbrev, capitalize]{cleveref}
\usepackage{tikz}
\usepackage{times}

\pgfplotsset{compat=1.18}

\usepackage{booktabs} % For professional-looking tables
\usepackage{multirow} % For multi-row cells
\usepackage{array}

\crefname{equation}{}{}
\crefrangeformat{equation}{(#3#1#4) --~(#5#2#6)}
\crefmultiformat{equation}{(#2#1#3)}{, (#2#1#3)}{, (#2#1#3)}{, (#2#1#3)}
\crefname{lem}{Lemma}{Lemmas}
\crefname{section}{Section}{Sections}
\crefname{subsubsubsection}{Section}{Sections}
\crefname{rem}{Remark}{Remarks}
\crefname{figure}{Figure}{Figures}
\crefname{table}{Table}{Tables}
\Crefname{lem}{Lemma}{Lemmas}
\crefname{thm}{Theorem}{Theorems}
\Crefname{thm}{Theorem}{Theorems}
\Crefname{item}{Item}{Items}

\usepackage{commath}

\newtheorem{theorem}{Theorem}[section]
\newtheorem{lemma}[theorem]{Lemma}
\newtheorem{proposition}[theorem]{Proposition}
\newtheorem{corollary}[theorem]{Corollary}
\newtheorem{assumption}{Assumption}
\newtheorem{fact}[theorem]{Fact}
\newtheorem{claim}[theorem]{Claim}

\newtheorem{definition}[theorem]{Definition}

\theoremstyle{remark}
\newtheorem{remark}{Remark}

\title{Better Models and Algorithms for Learning Ising Models from Dynamics\blfootnote{Much of this work was completed when J.G. was at the MIT Department of Mathematics, supported by Vannevar Bush Faculty Fellowship ONR-N00014-20-1-2826 and Simons Investigator Award 622132. A.M. is supported in part by a Microsoft Trustworthy AI Grant, NSF-CCF 2430381, an ONR grant, and a David
and Lucile Packard Fellowship. E.M. is supported in part by Vannevar Bush Faculty Fellowship ONR-N00014-20-1-2826, Simons Investigator Award 622132, Simons-NSF DMS-2031883, and ONR MURI Grant N000142412742.}}

\author{Jason Gaitonde\\
Duke University\\
\texttt{jason.gaitonde@duke.edu}
\and Ankur Moitra\\
Massachusetts Institute of Technology\\
\texttt{moitra@mit.edu}
\and Elchanan Mossel\\
Massachusetts Institute of Technology\\
\texttt{elmos@mit.edu}}

\date{}

\begin{document}

\maketitle
\thispagestyle{empty}

\begin{abstract}
We study the problem of learning the structure and parameters of the Ising model, a fundamental model of high-dimensional data, when observing the evolution of an associated Markov chain. A recent line of work has studied the natural problem of learning when observing an evolution of the well-known Glauber dynamics [Bresler, Gamarnik, Shah, IEEE Trans. Inf. Theory 2018, Gaitonde, Mossel STOC 2024], which provides an arguably more realistic generative model than the classical i.i.d. setting. However, this prior work crucially assumes that all site update attempts are observed, \emph{even when this attempt does not change the configuration}: this strong observation model is seemingly essential for these approaches. While perhaps possible in restrictive contexts, this precludes applicability to most realistic settings where we can observe \emph{only} the stochastic evolution itself, a minimal and natural assumption for any process we might hope to learn from. However, designing algorithms that succeed in this more realistic setting has remained an open problem [Bresler, Gamarnik, Shah, IEEE Trans. Inf. Theory 2018, Gaitonde, Moitra, Mossel, STOC 2025].

In this work, we give the first algorithms that efficiently learn the Ising model in this much more natural observation model that only observes when the configuration changes. For Ising models with maximum degree $d$, our algorithm recovers the underlying dependency graph in time $\mathsf{poly}(d)\cdot n^2\log n$ and then the actual parameters in additional $\widetilde{O}(2^d n)$ time, which qualitatively matches the state-of-the-art even in the i.i.d. setting in a much weaker observation model. Our analysis holds more generally for a broader class of reversible, single-site Markov chains that also includes the popular Metropolis chain by leveraging more robust properties of reversible Markov chains.
\end{abstract}

\newpage
\thispagestyle{empty}
\tableofcontents
\newpage
\setcounter{page}{1}
\section{Introduction}

The Ising model is a fundamental model of high-dimensional distributions on $\{-1,1\}^n$ that encode latent, pairwise dependencies between the variables, with wide-ranging applications across computer science, economics, machine learning, statistical physics, and probability theory. More formally, the Ising model $\pi=\pi_{A,\bm{h}}$ is parametrized by a symmetric matrix $A\in \mathbb{R}^{n\times n}$ and external fields $\bm{h}\in\{-1,1\}^n$, and is defined via:
\begin{equation*}
    \pi(\bm{x})=\frac{\exp\left(\frac{1}{2}\bm{x}^TA\bm{x}+\bm{h}^T\bm{x}\right)}{Z_{A,\bm{h}}}.
\end{equation*}
The \emph{partition function} $Z_{A,\bm{h}}$ ensures $\pi$ forms a probability distribution and captures important statistical information about the model as a function of $(A,\bm{h})$.
The Ising model thus provides a succinct representation of local interactions between variables in terms of the matrix $A$, whose elements encode the preference of adjacent sites to having matching or opposing signs in isolation. In particular, the Ising model has a naturally associated conditional dependency graph $G$ on $[n]$ whose edges correspond to the nonzero entries of $A$, such that the conditional law of any site depends \emph{only} on the value of its graph-theoretic neighbors.

Due to its wide-ranging applications, the algorithmic problem of learning the underlying dependencies or parameters of the Ising model from data has been the subject of intense study spanning several decades, originally in the setting where one receives i.i.d. samples~\cite{chow_liu,ravikumar2010high,DBLP:journals/siamcomp/BreslerMS13,DBLP:conf/stoc/Bresler15,DBLP:conf/nips/WuSD19,unified} (see \Cref{sec:related_work} for more discussion). A burgeoning line of work, originally pioneered by Bresler, Gamarnik, and Shah~\cite{DBLP:journals/tit/BreslerGS18}, has aimed to obtain efficient learning algorithms that succeed when one instead observes the trajectory of the \emph{Glauber dynamics}~\cite{glauber1963time}, a well-studied Markov chain corresponding to $\pi$~\cite{DBLP:journals/tit/BreslerGS18,DBLP:conf/icml/DuttLVM21,unified,gaitonde2024bypassing}. Since statistical samples are generated by some natural process whether in physical or economics applications, developing learning algorithms compatible with direct trajectory data of this type is arguably much more realistic---the Glauber dynamics have been considered as an exogenous model of best response dynamics in the economics literature~\cite{young,kandori,BLUME1993387,DBLP:conf/focs/MontanariS09} and as a model of non-equilibrium dynamics in statistical physics. At an equally fundamental level, well-known hardness reductions \cite{sly2010computational,DBLP:conf/focs/SlyS12} imply that in the ``low-temperature'' regime where sites can have strong global correlations, \emph{no} efficient process can generate i.i.d. samples from $\pi$, let alone nature; therefore, developing algorithms that can learn from natural dynamics broadens the applicability of these results. 

However, a critical assumption made in all prior work on learning the Ising model from dynamics is that these learning algorithms have \emph{strong observability} of the dynamics. To state this assumption a bit more precisely, the  Glauber dynamics corresponds to the continuous-time Markov chain $(X^t)_{t=0}^T$ on $\{-1,1\}^n$ where sites decide to resample their current value by re-randomizing according to the conditional law of $\pi$ given the current configuration at stochastic update times: we will call these times ``update attempts,'' which cruicially may or may not result in the value changing. In all prior work on learning the Ising model from dynamics, it is assumed that \emph{all} update times are observed, whether or not these updates result in a transition that \emph{changes} the current configuration, or a \emph{non-transition} that does not. 

To understand this assumption, the strong observation model of all prior work amounts to knowing all times $t$ that each site attempted to update even when $X^t$ does not change. In specific, highly controlled settings, similar information might plausibly be obtained. For instance, consider an economics application where each site corresponds to an individual in a social network, and the configuration describes the adoption of one of two types of technologies; the Glauber dynamics then corresponds to a noisy best-response dynamic for how individuals choose between them~\cite{DBLP:conf/focs/MontanariS09}. In certain cases, observing some information about update attempts that do not result in a change may be possible by online platforms that can directly observe whether individuals accessed the product site (but did not buy it), or by conducting repeated surveys to determine whether a consumer is \emph{reconsidering} their choice even if their behavior does not end up changing. But even in these settings, it is quite challenging (e.g. how can one ensure accurate and timely responses?) or prohibitively expensive to observe even weak information of this type---in general settings, and particularly in physical applications where sites correspond to particles, the strong observability model becomes harder to justify.

By contrast, the weaker observation model where one \emph{only} views the evolution of $(X^t)_{t=0}^T$ is natural for \emph{any} stochastic process, whether in physical systems or networks or beyond. Developing efficient (or any nontrivial) learning algorithms for this setting would thus greatly broaden the applicability of learning from dynamics, but this has remained open since the original work of Bresler, Gamarnik, and Shah~\cite{DBLP:journals/tit/BreslerGS18}\footnote{As described in their work, ``learning
without this data is potentially much more challenging,
because in that case information is obtained only when a
spin flips sign, which may occur only in a small fraction
of the update."} a decade ago and was raised again by Gaitonde, Moitra, and Mossel~\cite{gaitonde2024bypassing}. As we explain in \Cref{sec:overview}, the existing methods from these prior works seem incapable of extending to this setting. The fundamental difficulty is that it is quite difficult to deduce information about $\pi$ solely from observed transitions that change the configuration---these observations form a highly correlated subset of the full sequence of update attempts that include non-transitions as well, biasing natural estimators to infer information about the parameters or structure. While the fact that the configuration does not change in some part of the trajectory could be quite statistically informative, it is very unclear how to algorithmically access this information since the event, or identities, of sites that attempted to transition but fail is completely unobservable; therefore, it is impossible to determine \emph{why} sites stay constant at any given time during the trajectory.

In this work, we overcome all of these challenges by providing the first learning algorithms that efficiently recover the structure and parameters of the Ising model from the direct trajectory of the Glauber dynamics. Our first main result shows that if the dependency graph $G$ has maximum degree at most $d$ and under standard non-degeneracy assumptions, there is an algorithm that recovers $G$ with high probability after observing $O_d(\log n)$ updates per site and runtime $O_d(n^2\log n)$. We then show that under these same conditions, once $G$ has been recovered, we can then learn the actual parameters in additional time $\widetilde{O}(2^d n)$. These guarantees have the same dimensional dependence (on $n$) as all state-of-the-art work on learning the Ising model in \emph{any} observation model, with qualitatively similar dependences in all other model parameters. In fact, all of our results hold more generally for a somewhat broader class of reversible Markov chains that includes the popular Metropolis dynamics~\cite{metropolis1953equation,10.1093/biomet/57.1.97} as well. Our results thus not only succeed in the most natural observation model, but also relaxes the distributional assumptions on the precise model of the stochastic evolution. Our work thus significantly advances the literature on learning graphical models by bridging algorithmic guarantees with realistic models of data \emph{and} observations. 

\subsection{Main Results}

We now state our results a bit more precisely (see \Cref{sec:prelims} for formal definitions). Recall that our goal is to recover the structure and parameters of an Ising model $\pi_{A,\bm{h}}$ over $\{-1,1\}^n$ from the \emph{direct} observation of a single-site Markov chain. While our results will hold more generally, we will state our model and results for the Glauber dynamics.

Formally, we work in continuous time, so that our observations are $(X^t)_{t=0}^T\in \{-1,1\}^n$, where $X^0$ is an arbitrary configuration. The process $X^t$ evolves as follows:
\begin{itemize}
    \item Each site $i\in [n]$ has an independent, associated exponential clock that rings at unit rate. The update times $\Pi_i\subseteq [0,T]$  thus form a Poisson point process.
    \item For each $t\in \Pi_i$, site $i$ applies the transition kernel $\mathsf{P}_i$ that updates the current configuration $X^t$ by setting
    \begin{equation}
    \label{eq:glauber_update}
        X^t_i =\begin{cases}
          +1\quad \text{with prob. $\frac{\exp(\sum_{k\neq i}A_{ik}X^t_k+h_i)}{\exp(\sum_{k\neq i}A_{ik}X^t_k+h_i)+\exp(-\sum_{k\neq i}A_{ik}X^t_k-h_i)}$}\\
          -1\quad\text{with prob. $\frac{\exp(-\sum_{k\neq i}A_{ik}X^t_k-h_i)}{\exp(\sum_{k\neq i}A_{ik}X^t_k+h_i)+\exp(-\sum_{k\neq i}A_{ik}X^t_k-h_i)}$}.
        \end{cases}
    \end{equation}
    The re-randomization of site $i$ is therefore according to the conditional distribution in $\pi$ given $X^t_{-i}$.
\end{itemize}
Crucially, our \emph{only} observations are the piecewise constant stochastic process $(X^t)_{t=0}^T$; we therefore only observe the subset of the update times in $\Pi_i$ that actually results in site $i$ flipping values.

Our first main result provides the first efficient structure learning algorithm from this weak observation model under standard non-degeneracy conditions (see \Cref{assumption:ising}). 

\begin{theorem}[\Cref{thm:alg_main} and \Cref{thm:main_match}, specialized]
\label{thm:structure_intro}
    Suppose that the Ising model $\pi$ has maximum degree $d$. Then there is a structure learning algorithm that, taking as input the observations $(X_t)^T_{t=0}$ for $T=O(\mathsf{poly}(d)\cdot \log(n))$ generated by the Glauber dynamics, correctly outputs the dependence graph of $\pi$ with high probability. The runtime of the algorithm is $O(T\cdot n^2)$.
\end{theorem}
The implicit constants are of the form $\mathsf{poly}(\exp(\lambda),1/\alpha)$, where the $\ell_1$ ``width" parameter $\lambda$ governs the biasedness of any site and $\alpha$ lower bounds the magnitude of any nonzero matrix entry in $A$ to ensure edges are statistically detectable. These dependencies are known to be necessary~\cite{DBLP:journals/tit/SanthanamW12}, and these will be qualitatively the same as in all prior literature in the i.i.d. and dynamical settings~\cite{DBLP:conf/focs/KlivansM17,DBLP:conf/nips/WuSD19,DBLP:journals/tit/BreslerGS18,unified}. Since sites update at unit rate in continuous time, the input can be specified by just the initial configuration as well as the $O(n\cdot T)$ times the configuration changes at some site with high probability. The dependence on $n$ thus matches the state-of-the-art in all prior observation models  (see \Cref{sec:related_work} for more information).

Once the dependence graph has been recovered, we then show the following parameter learning result:

\begin{theorem}[\Cref{thm:main_paramter} and \Cref{rmk:external}, informal]
\label{thm:parameter_intro}
    Let $\pi$ be an Ising model known dependence graph $G$ with maximum degree $d$. Then there is an algorithm that, given $\varepsilon>0$ and observations $(X_t)_{t=0}^T$ generated by the Glauber dynamics, computes $(\widehat{A},\widehat{\bm{h}})$ such that $\|A-\widehat{A}\|_{\infty},\|\bm{h}-\widehat{\bm{h}}\|_{\infty}\leq \varepsilon$ with high probability for $T=\widetilde{O}(2^d)\cdot \log(n)\cdot \mathsf{poly}(1/\varepsilon)$. The runtime of the algorithm is $n\cdot T=\widetilde{O}(2^dn\cdot \mathsf{poly}(1/\varepsilon))$. 
\end{theorem}
The implicit constants are again of the form $\mathsf{poly}(\exp(\lambda))$ since the minimum probability of all Glauber transitions is bounded by the inverse of this quantity. \Cref{thm:structure_intro} and \Cref{thm:parameter_intro} thus resolve the problem of efficiently learning the Ising model from arguably the most natural data and observation model.

Both of our algorithmic results will actually hold somewhat more generally for any Ising model satisfies the standard nondegeneracy conditions when observing \emph{any} single-site, reversible Markov chain that satisfies natural assumptions.\footnote{In fact, the parameter learning algorithm holds for \emph{any} nondegenerate reversible, single-site Markov chain.} These assumptions amount to ensuring that the transition probabilities for each site $i\in [n]$ are suitably nondegenerate and moreover, satisfy a certain consistency across configurations (see \Cref{assumption:mc} for the abstract formulation). This is satisfied by not only Glauber dynamics, but also natural forms of the well-studied Metropolis dynamics as well. We view this robustness as evidence of the potential of our algorithmic approach to potentially succeed quite broadly for most reasonable, single-site Markov chains; we leave further investigation of this as an exciting question for future work.

\subsection{Other Related Work}
\label{sec:related_work}

\noindent\textbf{Learning Graphical Models from Dynamics.} As described above, the problem of learning the Ising model, or more general Markov random fields, from the Glauber dynamics has been recently explored in a series of works---however, these prior works all require strong observability. The pioneering work of Bresler, Gamarnik, and Shah~\cite{DBLP:journals/tit/BreslerGS18} first considered the problem of structure learning in this model; their work introduces a natural localization idea to coarsely determine adjacencies, a high-level strategy that we will also adopt. Their result obtains a $O(\mathsf{poly}(d)\cdot n^2\log n)$ structure learning algorithm, with qualitatively similar model dependencies to \Cref{thm:structure_intro}, in the strong observability model. More recent work of Gaitonde and Mossel~\cite{unified} extends these results in the same model to further obtain \emph{parameter learning} guarantees via logistic regression~\cite{DBLP:conf/nips/WuSD19} with sample and runtime complexity on par with the most general results from the literature on i.i.d. learning~\cite{DBLP:conf/focs/KlivansM17,DBLP:conf/nips/WuSD19}. Dutt, Lokhov, Vuffray, and Misra~\cite{DBLP:conf/icml/DuttLVM21} show empirically that the complexity of learning in the i.i.d. and dynamical setting under strong observability are comparable. The recent work of Gaitonde, Moitra, and Mossel~\cite{gaitonde2024bypassing} shows that a combination of these techniques works more generally for learning higher-order Markov random fields, which in fact overcomes known hardness barriers for the i.i.d. setting, but again in the strong observation model. The recent work of  Jayakumar, Lokhov, Misra, and Vuffray~\cite{jayakumar} shows that existing methods for learning in the i.i.d. case easily extend to the setting where one is given i.i.d. samples from a ``strongly metastable state." A natural hope would be to reduce learning from dynamics to this setting by simply using sufficiently time-spaced samples. However, this approach appears highly challenging to implement in our general setting, and likely quantitatively suboptimal, since the rigorous theory of slow mixing Markov chains and metastability is quite nascent and it is unclear when one can hope to obtain such samples---see e.g. \cite{bovier2016metastability} for a textbook treatment as well as \cite{DBLP:conf/stoc/GheissariS22,gheissari2025rapid,DBLP:conf/focs/LiuMRRW24} for recent results on this topic.

\noindent\textbf{Learning the Ising Model from I.I.D. Samples.} The traditional task of learning the Ising model from i.i.d. samples has been studied for several decades, dating back to the seminal work of Chow-Liu~\cite{chow_liu}. While early work provided efficient algorithms in ``high-temperature'' models~\cite{ravikumar2010high,DBLP:journals/siamcomp/BreslerMS13}, the first efficient algorithm that succeeded even at ``low-temperature,'' albeit with doubly-exponential dependence on the degree, was obtained by Bresler~\cite{DBLP:conf/stoc/Bresler15}. These results were later generalized by Hamilton, Koehler, and Moitra~\cite{DBLP:conf/nips/HamiltonKM17}, and state-of-the art algorithms were obtained by Klivans and Meka~\cite{DBLP:conf/focs/KlivansM17} (see also \cite{DBLP:conf/nips/VuffrayMLC16,DBLP:conf/nips/WuSD19}). In the setting of \Cref{assumption:ising}, their result requires $p=\frac{\exp(O(\lambda))\log(n)}{\varepsilon^4}$ i.i.d. samples and time $O(n^2\log(n))$ to compute $\varepsilon$-accurate estimates for all entries of $A$. In particular, their algorithm has no explicit dependence on the degree $d$; it would be interesting to see whether such guarantees are possible in the dynamical setting for this observation model. We note that the minimax sample-complexity of learning \emph{any} interaction matrix that induces a close model in total variation was shown by Devroye, Mehrabian, and Reddad ~\cite{devroye2020minimax} to be $\Theta(n^2)$. Our work focuses on the more challenging parameter learning task since in many applications, the primary objective is to determine which sites directly interact and in what way.

While an exponential dependence in the $\ell_1$-width is known to be information-theoretically necessary \cite{DBLP:journals/tit/SanthanamW12}, several recent works have shown these worst-case bounds do not apply in many interesting cases. In particular, Koehler, Heckett, and Risteski \cite{DBLP:conf/iclr/KoehlerHR23} show a reduction from learning to functional inequalities that are known to hold when, e.g. the eigenvalues of $A$ lie in an interval of length $1$~\cite{eldan2022spectral,DBLP:conf/focs/ChenE22}, and the recent work of Koehler, Lee, and Vuong~\cite{klv} extends these results when there are a constant number of outlier eigenvalues. In cases where such functional inequalities are not known to hold, like spin glasses, recent work of Gaitonde and Mossel~\cite{unified} and Chandrasekharan and Klivans~\cite{DBLP:conf/stoc/ChandrasekaranK25} shows how to obtain learning guarantees by directly analyzing moments of the external fields under typical samples.

Several variants of this problem have been studied: among them are refined learning guarantees for tree-structured models~\cite{bresler_karzand,DBLP:conf/focs/Boix-AdseraBK21,DBLP:conf/colt/KandirosDDC23,DBLP:journals/siamcomp/BhattacharyyaGPTV23}, models with latent variables~\cite{DBLP:conf/approx/BogdanovMV08,DBLP:conf/stoc/BreslerKM19,DBLP:conf/nips/GoelKK20}, learning with limited samples~\cite{DBLP:conf/stoc/DaganDDK21}, and robust learning~\cite{DBLP:conf/colt/GoelKK19,DBLP:conf/nips/PrasadSBR20,DBLP:conf/colt/DiakonikolasKSS21}.

\noindent\textbf{Learning from Dynamics.} Our results fall into the broader theme of \emph{learning from dynamics}, wherein one attempts to infer structure from trajectory information. Quintessential examples of this paradigm are the problem of PAC learning from random walks~\cite{DBLP:journals/iandc/BartlettFH02,DBLP:journals/jcss/BshoutyMOS05}, learning linear dynamical systems~\cite{Kalman,DBLP:conf/colt/SimchowitzBR19,DBLP:conf/stoc/BakshiLMY23}, and learning network structure from cascades~\cite{DBLP:conf/kdd/AbrahaoCKP13,DBLP:conf/sigmetrics/NetrapalliS12,DBLP:journals/pomacs/HoffmannC19}, among others.

\noindent\textbf{Acknowledgments}. We thank Anirudh Sridhar for very helpful discussions on this problem, especially for explaining the higher-order error in \Cref{prop:event_probs}.

\section{Technical Overview}
\label{sec:overview}

In this section, we describe our algorithmic approach in the setting of Glauber dynamics; our results hold more generally, but the main intuition is given in this setting. We provide all notation and assumptions in \Cref{sec:prelims}. For a vector $\bm{x}\in \{-1,1\}^n$, we will write $\bm{x}^{i\mapsto \sigma}$ to denote the vector where $x_i$ is set to $\sigma\in \{-1,1\}$, and also write $\bm{x}^{\oplus S}$ for a multiset $S$ to denote that each variable in $S$ is flipped with multiplicity. For a graph $G$, we will also write $i\sim j$ to denote $(i,j)\in G$.

Recall that we let $(X^t)_{t=0}^T$ be the trajectory of Glauber dynamics.  Each site $i\in [n]$ has an associated set of update times $\Pi_i\subseteq [0,T]$ following a unit rate Poisson point process (i.e. with gaps distributed according to an exponential random variable with mean 1), and at each update time $t\in \Pi_i$, the site re-randomizes according to

\begin{align}
\nonumber
    \mathsf{P}_i(X^t,X^{t,i\mapsto +1})&=\Pr_{\pi}(X_i = 1\vert X_{-i}=X^t_{-i})\\
    \label{eq:sigmoid}
    &=\frac{\exp\left(\sum_{k\neq i}A_{ik}X^t_k+h_i\right)}{\exp\left(\sum_{k\neq i}A_{ik}X_k^t+h_i\right)+\exp\left(-\sum_{k\neq i}A_{ik}X^t_k-h_i\right)}.
\end{align}

\subsection{Prior Work and Challenges} Before providing an overview of our new algorithmic approach and analysis, we briefly discuss the key ideas from existing work on learning the Ising model in the dynamical setting. A key idea of all prior work is that \Cref{eq:sigmoid} encodes highly algorithmically useful structure: in particular, one can hope to design (approximately) \emph{unbiased} statistical estimators to identify structure or parameters that are tractable. In all prior work on learning the Ising model from the Glauber dynamics~\cite{DBLP:journals/tit/BreslerGS18,unified,gaitonde2024bypassing}, an essential observation is that while the Glauber dynamics has strong correlations over time, the conditional law in \Cref{eq:sigmoid} can nonetheless be algorithmically leveraged in this way when all updates are observed. In fact, \Cref{eq:sigmoid} is also true in the i.i.d. setting, so this has been exploited for state-of-the-art algorithms there as well ~\cite{DBLP:conf/focs/KlivansM17,DBLP:conf/nips/WuSD19}. However, it is only possible to leverage the form of \Cref{eq:sigmoid} when one \emph{knows} that $t\in \Pi_i$; this is only possible in the strong observation model since we cannot determine that $t\in \Pi_i$ unless $X^i_t$ changes values. But the conditional law in \Cref{eq:sigmoid} is \emph{trivially false} when one can only condition on the fact that $t\in \Pi_i$ resulted in a change, since $X^t_i$ changed by definition! 

To design learning algorithms that succeed only as a function of the direct trajectory $(X_t)_{t=0}^T$, we must therefore identify substantially new structure observable just from the trajectory to reveal information about $\pi$, which leads to multiple challenges:
\begin{itemize}
    \item The first major challenge is determining what (or \emph{when}) information \emph{is} revealed by site flips, since this is the only information we have access to; in light of the previous discussion, standard estimators can become trivially biased when they are computed only on flip events since the fact that this event occurs also entails \emph{not seeing} the flip before.

    \item Similarly, we need to account for the information that some site $i\in [n]$ does \emph{not} change values in an interval. But solely from the observation of $(X_t)_{t=0}^T$, this can occur for (at least) three reasons: (i) the site simply never attempts to update its value, which is \emph{unobserved}, (ii) the site attempted to update, but had a strong conditional preference not to change its value given the values of its neighbors, or (iii) the site attempted to update and preferred to change values at these updates, but stochasticity in the system prevents it from doing so. This multiplicity makes challenging the search for suitable statistics that can ``explain'' the observations as given just by $(X^t)_{t=0}^T$---each of these three events can become more or less likely depending on the precise scale of the interval. For instance, it will typically be the case that a site $i\in [n]$ does not change values in some small interval $I$ simply because no update attempt occurs, not because there is a conditional preference in $\pi$ to remain at the current value; on longer timescales, the reverse may be true, and of course the inherent noisiness of the process can also cause this at intermediate regimes.
\end{itemize}

We  show how to overcome these fundamental problems for learning from direct trajectory data, in fact even beyond the Glauber setting. In \Cref{sec:structure_overview}, we first describe how we use \emph{localized flip cycles} as a key preliminary step to identifying the dependency structure. We will argue that these statistics will find the \emph{dense edges} (c.f. \Cref{def:dense_edges}) of the Ising model. All remaining edges will form an isolated matching in the full dependence graph, which we further show can be efficiently detected afterwards. We then explain in \Cref{sec:parameter_overview} our parameter learning algorithm given the dependency structure.

\subsection{Structure Learning from Transitions}
\label{sec:structure_overview}
Our first task is to recover the dependency graph directly from the transitions of the Glauber dynamics; again, our results apply more generally, but we focus on Glauber for the exposition. To do so, we heavily exploit a natural idea from prior work on learning from dynamics~\cite{DBLP:journals/tit/BreslerGS18,gaitonde2024bypassing}: correlations between sites manifest in \emph{localized update attempts} in the stochastic evolution where $i$ and $j$ attempt to update $\Theta(1)$ times in close proximity to each other. From these observations, one can hope to formulate a suitable statistic that distinguishes neighbors and nonneighbors. But while these prior algorithms in the strong observability model can directly observe a localized sequence of \emph{update attempts}, we can only observe localized sequences of \emph{site changes}. Therefore, the success of this approach in this much more challenging setting relies on the following question:

\begin{quote}
\begin{center}
    \emph{Can localized site \underline{changes} reveal dependencies between $i$ and $j$ in the Ising model? If so, which ones, and are they efficiently computable?}
    \end{center}
\end{quote}

\noindent\textbf{Cycle Statistics.} Our first main result is that for the Ising model, there indeed does exist such a local statistic that \emph{almost} works that can be efficiently computed and only requires localized observations of flips.
To construct this statistic, we first prove the following result that gives the probability of observing flip sequences on small windows of size $\Theta(\varepsilon)$. We show that if the flip sequence is of bounded length, and  $\varepsilon\ll 1/d$ where $d$ is the maximum degree, then we can get convenient formulas for the probability of observing the flip sequence that become nearly \emph{unbiased} with the length of the window:

\begin{proposition}[\Cref{prop:event_probs}, informal]
\label{prop:event_overview}
    For any fixed time $t>0$, let $X^t$ denote the current configuration. Then for any sequence $\bm{\ell}$ in $\{i,j\}^m$ and for sufficiently small $\varepsilon\ll 1/m$, the probability that we observe the ordered sequence of flips of $i$ and $j$ in an interval of length $m\varepsilon$ is given by:
    \begin{equation*}
        \varepsilon^m\left(\prod_{k=1}^{m} \mathsf{P}_{\ell_k}(X^{t,\oplus \ell_1\ldots \ell_{k-1}},X^{t,\oplus \ell_1\ldots \ell_{k}})\pm O(md\varepsilon)\right).
    \end{equation*}
\end{proposition}
The significance of \Cref{prop:event_overview} is that if we take $\varepsilon>0$ to be a sufficiently small constant depending mildly on the sequence length and the maximum degree, then we can recover the associated product of flip rates up to an explicit normalization. The intuition behind \Cref{prop:event_overview} is that the most likely way for the flip sequence to occur on a small scale is that each site attempts to flip, and succeeds in doing so, in order exactly the right number of times---this consistutes the dominant term. While there can indeed be confounding by additional flip attempts by these site or by a neighbor that changes the relevant configuration, the key point is that these contribute \emph{higher-order} error terms since the event still only occurs with probability proportional to $\varepsilon^m$ and these confounding events incur a multiplicative $\varepsilon$. Since there are $O(d)$ such confounding events by \Cref{assumption:ising}, the constant one pays by the union bound remains bounded independently of $n$.

Unfortunately, \Cref{prop:event_overview} does not imply any sort of efficient unbiased estimator even for the product of these rates since we simply cannot obtain enough samples to obtain an accurate empirical estimate; the evolution of the Markov chain will generically be quite unlikely to visit \emph{any} configuration $\bm{x}\in \{-1,1\}^n$ many times. In fact, since the stationary probability of any configuration under $\pi$ is also exponentially small in $n$, we could not expect this to be the case even for i.i.d. samples. However, we can nonetheless leverage \Cref{prop:event_overview} to construct a flip-based statistic that will help determine adjacency. In particular, we employ a more complex version of the (dependent) method of moments recently developed by Gaitonde, Moitra, and Mossel~\cite{gaitonde2024bypassing} for the problem of learning higher-order Markov random fields from the Glauber dynamics under full observations. In their work, they wait for the \emph{update} pattern $iijii$ for the Glauber dynamics in order to construct a nonnegative statistic that is suitably lower bounded if the conditional law of $i$ before and after the $j$ update noticeably differ. Since we have a much more restrictive observation model, we instead form a suitable, nonnegative statistic based purely on flip sequences.

To motivate the construction, suppose that the Markov chain is at a configuration $X^t$ and we then observe one of the following two sequences of flips in a short interval of time: $iijj$ or $jiij$. If $i\not\sim j$, it is heuristically clear that both sequences occur with approximately equal probability (up to the higher-order error) since the site transitions are determined only by outside variables, not on each other by assumption.

Suppose now that $i$ and $j$ are indeed adjacent in the Ising model so that $\vert A_{ij}\vert>\alpha$ for some known $\alpha>0$ under \Cref{assumption:ising}. By construction, in the first sequence, both $i$ and $j$ transition along a single edge of the hypercube each \emph{only} when the other is in the initial configuration. In the second sequence, however, $i$ updates only when $j$ is \emph{flipped} from the initial configuration, while $j$ only flips while $i$ is in the initial state. Therefore, \Cref{prop:event_overview} implies that the probability of the first event should be (up to a negligible error term):
\begin{equation*}
    \varepsilon^4\cdot\mathsf{P}_i(X^t,X^{t,\oplus i})\mathsf{P}_i(X^{t,\oplus i},X^{t})\mathsf{P}_j(X^{t},X^{t,\oplus j})\mathsf{P}_j(X^{t,\oplus j},X^t),
\end{equation*}
while the probability of the $jiij$ event should be (approximately)
\begin{equation*}
    \varepsilon^4\cdot\mathsf{P}_j(X^t,X^{t,\oplus j})\mathsf{P}_i(X^{t,\oplus j},X^{t,\oplus \{i,j\}})\mathsf{P}_i(X^{t,\oplus \{i,j\}},X^{t,\oplus j})\mathsf{P}_j(X^{t,\oplus j},X^t).
\end{equation*}
In particular, the difference between them is (approximately)
\begin{equation}
\label{eq:event_diff_overview}
    \varepsilon^4\mathsf{P}_j(X^{t},X^{t,\oplus j})\mathsf{P}_j(X^{t,\oplus j},X^t)\left(\mathsf{P}_i(X^t,X^{t,\oplus i})\mathsf{P}_i(X^{t,\oplus i},X^{t})-\mathsf{P}_i(X^{t,\oplus j},X^{t,\oplus \{i,j\}})\mathsf{P}_i(X^{t,\oplus \{i,j\}},X^{t,\oplus j})\right).
\end{equation}

The identity \Cref{eq:event_diff_overview} is quite promising, since there is a difference in the moments \emph{precisely} when there is a difference in the product of transition rates of $i$ along a hypercube edge when $j$ is either in the initial configuration or flipped. Moreover, the explicit form of the Glauber transitions as in \Cref{eq:sigmoid} shows that 
\begin{align*}
    \mathsf{P}_i(X^t,X^{t,\oplus i})\mathsf{P}_i(X^{t,\oplus i},X^{t})&=\sigma\left(2\sum_{k\neq i}A_{ik}X^t_k+2h_i\right)\left(1-\sigma\left(2\sum_{k\neq i}A_{ik}X^t_k+2h_i\right)\right)\\
    \mathsf{P}_i(X^{t,\oplus j},X^{t,\oplus \{i,j\}})\mathsf{P}_i(X^{t,\oplus \{i,j\}},X^{t,\oplus j})&=\sigma\left(2\sum_{k\neq i,j}A_{ik}X^t_k+2h_i- 2A_{ij}X_j^t\right)\\
    &\cdot\left(1-\sigma\left(2\sum_{k\neq i,j}A_{ik}X^t_k+2h_i-2A_{ij}X^t_j\right)\right),
\end{align*}
where $\sigma(z)= 1/(1+\exp(-z))$ is the sigmoid function.

Let $p$ denote the $\sigma(\cdot)$ factor in the first identity and $p'$ denote the same factor in the second identity. In this case, it follows that if \Cref{eq:event_diff_overview} is small in absolute value, then we must have
\begin{equation}
\label{eq:moment_eq}
    p(1-p)\approx p'(1-p');
\end{equation}
however, note that the function $x\mapsto x(1-x)$ is two-to-one, and is easily shown to be stable in the sense that if \Cref{eq:moment_eq} holds, then it must be the case that either $p\approx p'$ or $p\approx 1-p'$ in a quantitative sense.

When is it possible for \Cref{eq:moment_eq} to hold? Certainly this will be the case when $A_{ij}=0$; this corresponds precisely to the case that $i\not\sim j$ which we already argued should have no difference by conditional independence of these sites. But if instead $\vert A_{ij}\vert>\alpha>0$ for some known constant $\alpha>0$, then we know that $p'\not\approx p$ since the argument to the sigmoid must have noticeably shifted. As a result, \Cref{eq:moment_eq} requires that instead, $p\approx 1-p'$ instead. However, suppose that the \emph{rest} of site $i$'s interactions are nontrivial, in the sense that the linear form
\begin{equation*}
    \ell(\bm{x})=\sum_{k\neq i,j} A_{ik} x_k
\end{equation*}
is not identically zero and has noticeable coefficients. One can easily show that $p\approx 1-p'$ forces $\bm{\ell}(X_{-i,j}^t)+h_i\approx 0$, which imposes an explicit constraint on $\ell(X_{-i,j}^t)$. But at this point, we can appeal to a structural result on anti-concentration of linear functions of sites (\Cref{cor:anti}) to assert that this explicit constraint must \emph{fail} to hold a noticeable fraction of the time we compute this difference of cycle statistics. Therefore, we can assert that the difference in probabilities in \Cref{eq:event_diff_overview} is \emph{noticeably far from zero a noticeable fraction of time}, so one can hope to elicit this information along the trajectory.

Even in the case that we can appeal to this anti-concentration argument, note that the \emph{sign} of this statistic will vary since the quantities depend on how close the conditional biases are to $0$ or $1$, which depends on the outside configuration. While site $i$ and $j$ always have the same conditional preference to match signs or flip signs from each other in the Ising model (depending on the sign of $A_{ij}$), this will not be reflected in these cycle statistics since we observe flips in both directions an equal number of times. To handle this, we employ a ``squaring" trick of Gaitonde, Moitra, and Mossel~\cite{gaitonde2024bypassing}. We can convert this absolute difference in probabilities in \Cref{eq:event_diff_overview} into a strictly positive difference by instead computing the following \emph{degree-8} statistic on an interval of length $8\varepsilon$:

\begin{equation*}
    Z^{i,j}_t=\mathbf{1}\{iijjiijj\}-2\cdot \mathbf{1}\{iijjjiij\}+\mathbf{1}\{jiijjiij\},
\end{equation*}
where we abuse notation to write out the events of a observing the written sequence of flips in the small interval beginning at $t$.

While this may look complicated, a simple application of \Cref{prop:event_overview} reveals that this can be viewed as the ``square'' of the previous cycle statistic obtained by composing the two different cycles in the right order; here, it is essential that each length-four cycle of flips returns to the initial configuration. In particular, one can show that up to the higher-order error term whose relative error can be driven to zero,
\begin{align*}
    \mathbb{E}[Z^{i,j}_t]&\approx \varepsilon^8\mathsf{P}^2_j(X^{t},X^{t,\oplus j})\mathsf{P}^2_j(X^{t,\oplus j},X^t)\\
    &\cdot\left(\mathsf{P}_i(X^t,X^{t,\oplus i})\mathsf{P}_i(X^{t,\oplus i},X^{t})-\mathsf{P}_i(X^{t,\oplus j},X^{t,\oplus \{i,j\}})\mathsf{P}_i(X^{t,\oplus \{i,j\}},X^{t,\oplus j})\right)^2.
\end{align*}
By the previous reasoning, this statistic will be noticeably positive with non-negligible probability so long as there exists some $k\neq j$ such that $i\sim k$ so that the linear form is not identically zero; this culminates in the following result:

\begin{theorem}[\Cref{cor:non_neighbors} and \Cref{cor:bulk_neighbors}, informal]
\label{thm:stat_diff}
    For any time $t>1$, and any conditional history $\mathcal{F}_{t-1}$, if $i\sim j$ \emph{and} there exists $k\neq j$ such that $i\sim k$, then it holds that
    \begin{equation*}
        \mathbb{E}[Z^{i,j}_t\vert \mathcal{F}_{t-1}]\geq \Omega(\varepsilon^8),
    \end{equation*}
    while if $i\not\sim j$
    \begin{equation*}
        \mathbb{E}[Z^{i,j}_t\vert \mathcal{F}_{t-1}]\leq O(d\varepsilon^9).
    \end{equation*}
    In particular, if $\varepsilon\ll 1/d$, then there is an explicit separation between them.
\end{theorem}

As such, the first step of our structure learning algorithm does the following:
\begin{itemize}
    \item For each $(i,j)$ pair, aggregate many samples of $Z^{i,j}_t$ with constant time spacing to ensure the samples are sufficiently independent to apply concentration for the aggregates \emph{and} anticoncentration bounds for linear forms under dynamics.
    \item If the empirical average is suitably positive, then output $i\sim j$.
\end{itemize}

By \Cref{thm:stat_diff}, we conclude that the algorithm finds a true subset of $G$ that contains all \emph{dense edges}, meaning those where either $i$ or $j$ has degree at least $2$ (c.f. \Cref{def:dense_edges}).
We crucially do \emph{not} deduce $i\not\sim j$ if the empirical average of $Z^{i,j}_t$ over many samples is small. As mentioned, this is because this can be simply false; one can reverse the above logic to deduce that these flip cycle statistics are all equal for the Glauber dynamics when $\pi(\bm{x})\propto\exp(x_ix_j)$, so no such statistic could possibly distinguish them. However, this reasoning also shows that that this can \emph{only} occur if $i\sim j$ is an \emph{isolated edge} in $G$, since it is not a dense edge by the above. In the next section, we describe our method to recover the remaining edges.

We briefly note that an essentially identical argument will hold for any Markov chain that satisfies similar abstract properties (c.f. \Cref{assumption:mc}); notably this holds for natural forms of the popular Metropolis dynamics as well. As a result, our main result on structure learning can be formulated for this more general setting in a unified way.

\noindent\textbf{Recovering Matchings and Independent Vertices}. To summarize the previous argument, the algorithmic guarantees of the degree-8 cycle statistics are that:
\begin{enumerate}
    \item If $i\sim j$ is a dense edge (\Cref{def:dense_edges}), then the cycle statistic finds that $i\sim j$.
    \item If $i\not\sim j$, then the cycle statistic will rightly not detect an adjacency between them.
\end{enumerate}
In particular, when the cycle statistic finds for a certain node $i$ that there are no adjacencies, the only possibilities are that (i) site $i$ is indeed independent (i.e. has no adjacencies) from all other sites, or (ii) there exists a \emph{unique} $j$ that the cycle statistic also finds no adjacencies for and $i\sim j$. The uniqueness in (ii) follows from the fact that all dense edges are found, and the rest of $G$ forms an isolated matching on the rest of the sites (c.f. \Cref{fact:dense_edge_structure}).

It then suffices to design an algorithm that, given the set $\mathcal{O}\subseteq [n]$ of nodes that the cycle test cannot find any adjacency for, can detect all adjacencies \emph{among} $\mathcal{O}$. Our structural result will imply that the induced dependence graph on $\mathcal{O}$ must form a \emph{matching}. From this point, we show that one can efficiently recover all of these edges using spin-spin correlations computed on a short timescale:
\begin{theorem}[\Cref{thm:main_match}, informal]
\label{thm:parameter_overview_thm}
Given the set $\mathcal{O}$ as above, there is an algorithm that computes time-averaged estimates of the spin-spin probabilities $\pi(x_i,x_j)$ for each $i,j\in \mathcal{O}$ over a trajectory of length $T=O(
\log(n))$ and correctly determines adjacencies in $\mathcal{O}$. The runtime is $O(Tn^2)$.
\end{theorem}

To establish \Cref{thm:parameter_overview_thm}, note that since all connected components in $\mathcal{O}$ have size at most $2$, the restricted Markov chain forms a product chain where each component \emph{rapidly mixes} under \Cref{assumption:ising} and \Cref{assumption:mc}---this can be easily seen using the method of canonical paths. As a result, we can apply a concentration bound  (\Cref{thm:lezaud}) given by Lezaud~\cite{lezaud2001chernoff} that asserts that the time average of any bounded function converges fast to its expectation under the stationary measure $\pi$. Because these bounds give Chernoff-type concentration, we can apply \Cref{thm:lezaud} to accurately compute all spin-spin correlations in $\mathcal{O}$ under $\pi$ by observing the trajectory for only $O(\log(n))$ time. At that point, we can easily show that neighbors in $\mathcal{O}$ will have inconsistent spin-spin correlations from product distributions, so we can correctly determine the remaining matching. This completes the algorithm for structure learning.

\subsection{Recovering Model Parameters Efficiently}
\label{sec:parameter_overview}
Once the dependency graph is recovered, the task of recovering the actual model parameters is still not trivial since we can only observe site changes rather than all updates. At a high-level, the approach is quite natural: since we know the at most $d$ neighbors of each site $i\in [n]$, we can first directly try to estimate each of the transition probabilities $\mathsf{P}_i(\bm{x},\bm{x}^{\oplus i})$ up to suitable accuracy; since the transitions only depend on $\bm{x}_{i\cup \mathcal{N}(i)}$, we can restrict to the at most $2^{d+1}$ relevant configurations and ignore the outside coordinates in estimating these transition rates. If we can obtain these estimates, then one can show that by the reversibility of Glauber dynamics (c.f. \Cref{def:reversibility})
\begin{equation}
\label{eq:parameter}
    \exp\left(4A_{ij}\right) = \frac{\mathsf{P}_i(\bm{x}^{i\mapsto -1,j\mapsto -1},\bm{x}^{i\mapsto +1,j\mapsto -1})/\mathsf{P}_i(\bm{x}^{i\mapsto +1,j\mapsto -1},\bm{x}^{i\mapsto -1,j\mapsto -1}) }{\mathsf{P}_i(\bm{x}^{i\mapsto -1,j\mapsto +1},\bm{x}^{i\mapsto +1,j\mapsto +1})/\mathsf{P}_i(\bm{x}^{i\mapsto +1,j\mapsto +1},\bm{x}^{i\mapsto -1,j\mapsto +1})}.
\end{equation}
While this is the approach that we will end up taking, there are two subtle difficulties in implementing this approach accurately and efficiently:
\begin{itemize}
    \item First, how do we obtain unbiased estimators of $\mathsf{P}_i(\bm{x},\bm{x}^{\oplus i})$ for a given value of $\bm{x}\in \{-1,1\}^{d+1}$? We still have the original issue that we cannot observe failed transitions, so naive estimators will be highly biased.\label{item:issue_1}
    \item The above estimator relies on having sufficient samples to estimate all of the ratios for just \emph{some} outside configuration $\bm{x}_{-i,j}$ but with \emph{all} settings of $x_i$ and $x_j$ in $\{-1,1\}$. Since the evolution of the Markov chain is quite complex, it is not clear how long it will take to find such a point with sufficiently many samples for all four relevant configurations.\label{item:issue_2}
\end{itemize}

For the first item, we proceed by using the same localization trick as when computing cycle statistics as in \Cref{prop:event_overview}: each time we are at an outside configuration $\bm{x}_{-i}$, we can compute the fraction of times that $x_i$ flips in a small $\varepsilon>0$ window. The same analysis will show that if $\varepsilon>0$ is sufficiently small (say $\varepsilon\ll 1/d$), the bias of the (appropriately normalized) estimator can be driven to zero. While this scale determines the variance of this empirical estimator, the dependence will be polynomial in all the relevant parameters so long as we obtain enough observations for this $\bm{x}_{-i}$. The crucial difference now is that there are only $2^{d+1}$ possible configurations to consider rather than $2^n$ as was the case before.

The second item is somewhat more subtle to deal with to get better algorithmic dependencies. One approach would be to simply pay a worst case bound to try to collect accurate rates for \emph{all} configurations $\bm{x}\in \{-1,1\}^{d+1}$. However, the probability of observing a fixed configuration $\bm{x}$ even in the i.i.d. setting can be as low as $\exp(-\Omega(\lambda d))$; in the dynamical setting, this kind of behavior can easily persist. Estimating \emph{all} rates would therefore require at least on the order of $\exp(\Omega(\lambda d))$ samples at best. 

Since this heuristic approach is already somewhat tricky to implement properly, we can instead argue more carefully as follows to replace the exponential dependence on $\lambda d$ with a much sharper dependence on $d$. Our main result for parameter learning is the following:

\begin{theorem}[\Cref{cor:emp_accurate}, informal]
\label{thm:parameter_rates_overview}
    Given the dependence graph of $G$, for any $i\in [n]$ and $j\in \mathcal{N}(i)$, there is an algorithm that observes the trajectory for time $T=\widetilde{O}(2^d\log(1/\delta))$ and obtains accurate estimates of each $\mathsf{P}_i(\bm{z},\bm{z}^{\oplus i})$ along a dimension $2$ subcube of $\{-1,1\}^{\mathcal{N}(i)\cup \{i\}}$ that has all configurations for each setting of $x_i,x_j$, with probability at least $1-\delta$.
\end{theorem}

In words, \Cref{thm:parameter_rates_overview} asserts that after just $T=\widetilde{O}(2^d)$ time, we can obtain accurate estimates of each of the quantities on the right hand side of \Cref{eq:parameter} for some subcube; any subcube with sufficient samples will suffice by concentration. Doing this for each $i\in [n]$ and $j\in \mathcal{N}(i)$ yields our overall runtime of $T=\widetilde{O}(2^dn)$. To further recover the external fields $\bm{h}$, we can employ similar reasoning using reversibility so long as we have estimates of each $A_{ij}$ to accuracy $\ll 1/d$ to control the error. Note that $\Omega(2^d)$ samples would already be required to observe sufficient samples for each point in any subcube as above even upon getting i.i.d. samples from the uniform distribution on $\{-1,1\}^d$ by standard coupon collector arguments, so the sample complexity in \Cref{thm:parameter_rates_overview} is essentially optimal for this approach.

To show \Cref{thm:parameter_rates_overview}, we will collect the above samples for outside configurations spaced out by a fixed constant, say $2$; this will ensure there is at least some weak independence between consecutive samples. Conditioned on observing a configuration at some timestep, the law of the configuration at the next timestep is somewhat complex. However, we show (c.f. \Cref{prop:local_stable}) that the distribution on the next configuration can be lower-bounded by a (sub)-distribution with constant probability mass that satisfies the following guarantee: for any setting of the the outside configuration $\bm{y}\in \{-1,1\}^{\mathcal{N}(i)\setminus \{i,j\}}$, each of the conditional (sub)-probabilities of the four ways to set $i,j$ are at least a constant. 

Therefore, while the law of $\bm{y}\in \{-1,1\}^{\mathcal{N}(i)\setminus \{i,j\}}$ may itself be complicated and vary drastically between timesteps after conditioning on the previous timestep, we can deduce by an averaging argument that after at most $T=\widetilde{O}(2^d)$ timesteps, there surely exists a sub-cube $\bm{y}\in \{-1,1\}^{\mathcal{N}(i)\setminus \{i,j\}}$ such that the pathwise sum of \emph{conditional probabilities} of each of all four ways to set $i,j$ is fairly large. By employing an appropriate version of Freedman's pathwise martingale inequality, we can ensure that the error of \emph{all} estimators at all sites we obtain after $T$ timesteps are accurate at squareroot scale of the pathwise sum of conditional probabilities. As a result, we ensure with probability $1$ that there \emph{exists} a configuration $\bm{x}\in \{-1,1\}^{\mathcal{N}(i)\setminus \{i,j\}}$ such that we have many samples of flip events for each setting of $\{i,j\}$, and moreover, these estimates will be accurate with high-probability. Here, the use of Freedman's inequality appears essential to obtaining the $\widetilde{O}(2^dn)$ overall runtime when setting parameters appropriately.

We note that this entire argument is \emph{general}, and relies only on reversiblility and single-site updates of the Markov chain to exploit \Cref{eq:parameter}, as well as obvious nondegeneracy conditions ensuring the Markov chain moves nontrivially. Therefore, the algorithmic guarantees for parameter learning (assuming the dependency graph is know and has maximum degree $d$) extend broadly even with minimal assumptions on the precise generative process.

\section{Preliminaries}
\label{sec:prelims}

\noindent\textbf{Notation}.
We use capital letters $X,Y,\ldots$ to denote random variables and bold font $\bm{x},\bm{y},\ldots$ to denote non-random vectors. For a multiset $S$, we write $\bm{x}^{\oplus S}$ to denote the vector $\bm{x}\in \{-1,1\}^n$ with the bits in $S$ flipped with multiplicity, i.e. if $x_i^{\oplus S} = (-1)^{m(i,S)}x_i$ where $m(i,S)$ denotes the multiplicity of $i$ in $S$. We also write $\bm{x}^{i\mapsto a}$ to denote the vector $\bm{x}$ where the $i$th value is reset to $a$.

We will use the notation $\mathcal{A},\mathcal{B},\ldots$ to denote events. We write $\mathcal{E}^c$ to denote the complement of the event $\mathcal{E}$. Given a subset of indices $S\subseteq [n]$, we use the subscript $-S$ to denote the restriction of a vector to the coordinates outside $S$. We will occasionally write $-i$ or $-i,j$ in place of $-\{i\}$ and $-\{i,j\}$ for notational ease.

\subsection{Ising Models}
We consider Ising models parameterized by a symmetric matrix $A\in \mathbb{R}^{n\times n}$ and external fields $\bm{h}\in \mathbb{R}^n$. Then the corresponding Ising model is the distribution $\pi=\pi_{A,\bm{h}}$ given by 
\begin{equation*}
    \pi(\bm{x})=\frac{\exp\left(\frac{1}{2}\bm{x}^TA\bm{x} + \bm{h}^T\bm{x}\right)}{Z},
\end{equation*}
where $Z$ is the partition function, or normalizing constant that ensures $\pi$ is a probability distribution. 

We will write $i\sim j$ to denote $\vert A_{i,j}\vert>0$; that is, $i$ and $j$ directly interact with each other in the potential. Then the \textbf{dependence graph} of $\mu_{A,\bm{h}}$ is the graph $G=([n],E)$ with edge set
\begin{equation*}
    E = \{(i,j): \vert A_{ij}\vert>0\}.
\end{equation*}

We make the following definition:
\begin{definition}
\label{def:dense_edges}
    Let $G=(V,E)$ denote a graph. Then the set $H$ of \textbf{dense edges} of $G$ is defined to be the set of edges that lie in connected components with average degree strictly greater than $1$.
\end{definition}

The following fact is immediate from \Cref{def:dense_edges}.
\begin{fact}
\label{fact:dense_edge_structure}
    For any graph $G=(V,E)$, let $H$ denote the dense edges. Then it holds that $\mathcal{O}=E\setminus H$ is a matching. Moreover, there are no edges in $E$ between a vertex in $\mathcal{O}$ and a vertex adjacent to an edge in $H$.

    In particular, vertices with no edge in $H$ are either isolated in $E$ or belong to an isolated edge in $E$ with no neighbor in $H$.
\end{fact}

We will make the following non-degeneracy assumptions on the parameters of the underlying model:
\begin{assumption}
\label{assumption:ising}
    The Ising model $\pi=\pi_{A,\bm{h}}$ satisfies the following conditions for known parameters $d,\lambda,\alpha>0$:
    \begin{enumerate}
        \item (Bounded Degree) For each $i\in [n]$, $\|A_{i,:}\|_0\leq d$. That is, each site has at most $d$ neighbors in the dependency graph.
        \item (Bounded Width) For each $i\in [n]$, $\|A_{i,:}\|_1+\vert h_i\vert:=\sum_{k\neq i} \vert A_{ik}\vert+\vert h_i\vert\leq \lambda$.
        \item (Neighbor Nondegeneracy) For each $i,j$ such that $\vert A_{i,j}\vert\neq 0$, it holds that $\vert A_{i,j}\vert\geq \alpha$.
    \end{enumerate}
\end{assumption}

\subsection{Continuous-Time Single-Site Markov Chains}
Throughout this paper, we will consider observations of the trajectory of single-site Markov chains on the state space $\{-1,1\}^n$ that are \emph{reversible} with respect to $\pi$. In particular, each site $i\in [n]$ has an associated, independent Poissonian clock with unit rate\footnote{This assumption can be made essentially without loss of generality with little algorithmic modification. For homogeneous rates, we may rescale time so that the fastest rate is $1$. In that case, the Markov chain is equivalent to rescaling the transition kernels of the other sites to induce the same law up to this universal scaling of time.} where the transition kernel $\mathsf{P}_i$ is applied for the site. More formally, the set of update times $\Pi_i\subseteq \mathbb{R}_+$ follows an independent Poisson point process with rate $1$; equivalently, the difference in subsequent update times in $\Pi_i$ has an independent exponential law with mean $1$. 

In more detail, this process is such that for any interval $I\subseteq \mathbb{R}_{\geq 0}$,
\begin{equation}
\label{eq:prob_no_update}
    \Pr(\Pi_i\cap I= \emptyset)=\exp(-\vert I\vert),
\end{equation}
where $\vert I\vert$ is the length of $I$. For an interval $I\subseteq \mathbb{R}$, we write $\Pi_i(I)=\Pi_i\cap I$ for the sequence of update times of node $i$ in $I$. These sets are independent across any sites as well as between nonintersecting sets. For convenience, we write $\Pi_i(t_1,t_2)$ as shorthand for $\Pi_i([t_1,t_2])$ and $\Pi_i(t)$ as shorthand for $\Pi_i([0,t])$.

We require the following simple estimates on the probabilities that a subset of variables is or is not updated in a given interval, which are immediate from the definition and independence/union bounding:

\begin{lemma}
\label{lem:update_bounds}
Let $S\subseteq [n]$ be a subset of size $\ell$. Fix an interval $I\subseteq \mathbb{R}_{\geq 0}$ of length $T$ and let $U_I$ denote the set of sites that are ever chosen for updating in $I$. Then it holds that:
\begin{gather*}
    \Pr(S\subseteq U_I)=(1-\exp(-T))^{\ell}\geq 1-\ell\exp(-T),\\
    \Pr(S\cap U_I=\emptyset)= \exp(-T\ell).
\end{gather*}
\end{lemma}

We will assume that all single-site transition kernels $\mathsf{P}_i$ satisfy the following common condition from the theory of Markov chains:

\begin{definition}
\label{def:reversibility}
Let $\mathsf{P}_i(\bm{x},\cdot)$ be the transition kernels associated to each $\bm{x}\in \{-1,1\}^n$ as above. Then the single-site Markov chain is \textbf{reversible} with respect to $\pi$ if the transition kernels satisfy the detailed balance equations:
\begin{equation*}
    \pi(\bm{x})\mathsf{P}_i(\bm{x},\bm{x}^{\oplus i}) = \mathsf{P}_i(\bm{x}^{\oplus i},\bm{x})\pi(\bm{x}^{\oplus i}).
\end{equation*}
\end{definition}

Equivalently, the associated Markov operators $\mathsf{P}^t$ on functions $f:\{-1,1\}^n\to \mathbb{R}$ form a semigroup that is given by
\begin{equation*}
    \mathsf{P}^tf(\bm{x}) := \mathbb{E}_{X^t}[f(X^t)\vert X^0 = \bm{x}] = f + t\sum_{i=1}^n (\mathsf{P}_if -f) + O(t^2),
\end{equation*}
where the operator $\mathsf{P}_i$ acts on functions in the natural way by resampling the $i$th coordinate according to the distribution given by $\mathsf{P}_i(\bm{x})$. This is equivalent to the \emph{generator} $\mathcal{L}$ of the Markov chain being given by
\begin{equation*}
    \mathcal{L}f :=\lim_{t\to 0} \frac{\mathsf{P}^tf-f}{t}= \sum_{i=1}^n (\mathsf{P}_i-I)f:=\sum_{i=1}^n  \mathcal{L}_if.
\end{equation*}
The transition probabilities after running the evolution for $t$ units of time are then given by the matrix $H_t$ where
\begin{equation*}
    H_t=\exp(t\mathcal{L}).
\end{equation*}

We will require the following quantitative form of the Chernoff bound for Markov chains as given by Lezaud~\cite{lezaud2001chernoff}:
\begin{theorem}[Theorem 1.1 of \cite{lezaud2001chernoff}, Equation (1.2)]
\label{thm:lezaud}
    There is an absolute constant $C>0$ such that the following holds. Let $f:\{-1,1\}^n\to \mathbb{R}$ be any function such that $\vert f(\bm{x})\vert\leq a$. Suppose $\mathcal{L}=\mathsf{P}-I$ is the generator of a reversible Markov chain with respect to $\pi$ with spectral gap $\rho>0$. Then for any starting configuration $X^0\in \{-1,1\}^n$ of the Markov chain evolving with generator $\mathcal{L}$, any $\varepsilon>0$ and any $T>0$,
    \begin{equation*}
        \Pr\left(\left\vert\frac{1}{T}\int_0^T f(X^t)\mathrm{d}t - \mathbb{E}_{\pi}[f] \right\vert>\varepsilon\right)\leq \frac{2}{\pi_{\mathsf{min}}}\exp\left(\frac{-\rho T\varepsilon^2}{Ca^2}\right),
    \end{equation*}
    where $\pi_{\mathsf{min}}=\min_{\bm{x}\in \{-1,1\}^n}\pi(\bm{x})$.
\end{theorem}

\begin{corollary}
\label{cor:leuzard_all}
    Under the conditions of \Cref{thm:lezaud}, suppose that $f_1,\ldots,f_m:\{-1,1\}^n\to \mathbb{R}$ are functions bounded by $a$ in absolute value. Then there is an absolute constant $C>0$ such that for any $\varepsilon>0$ and $\delta<1$,
    \begin{equation*}
        T\geq \frac{C\log(m/\delta \pi_{\min})}{\rho a^2\varepsilon^2},
    \end{equation*}
    then with probability at least $1-\delta$, it holds simultaneously for all $k\leq m$ that
    \begin{equation*}
        \left\vert\frac{1}{T}\int_0^T f_k(X^t)\mathrm{d}t - \mathbb{E}_{\pi}[f_k]\right\vert \leq \varepsilon.
    \end{equation*}
\end{corollary}

To later apply this result, we will use the following fact that can be derived by a direct application of the method of canonical paths (e.g. Corollary 13.21 of ~\cite{levin2017markov}):
\begin{fact}
\label{fact:spectral_gap}
    Let $\pi$ be a distribution on $\{-1,1\}^n$ for some $n=O(1)$ such that $\min_{\bm{x}}\pi(\bm{x})\geq \zeta$. Suppose $\mathsf{P}$ is a reversible and irreducible Markov transition kernel with respect to $\pi$ such that each nonzero transition has probability at least $\gamma$. Then the spectral gap of $\mathcal{L}$ is at least $c/\gamma\zeta$.
\end{fact}

\subsection{Consistent and Stable Chains}

In this section, we formalize the class of Markov chains that our algorithms works for. As we show, this general formulation will capture both the Glauber dynamics and the popular Metropolis dynamics.

The first definition is that the associated site transitions depend only on the \emph{probability ratio} of the transition. In the Ising model, reversibility implies that the transitions depend only on the local field.

\begin{definition}[Site-Consistency]
\label{def:site_consistent}
    A single-site, reversible Markov chain with respect to $\pi$  is \textbf{site-consistent} if for each $i\in [n]$, there exists a monotone nondecreasing function $f_i:\mathbb{R}_+\to [0,1]$ such that for all $\bm{x}\in \{-1,1\}^n$ and $i\in [n]$,
    \begin{equation*}
        \mathsf{P}_i(\bm{x}^{i\mapsto -1},\bm{x}^{i \mapsto +1}) = f_i\left(\frac{\pi(\bm{x}^{i \mapsto +1})}{\pi(\bm{x}^{i \mapsto -1})}\right).
    \end{equation*}
\end{definition}

If site-consistency \emph{fails}, then a site $i$ can have transition probabilities that may be of vastly different scales along different $i$ edges of $\{-1,1\}^n$. In this case, learning seems very difficult since the parameters only determine the \emph{relative} probabilities of transitioning along the two directions of any \emph{single edge} by reversibility, but these transitions can otherwise be \emph{arbitrary} for different edges. Since Markov chains are unlikely to traverse any edge more than $O(1)$ times on reasonable scales, it appears very difficult to learn without any consistency for different hypercube edges.

\begin{corollary}
\label{cor:site}
    Suppose that a single-site, reversible Markov chain is site-consistent as in \Cref{def:site_consistent}. Then
    \begin{equation}
        \mathsf{P}_i(\bm{x}^{i\mapsto +1},\bm{x}^{i \mapsto -1}) = \frac{\pi(\bm{x}^{i\mapsto -1})}{\pi(\bm{x}^{i\mapsto +1})}f_i\left(\frac{\pi(\bm{x}^{i \mapsto +1})}{\pi(\bm{x}^{i \mapsto -1})}\right),
    \end{equation}
    and therefore
    \begin{equation}
        \mathsf{P}_i(\bm{x}^{i\mapsto -1},\bm{x}^{i \mapsto +1})\mathsf{P}_i(\bm{x}^{i\mapsto +1},\bm{x}^{i \mapsto -1}) = \frac{\pi(\bm{x}^{i\mapsto -1})}{\pi(\bm{x}^{i\mapsto +1})}f_i^2\left(\frac{\pi(\bm{x}^{i \mapsto +1})}{\pi(\bm{x}^{i \mapsto -1})}\right):=g_i\left(\frac{\pi(\bm{x}^{i \mapsto +1})}{\pi(\bm{x}^{i \mapsto -1})}\right).
    \end{equation}
\end{corollary}
\begin{proof}
    By reversibility and \Cref{def:site_consistent}, we can express
    \begin{equation*}
        \mathsf{P}_i(\bm{x}^{i\mapsto +1},\bm{x}^{i \mapsto -1}) = \frac{\pi(\bm{x}^{i\mapsto -1})}{\pi(\bm{x}^{i\mapsto +1})}\mathsf{P}_i(\bm{x}^{i\mapsto -1},\bm{x}^{i \mapsto +1})=\frac{\pi(\bm{x}^{i\mapsto -1})}{\pi(\bm{x}^{i\mapsto +1})}f_i\left(\frac{\pi(\bm{x}^{i \mapsto +1})}{\pi(\bm{x}^{i \mapsto -1})}\right).
    \end{equation*}
    The second identity is an immediate consequence by multiplication.
\end{proof}

Our next definition appears rather technical, but as we will see, can be readily established for both Glauber and Metropolis. The intuition behind it is that for most chains, the product of transition probabilities across edges in each direction should be monotone increasing as a function of the energy ratio in $[0,a]$ and then decreasing on $[a,\infty)$. This implies that each level set has size at most 2, and \Cref{def:stable} asserts that if two points with fixed ratio lie in the same level set, then any other two points with the same fixed ratio that are also nearly in the same level set must be close by.

\begin{definition}[Stability of Transitions]
\label{def:stable}
    A site-consistent Markov chain with respect to $\pi$ is \textbf{$(\lambda,\alpha_0,\delta_0,\eta)$-stable} for $0<\alpha_0\leq \lambda$ and $\eta:[0,1]\to \mathbb{R}_+$ a monotone increasing function such that $\eta(0)=0$ if for all $i\in [n]$, the following holds. Define 
    \begin{equation*}
        g_i(z) := f_i(z)^2/z.
    \end{equation*} 
    Then for any $\alpha\geq \alpha_0$, there is a unique $z^*(\alpha)>0$ that satisfies the equation
    \begin{equation*}
        g_i(z^*)=g_i(\exp(\alpha)z^*).
    \end{equation*}

    Moreover, for any sufficiently small $\delta\leq \delta_0$, if $z\in [\exp(-2\lambda),\exp(2\lambda)]$ satisfies
    \begin{equation*}
        \left\vert g_i(z)-g_i(\exp(\alpha)z)\right\vert\leq \delta,
    \end{equation*}
    then
    \begin{equation*}
        \vert z^*-z\vert\leq \eta(\delta).
    \end{equation*}
\end{definition}

We note that existence and uniqueness is implied by the natural condition that $g(0)=0$ and that $g$ is increasing on $[0,a]$ and decreasing on $[a,\infty)$ for some $a>0$ (see \Cref{fact:exist_unique}).

The next condition amounts to asserting that the likelihood a site $i$ updates to a fixed spin $\sigma$ in two configurations cannot differ by much more than the worst case difference in the local interactions of site $i$ between configurations.

\begin{definition}[Boundedness]
    \label{def:bounded}

    A Markov chain with respect to $\pi$ is \textbf{$\gamma$-bounded} for some constant $\gamma\geq 0$ if for each $i\in [n]$ and  all states $\bm{x},\bm{y}$ such that $x_i=y_i$, and $\sigma\in \{\pm 1\}$,
    \begin{equation*}
        \frac{\mathsf{P}_i(\bm{x},\bm{x}^{i\mapsto\sigma})}{\mathsf{P}_i(\bm{y},\bm{y}^{i\mapsto\sigma})}\leq \exp\left(\gamma \sum_{k\in S} \vert A_{ik}\vert\right),
    \end{equation*}
    where $S$ is the set of coordinates in $[n]$ where $\bm{x}$ and $\bm{y}$ differ.
\end{definition}

We may state our final assumptions on which Markov chains our results will hold for.
\begin{assumption}
\label{assumption:mc}
    The evolution of the observed single-site, reversible Markov chain satisfies the following conditions:
    \begin{enumerate}
        \item The Poissonian update times have rate $1$,
        \item The Markov chain is site-consistent and $(\lambda,\alpha_0,\delta_0,\eta)$ stable as in \Cref{def:site_consistent} and \Cref{def:stable} where $\alpha_0$ and $\lambda$ are the same as in \Cref{assumption:ising}.
        \item The Markov chain is $\gamma$-bounded for some constant $\gamma \geq 0$.
        \item There exists $\kappa>0$ such that for all $\bm{x}$ and $i\in [n]$,
        \begin{equation*}
            \mathsf{P}_i(\bm{x},\bm{x}^{\oplus i})\geq \kappa.
        \end{equation*}
    \end{enumerate}
\end{assumption}

\subsubsection{Glauber Dynamics}
\label{sec:glauber}

\begin{definition}[Glauber Dynamics]
\label{def:glauber}
    The \textbf{Glauber dynamics} are given by the transition kernels:
    \begin{equation*}
        \mathsf{P}_i(\bm{x},\bm{x}) = \frac{\pi(\bm{x})}{\pi(\bm{x})+\pi(\bm{x}^{\oplus i})},\quad \mathsf{P}_i(\bm{x},\bm{x}^{\oplus i}) = \frac{\pi(\bm{x}^{\oplus i})}{\pi(\bm{x})+\pi(\bm{x}^{\oplus i})}.
    \end{equation*}
\end{definition}

In words, the Glauber dynamics resamples the chosen site according to the conditional distribution of the site given the other coordinates in the base measure $\pi$. More explicitly, let $\sigma(z):=\frac{1}{1+\exp(-z)}$ denote the sigmoid function. Given any $i\in [n]$ and configuration $\bm{x}_{-i}\in \{-1,1\}^{n-1}$, the Glauber update at site $i$ given that $X_{-i}^t=\bm{x}_{-i}$ and $t\in \Pi_i$ has the conditional law:
\begin{equation}
\label{eq:glauber}
    \Pr(X^t_i=1\vert X^t_{-i}=\bm{x}_{-i},t\in \Pi_i)=\sigma\left(2\sum_{k\neq i} A_{ik} x_k + 2h_i\right).
\end{equation}

We require the following lower bounds on the strict monotonicity of $\sigma$.

\begin{fact}[\cite{DBLP:conf/focs/KlivansM17}]
\label{fact:km_sigmoid}
    For any $x,y\in \mathbb{R}$, $
        \vert \sigma(x)-\sigma(y)\vert\geq \exp(-\vert x\vert-3)\min\{1,\vert x-y\vert\}.
$
\end{fact}

We now state the following guarantees that verify that the Glauber dynamics indeed satisfy \Cref{assumption:mc}: we defer the details to \Cref{sec:app_gd}.
\begin{proposition}
\label{prop:gd_assumption}
Under \Cref{assumption:ising}, the Glauber dynamics satisfy the conditions of \Cref{assumption:mc} with:
\begin{gather*}
    \delta_0=c\min\{\alpha_0^2,1\}\exp(-O(\lambda)),\\
    \eta(\delta)=C\max\{1/\alpha_0^2,1)\cdot \delta,\\
    \kappa = \frac{\exp(-2\lambda)}{2},\\
    \gamma = 4.
\end{gather*}
\end{proposition}

\subsubsection{Metropolis Dynamics}
\label{sec:metro}

\begin{definition}[Metropolis Dynamics]
\label{def:metropolis}
    The (site-homogeneous) \textbf{Metropolis dynamics} are given as follows: each site $i$ has a \emph{proposal rule} that proposes to \emph{flip} to $+1$ with probability $r^i_+\in [0,1]$ and flip to $-1$ with probability $r_-^i$.\footnote{The most common update rules, to our knowledge, are $(1/2,1/2)$ and $(1,1)$, which correspond to a uniform prior and a preference to move as frequently as possible subject to reversibility. Our results can accommodate more general proposal distributions so long as they only depend on the identity of the site.} The transitions are then given by 
    \begin{equation*}
        \mathsf{P}_i(\bm{x},\bm{x}^{\oplus i}) = r^i_{-x_i}\min\left\{\frac{r^i_{x_i}\pi(\bm{x}^{\oplus i})}{r^i_{-x_i}\pi(\bm{x})},1\right\}.
    \end{equation*}
\end{definition}

In words, the sampling is done by proposing whether to flip according to the proposal law given the current value, and then accepting with probability according to the min term. It is straightforward to check that the Metropolis dynamics are reversible with respect to $\pi$ for any choices of proposal distribution by construction.

We show the following settings of parameters in \Cref{sec:app_metro}:
\begin{proposition}
\label{prop:metro_assumption}
Under \Cref{assumption:ising}, the Metropolis dynamics satisfy the conditions of \Cref{assumption:mc} with:
\begin{gather*}
    \delta_0=c\min\{\alpha_0,1\}\exp(-O(\lambda))\cdot \min\{r_+^2,r_-^2\},\\
    \eta(\delta)=\delta/r_-^2,\\
    \kappa = \min\{r_-,r_+\}\exp(-2\lambda),\\
    \gamma = 4.
\end{gather*}
\end{proposition}

\subsection{Observation Filtrations}
\label{sec:obs}

We assume that we only observe the evolution $(X_t)_{t=0}^T$ of a Markov chain satisfying \Cref{assumption:mc} for some suitable value of $T$, but not the set of updates $\Pi_i$.  More formally, we observe the random sets
\begin{equation*}
    \Pi'_i = \{t\leq T: X^{t,-}_i\neq X^{t,+}_i\}\subseteq \Pi_i,
\end{equation*}
where we use the natural notation to denote the left- and right-limits of the coordinates. More formally, we have the following definition:

\begin{definition}[Filtrations]
The \textbf{observation filtration} of the Markov chain $(X_t)_{t=0}^T$ is given by $\mathcal{F}_t=\sigma(X^0,\Pi_1'(t),\ldots,\Pi_n'(t))$. The \textbf{full filtration} of the Markov chain is given by $\mathcal{G}_t=\sigma((X_{\tau})_{\tau=0}^t, \Pi_1(t),\ldots,\Pi_n(t))$.
\end{definition}
In particular, we assume that the learning algorithm must be measurable with respect to the flip observations $\mathcal{F}_t$, a rather complex sub-sigma field of the full history given by $\mathcal{G}_t$. With this larger sigma-field, one can more easily perform estimation using the fact that all update times are known and thus one has an explicit guarantee that each observation of a site update has a valid conditional sample from $\pi$ given the rest of the configuration---this fact is crucially exploited in all prior work on learning the Ising model from dynamics, which thus must permit algorithms that are measurable with respect to the larger $\mathcal{G}_t$. By contrast, the fact that update times are unknown except for those corresponding to sign flips vastly complicates the joint law of the dynamics since a failure to flip comes both from the Markov transitions \emph{and} the unobserved Poissonian clocks.

\section{Anticoncentration of Dynamics}

In this section, we demonstrate a number of anticoncentration statements that will enable our learning guarantees. The main upshot is that the Glauber dynamics, or other reasonable Markov chains, are sufficiently random that we will be able to argue that a small number of sites is not too likely to be determined by the outside configuration after running the dynamics for at least one unit of time. As an application, we can easily derive a crucial estimate on the probability that linear forms anticoncentrate, which will prove to be essential in our analysis in \Cref{sec:structure_learning}. However, these results are somewhat technical and this section can be skipped until the results are needed later.

First, we show in \Cref{prop:local_stable} that while the evolution of the Markov chain may be rather complex after running for a unit of time, there exists a \emph{locally stable} sub-distribution which lower bounds the kernel such that for any initial configuration, the (sub)-probability of any final configuration cannot decrease dramatically upon flipping the site values in $i$ and $j$. Moreover, this sub-transition kernel is quite large in that it has constant probability mass for any initial configuration.

\begin{proposition}[Local Stability Under Dynamics]
\label{prop:local_stable}
    There exists an absolute constant $c_{\ref{prop:local_stable}}>0$ such that the following holds. Let $H_1=\exp(\mathcal{L})$ be the transition matrix on $\{-1,1\}^n$ obtained by running a single-site, reversible Markov chain on an Ising model satisfying \Cref{assumption:ising} and \Cref{assumption:mc}. Then for each $i\neq j$, there exists a sub-transition kernel $Q_{ij}(\cdot,\cdot)$ such that
    \begin{enumerate}
        \item For all $\bm{x},\bm{y}\in \{-1,1\}^n$, 
        \begin{equation}
            H_1(\bm{x},\bm{y})\geq Q_{ij}(\bm{x},\bm{y}),
        \end{equation}
        \item For each $\bm{x}\in \{-1,1\}^n$, $\bm{y}\in \{-1,1\}^{[n]\setminus \{i,j\}}$ (setting of variables outside $i,j$) and $\bm{b},\bm{b}'\in \{-1,1\}^{\{i,j\}}$ (settings of variables for $i,j$),
        \begin{equation}
        \label{eq:stable}
            Q_{ij}(\bm{x},(\bm{y},\bm{b}))\geq c_{\ref{prop:local_stable}}\exp(-O(\gamma\lambda))\kappa^4Q_{ij}(\bm{x},(\bm{y},\bm{b}')),
        \end{equation}
        \item and for all $\bm{x}\in \{-1,1\}$,
        \begin{equation}
            \sum_{\bm{y}\in \{-1,1\}^n} Q_{ij}(\bm{x},\bm{y})\geq c_{\ref{prop:local_stable}}.
        \end{equation}
    \end{enumerate}
\end{proposition} 

Note that \Cref{prop:local_stable} is easily seen to be true if the Markov chain is instead run beyond the mixing time, as then the transitions are close to the stationary distribution where \Cref{eq:stable} is immediate from \Cref{assumption:ising}. However, \Cref{prop:local_stable} would be \emph{false} if instead the chain were run for just $t\ll 1$ time as we should not expect $i$ and $j$ to both update in this interval. \Cref{prop:local_stable} asserts that we nonetheless attain \Cref{eq:stable} at just a constant scale.

\begin{proof}[Proof of \Cref{prop:local_stable}]
    Fix any $\bm{x}\in \{-1,1\}^n$ as well as $\{i,j\}$ and consider running the Markov chain with generator $\mathcal{L}$ for a unit of time to obtain a configuration $X^1$ given $X^0=\bm{x}$. Our goal is to argue that there is a constant probability event $\mathcal{E}$ such that the transition probabilities on this event are locally stable in the sense described above; the statements then all follow by simply considering $Q$ to be the distribution obtained on this event after undoing the conditioning.

    First, for each $\bm{m}=(m_i,m_j)\in \{0,1\}^2$, let $\mathcal{E}_{m_i,m_j}$ denote the event that that $i$ and $j$ are chosen for updating \emph{exactly} $m_i$ and $m_j$ times, respectively. Observe that with some constant probability $c'>0$, $\Pr(\mathcal{E}_{m_1,m_2})\geq c'$ by a simple application of independence of site update times and \Cref{lem:update_bounds}.
    
    Next, define $N_k$ to be the number of times each site $k$ is chosen for updating according to the single-site dynamics in the interval $[0,1]$; note that these are all independent. 
    Let $\mathcal{E}$ denote the following event:
    \begin{equation*}
        \left\{\sum_{k\neq i,j} N_k \vert A_{ik}\vert\leq 4\lambda \right\}\cap \left\{\sum_{k\neq i,j} N_k \vert A_{jk}\vert\leq 4\lambda\right\},
    \end{equation*}
    where $\lambda$ is the width condition from \Cref{assumption:ising}. It is again straightforward to see that
    \begin{equation*}
        \mathbb{E}\left[\sum_{k\neq i,j} N_k \vert A_{ik}\vert\right]=\sum_{k\neq i,j} \mathbb{E}[N_k] \vert A_{ik}\vert = \sum_{k\neq i,j} \vert A_{ik}\vert,
    \end{equation*}
    where the expectation is over the sequence of update times $\{\Pi_k(1)\}_{k\neq i,j}$, which are independent; by standard properties of Poisson point processes, it is immediate to see that all expectations are just $1$. By Markov's inequality, both of these events thus occurs with probability at least $3/4$, and so $\Pr(\mathcal{E})\geq 1/2$. Finally, since the event that $i$ and $j$ are chose for updating exact $m_i$ and $m_j$ times for $(m_i,m_j)\in \{0,1\}^2$ are independent of $\mathcal{E}$, it follows that
    \begin{equation}
    \label{eq:initial_good}
        \Pr(\mathcal{E}\cap \mathcal{E}_{m_i,m_j})\geq c''
    \end{equation}
    for some slightly different absolute constant $c''$. Let $\mathcal{E}'$ denote the following event:
    \begin{equation*}
        \mathcal{E}'=\mathcal{E}\cap \left(\cup_{\bm{m}\in \{0,1\}^2} \mathcal{E}_{m_1m_2}\right).
    \end{equation*}
    Note that when $\mathcal{E}'$ occurs, exactly one of the $\mathcal{E}_{m_1m_2}$ occurs by disjointness, and we know that $\Pr(\mathcal{E}')\geq c''$.

    For the main result, we will establish the following two claims, which show that upon revealing the final configuration outside $i$ and $j$, when $\mathcal{E}'$ holds, each of the possibilities for $\mathcal{E}_{m_1m_2}$ hold with constant probability. We then show that on these events, $i$ and $j$ are still somewhat random, which will give the claim:

    \begin{claim}
    \label{claim:all_events}
        For any $\bm{y}\in \{-1,1\}^{[n]\setminus \{i,j\}}$, and any $\bm{m}\in \{0,1\}^{i,j}$,
        \begin{equation*}
            \Pr\left(\mathcal{E}_{m_1m_2}\vert X^1_{-i,j}=\bm{y},\mathcal{E}'\right)\geq c\exp(-O(\gamma\lambda))\kappa^2.
        \end{equation*}
    \end{claim}

    \begin{claim}
    \label{claim:all_vals}
        For any $\bm{y}\in \{-1,1\}^{[n]\setminus \{i,j\}}$, and any $\bm{m}\in \{0,1\}^{i,j}$,
        \begin{equation*}
            \Pr\left(X^1_{ij} = ((-1)^{m_1}X^0_i,(-1)^{m_2}X^0_j)\vert X^1_{-i,j}=\bm{y},\mathcal{E},\mathcal{E}_{m_1m_2}\right)\geq c\exp(-O(\gamma\lambda))\kappa^2.
        \end{equation*}
    \end{claim}

    We claim that these two inequalities yield the conclusion. Fix any $\bm{y}\in \{-1,1\}^{[n]\setminus\{i,j\}}$ and let $\bm{b}$ be such that $\bm{b} = ((-1)^{m_1}X^0_i,(-1)^{m_2}X^0_j)$. Then applying \Cref{claim:all_events} and \Cref{claim:all_vals}
    \begin{align}
    \nonumber
        \Pr(X^1_{i,j}=\bm{b}\vert X^1_{-i,j}=\bm{y},\mathcal{E}') &\geq \Pr\left(\mathcal{E}_{m_1m_2}\vert X^1_{-i,j}=\bm{y},\mathcal{E}'\right)\cdot \Pr(X^1_{i,j}=\bm{b}\vert X^1_{-i,j}=\bm{y},\mathcal{E},\mathcal{E}_{m_1m_2})\\
        \label{eq:conditional_stability}
        &\geq c^2\exp(-O(\gamma\lambda))\kappa^4.
    \end{align}
    Therefore, on the event $\mathcal{E}'$ and given any configuration $\bm{y}\in \{-1,1\}^{[n]\setminus \{i,j\}}$ for the variables outside $i$ and $j$ at time $1$, all possible values for the $i,j$ coordinate occur with constant probability. We can thus define the sub-transition kernel $Q_{ij}(\bm{x},\bm{y})$ via 
    \begin{equation*}
        Q_{ij}(\bm{x},\bm{y})=\Pr\left(X^1=\bm{y},\mathcal{E}'\vert X^0=\bm{x}\right).
    \end{equation*}
    The conditional probabilities follow from \Cref{eq:conditional_stability} upon replacing possibly adjusting the value of $c$ and the lower bound on the sub-transition kernel follows from \Cref{eq:initial_good}.

    We now prove these claims in order:
    \begin{proof}[Proof of \Cref{claim:all_events}]
    First, applying a simple averaging argument via Bayes' rule as in \Cref{lem:averaging} shows that
    \begin{equation}
    \label{eq:ratio_sup}
        \frac{\Pr\left(\mathcal{E}_{m_1m_2}\vert X^1_{-i,j}=\bm{y},\mathcal{E}'\right)}{\Pr\left(\mathcal{E}_{m'_1m'_2}\vert X^1_{-i,j}=\bm{y},\mathcal{E}'\right)}\leq \sup_{(\Pi,Z),(\Pi^{i,j}_{\bm{m}},Z_{\bm{m}}),(\Pi^{i,j}_{\bm{m}'},Z_{\bm{m}'})} \frac{\Pr(Z,Z_{\bm{m}}\vert \Pi,\Pi^{i,j}_{\bm{m}})}{\Pr(Z,Z_{\bm{m}'}\vert \Pi,\Pi^{i,j}_{\bm{m}'})},
    \end{equation}
    where $\Pi$ denotes the update times of sites outside of $\{i,j\}$ that satisfy $\mathcal{E}$ and $Z$ denotes a sequence of transitions for the sites outside $\{i,j\}$ that induce $\bm{y}$, $\Pi^{i,j}_{\bm{m}}$ denotes any choice of update times for $\{i,j\}$ satisfying $\mathcal{E}_{m_1m_2}$ and $Z_{\bm{m}}$ is any sequence of transitions with strictly positive probability under the transition kernels, and analogously for $\Pi^{i,j}_{\bm{m}'}$. We will now show that this is in turn bounded by a suitable constant depending only on the stated parameters.

    Since this conditioning stipulates all the update times in $[0,1]$, both probabilities of this path of updates factorizes as the product over the transitions given by the sequence $Z$ and the $i,j$ updates by the Markov property. For the transitions in $\{i,j\}$, we may upper bound the numerator by $1$ and lower bound the denominator transition factors by $\kappa^2$ using the lower bound of \Cref{assumption:mc} as there are at most two such site updates. For each transition step in both the numerator and denominator, the corresponding ratio is
    \begin{equation*}
        \frac{\mathsf{P}_k(X,X^{k\mapsto \pm 1})}{\mathsf{P}_k(X',X^{',k\mapsto\pm 1})}
    \end{equation*}
    for some configurations $X,X'$ that differ at most at the values of $i$ and $j$. By \Cref{assumption:mc} and the boundedness condition of \Cref{def:bounded}, this ratio is at most 
    \begin{equation*}
 \exp\left(\gamma(\vert A_{i,k}\vert+\vert A_{j,k}\vert)\right),
    \end{equation*}
    where we use the fact that all terms cancel except possibly the contribution of the differences at site $i$ and $j$.
    It follows that \Cref{eq:ratio_sup} is bounded by 
    \begin{equation*}
        \frac{\exp\left(\gamma \sum_{k\neq i,j} N_k\vert A_{i,k}\vert\right)}{\kappa^2}\leq \frac{\exp(O(\gamma\lambda))}{\kappa^2},
    \end{equation*}
    where we use the definition of $\mathcal{E}$ to bound the sum in the exponential. Since there are only four such events $\mathcal{E}_{\bm{m}}$ since $\bm{m}\in \{0,1\}^2$ and these conditional probabilities must sum to $1$, this upper bound on \Cref{eq:ratio_sup} implies that
    \begin{equation*}
        \Pr\left(\mathcal{E}_{m_1m_2}\vert X_{-i,j}=\bm{y},\mathcal{E}'\right)\geq c\exp(-O(\gamma \lambda))\kappa^2,
    \end{equation*}
    as claimed.
    \end{proof}

    \begin{proof}[Proof of \Cref{claim:all_vals}]
        We use a similar argument as before. It suffices to upper bound, for any value of $\bm{b}\in \{-1,1\}^{\{i,j\}}$, the ratio
        \begin{equation*}
            \frac{\Pr\left(X^1_{ij} = \bm{b}\vert X_{-i,j}=\bm{y},\mathcal{E},\mathcal{E}_{m_1m_2}\right)}{\Pr\left(X^1_{ij} = ((-1)^{m_1}X^0_i,(-1)^{m_2}X^0_j)\vert X_{-i,j}=\bm{y},\mathcal{E},\mathcal{E}_{m_1m_2}\right)};
        \end{equation*}
        since there are at most $4$ possible values of $X^1_{i,j}$, this suffices to prove the claim. By another application of \Cref{lem:averaging}, it suffices to bound, for any $(\Pi,Z),(\Pi^{i,j}_{\bm{m}})$ satisfying $\mathcal{E},\mathcal{E}_{m_1m_2}$ using the same notation as in the proof of \Cref{claim:all_events},
        \begin{equation*}
            \frac{\Pr\left(X^1_{ij} = \bm{b},Z\vert \Pi,\Pi^{i,j}_{\bm{m}}\right)}{\Pr\left(X^1_{ij} = ((-1)^{m_1}X^0_i,(-1)^{m_2}X^0_j),Z\vert \Pi,\Pi^{i,j}_{\bm{m}}\right)}.
        \end{equation*}
        As before, since all update times are given and conditioned upon, both ratios factorize according to the transition probabilities. The exact same argument (bounding the numerator transitions of $i$ and $j$ by $1$ trivially and the denominator transitions below by at most $\kappa^2$, and then the ratios in the exact same way) yields an upper bound of
        \begin{equation*}
            \exp(O(\gamma \lambda))/\kappa^2.\qedhere
        \end{equation*}
    \end{proof}
    \end{proof}

    While \Cref{prop:local_stable} was stated to hold for the Markov chain run for time $1$ from a starting configuration, a simple application of the Markov property shows that the same holds for any $t>1$ and any random initial configuration:

    \begin{corollary}[Local Stability with Random Initialization]
    \label{cor:stable_random}
        Let $X^0\sim \mathcal{D}$ for an arbitrary distribution $\mathcal{D}$ on $\{-1,1\}^n$. Let $P$ denote the law of $X^1$ with this initial configuration after running the Markov chain for time 1 satisfying \Cref{assumption:mc}. Under the conditions of \Cref{prop:local_stable}, there exists a sub-distribution $Q'_{ij}$ on $\{-1,1\}^n$ such that
        \begin{enumerate}
            \item For all $\bm{y}\in \{-1,1\}^n$,
            \begin{equation*}
                P(X^1=\bm{y})\geq Q'_{ij}(\bm{y}),
            \end{equation*}
            \item For all $\bm{y}\in \{-1,1\}^{[n]\setminus \{i,j\}}$ (setting of variables outside $i,j$) and $\bm{b},\bm{b}'\in \{-1,1\}^{\{i,j\}}$ (settings of variables for $i,j$),
        \begin{equation*}
            Q'_{ij}((\bm{y},\bm{b}))\geq c\exp(-O(\gamma\lambda))\kappa^4Q'_{ij}(\bm{x},(\bm{y},\bm{b}')),
        \end{equation*}
        \item and
        \begin{equation*}
            \sum_{\bm{y}\in \{-1,1\}^n} Q'_{ij}(\bm{y})\geq c.
        \end{equation*}
        \end{enumerate}

        In particular, \Cref{prop:event_probs} holds as stated for the transition matrix of $H_t$ for \emph{any} $t\geq 1$ uniformly.
    \end{corollary}
    \begin{proof}
        Define the following sub-distribution $Q'_{ij}$ on $\{-1,1\}^n$ via
        \begin{equation*}
            Q'_{ij}(\bm{y})=\mathbb{E}_{X^0\sim \mathcal{D}}[Q_{ij}(X,\bm{y})],
        \end{equation*}
        where $Q_{ij}$ is obtained from \Cref{prop:local_stable}.
        The first inequality and third inequalities are immediate from the corresponding inequalities there, while the second follows from
        \begin{align*}
            Q'_{ij}(\bm{y},\bm{b}) &= \mathbb{E}_{X^0\sim \mathcal{D}}[Q_{ij}(X,(\bm{y},\bm{b}))]\\
            &\geq c\exp(-O(\gamma\lambda))\kappa^4 \mathbb{E}_X[Q_{ij}(X,(\bm{y},\bm{b}'))]\\
            &=\exp(-O(\gamma\lambda))\kappa^4Q'_{ij}(\bm{y},\bm{b}').
        \end{align*}

    The ``in particular" part is an immediate consequence, since by the Markov property, for any $t\geq 1$:
    \begin{equation*}
        H_t(\bm{x},\bm{y}) = \mathbb{E}_{X\sim \mathcal{D}_{\bm{x}}}[H_1(X,\bm{y})],
    \end{equation*}
    where $\mathcal{D}_{\bm{x}}$ is the law of $X^{t-1}$ conditional on $X^0=\bm{x}$.
    \end{proof}

    We now derive the following simple anticoncentration result for linear forms with a noticeable coefficient. We note that one can establish this more directly, but the quantitative guarantees will suffice for our applications.
    \begin{corollary}[Dynamical Anticoncentration of Linear Forms]
    \label{cor:anti}
        Under the conditions of \Cref{prop:local_stable}, the following holds for any $\alpha>0$. Let $\ell(\bm{x})=\sum_{k=1}^n v_kx_k$ be any linear form such that $\vert v_i\vert\geq \alpha$. Then for any $c\in \mathbb{R}$, any initial distribution $\mathcal{D}$ for $X^0$, and any $t\geq 1$, it holds that
        \begin{equation*}
            \Pr_{X^1}\left(\vert \ell(X^1)-c\vert\geq\alpha\right)\geq c_{\ref{cor:anti}}\exp(-O(\gamma\lambda))\kappa^4.
        \end{equation*}
    \end{corollary}
    \begin{proof}
    For any $j\in [n]$, let $Q'_{ij}$ denote the (sub)-distribution of \Cref{cor:stable_random}. Fix any $\bm{y}\in \{-1,1\}^{[n]\setminus\{i,j\}}$, and observe that upon setting the variables in $\ell(\cdot)$ to $\bm{y}$, there exists at least one setting of $i,j$ such that
    \begin{equation*}
        \vert \ell(X^1)-c\vert\geq \alpha;
    \end{equation*}
    indeed, the value of site $j$ can be arbitrary, and then the value of $x_i$ can be set to have the same sign as $\mathsf{sgn}(v_i)\left(v_jx_j+\sum_{k\neq i,j}v_ky_k-c\right)$ to ensure the absolute value is at least $\alpha$ from the assumption on $\vert v_i\vert\geq \alpha$. For this $\bm{y}$, the conditional probability of this value of $i,j$ is at least $c\exp(-O(\gamma\lambda))\kappa^4$. It follows that
    \begin{equation*}
        Q'_{ij}\left(\vert \ell(X^1)-c\vert\geq\alpha\right)\geq c'\exp(-O(\gamma\lambda))\kappa^4,
    \end{equation*}
    where we possibly adjust the value of $c'$. Since $Q'_{ij}$ gives a lower bound for the true probabilities by \Cref{cor:stable_random}, this completes the proof of the first part. An identical proof holds for $X^t$ when $t\geq 1$ also by \Cref{cor:stable_random}.
    \end{proof}

\section{Short Cycles and Structure Learning}
\label{sec:structure_learning}

In this section, we provide our main structure learning algorithm. As described in \Cref{sec:overview}, our analysis shows that short \emph{cycles} where sites $i$ and $j$ flip in relatively close proximity can \emph{almost} reveal dependency in the Ising model. In \Cref{sec:flip_stats}, we provide the key technical estimate that provides useful identities for the probability of observing short cycle sequences; the main point will be that on small enough windows, with size independent of $n$, the relative error of the statistic will tend to zero. We then leverage this in \Cref{sec:cycles} to define our main cycle statistic to detect dense edges in the Ising model---we provide the full algorithm in \Cref{sec:bulk}. Finally, we provide our sub-routine to determine the remaining edges, which are necessarily isolated in the full graph $G$, in \Cref{sec:matchings}.

\subsection{Flip Statistics}
\label{sec:flip_stats}

Let $\varepsilon>0$ be a small constant that we will choose later. For a fixed pair $i\neq j\in [n]$, let $\bm{\ell}=(\ell_1,\ldots,\ell_m)\in \{i,j\}^*$ be any sequence of $i$ and $j$ pairs. We will write $\overline{\ell_k}$ to denote the other index in $\{i,j\}$ that is not given by $\ell_k$. For some fixed time $t>0$, let $\mathcal{E}^{t}_{\bm{\ell}}$ denote the following event:

\begin{align}
\label{eq:flip_event}
    \mathcal{E}^{t}_{\bm{\ell}}& = \bigcap_{k=1}^m \left\{\vert \Pi'_{\ell_k}(t+(k-1)\varepsilon,t+k\varepsilon)\vert=1, \vert \Pi'_{\overline{\ell_k}}(t+(k-1)\varepsilon,t+k\varepsilon)\vert=0\right\}\\
    \nonumber
    &:= \bigcap_{k=1}^m \left\{\vert \Pi'_{\ell_k}(I_k)\vert=1, \vert \Pi'_{\overline{\ell_k}}(I_k)\vert=0\right\},
\end{align}
where we have defined $I_k:=[t+(k-1)\varepsilon,t+k\varepsilon]$.

In words, these events measure short cycles of flips where both $i$ and $j$ flip exactly once in each interval of length $\varepsilon$ in the order given by $(\ell_1,\ldots,\ell_m)$ that starts at time $t$. Our key observation is that suitable choices of indices will almost always reveal the dependency structure if $\varepsilon>0$ is taken to be a sufficiently small constant, \emph{except} for a pathological case that we can then test for directly. First, we require the following expression for the likelihood of these events:

\begin{proposition}
\label{prop:event_probs}
    There is an absolute constant $C_{\ref{prop:event_probs}}>0$ such that the following holds. Suppose that $(X_t)_{t=0}^T$ follows any reversible, single-site Markov chain with respect to $\pi$ satisfying \Cref{assumption:ising} and \Cref{assumption:mc}. Let $t>0$ be some fixed time. Let $\bm{\ell}\in \{i,j\}^*$ denote any sequence and set $m = \vert \bm{\ell}\vert$. Then for any $\varepsilon<1/C_{\ref{prop:event_probs}}m$, it holds that
    \begin{align*}
\Pr\left(\mathcal{E}^{t}_{\bm{\ell}}\vert\mathcal{F}_{t}\right) &= \varepsilon^{m}\prod_{k=1}^{m} \mathsf{P}_{\ell_k}(X^{t,\oplus \ell_1\ldots \ell_{k-1}},X^{t,\oplus \ell_1\ldots \ell_{k}}) \pm C_{\ref{prop:event_probs}}md\varepsilon^{m+1}\\
&=\varepsilon^m\left(\prod_{k=1}^{m} \mathsf{P}_{\ell_k}(X^{t,\oplus \ell_1\ldots \ell_{k-1}},X^{t,\oplus \ell_1\ldots \ell_{k}})\pm C_{\ref{prop:event_probs}}md\varepsilon\right)
    \end{align*}
\end{proposition}
In words, this result shows that so long as we set $\varepsilon$ to be a sufficiently small constant depending only on the length of the flip sequence and degree of the Ising model, then the probability of observing a given flip sequence is given by product of the transitions up to a small error after accounting for the scaling. The main observation here is the justification of the error as being higher-order depending mildly only on the sequence length and degree, not system size. We will only care about sequences with $m=O(1)$, so the error term can be made negligible if $\varepsilon\ll 1/d$.

\begin{proof}
The main idea is simply that the most likely way for the stated event to occur is under the assumption that site $i$ and $j$ attempt to update, and succeed in flipping, exactly in the stated order with multiplicity while no other neighbor updates along this interval. Any additional updates that induce unwieldy dependencies yet satisfy the event implies that there were at least $m+1$ update attempts among this set of sites, which has higher-order probability $O(d\varepsilon^{m+1})$.

    More formally, let $\mathcal{A}$ denote the event $\cap_{k\in \mathcal{N}(i)\cap \mathcal{N}(j)\setminus \{i,j\}} \left\{\Pi_k(t,t+m\varepsilon)=\emptyset\right\}$, i.e. no neighbor of either $i$ or $j$ attempts to update in the interval of length $m\varepsilon$.
    We can now compute
    \begin{align}
    \nonumber
        \Pr\left(\mathcal{E}^{t}_{\bm{\ell}}\vert X^{t}\right) &= \Pr\left(\mathcal{E}^{t}_{\bm{\ell}}\vert X^{t},\mathcal{A}\right)\cdot \Pr(\mathcal{A}\vert X^t)\\
        \label{eq:decomp}
        &+\Pr\left(\mathcal{E}^{t}_{\bm{\ell}}\cap \mathcal{A}^c\vert X^{t}\right).
    \end{align}

    We first bound the probability of the latter term. Observe that
    \begin{equation*}
        \mathcal{E}^t_{\bm{\ell}}\cap \mathcal{A}^c\subseteq \cap_{k=1}^m \{\Pi_{\ell_k}(I_k)\neq \emptyset\}\cap \mathcal{A}^c:=\mathcal{B},
    \end{equation*}
    where $U_I$ denotes the set of sites that update in the interval $I=[t,t+m\varepsilon]$. It follows that
    \begin{equation*}
        \Pr\left(\mathcal{E}^t_{\bm{\ell}}\cap \mathcal{A}^c\vert X^t\right)\leq \Pr(\mathcal{B}),
    \end{equation*}
    where we may drop the conditioning as this event depends only on update times which are independent of the configuration at time $t$. By the independence of update times across sites, we obtain
    \begin{equation*}
        \Pr(\mathcal{B})\leq \Pr(\mathcal{A}^c)\cdot \prod_{k=1}^m \Pr(\Pi_{\ell_k}(I_k)\neq \emptyset).
    \end{equation*}
    By \Cref{lem:update_bounds}, the product is given by $(1-\exp(-\varepsilon))^m\leq \varepsilon^m$ where we use the simple inequality $1-\exp(-x)\leq x$ for all $x\geq 0$. For the first term,
    \Cref{lem:update_bounds} again implies
    \begin{align*}
    \Pr(\mathcal{A}) &\geq \exp(-m\varepsilon\vert \mathcal{N}(i)\cup \mathcal{N}(j)\vert)\\
    &\geq \exp(-2md\varepsilon)\\
    &\geq 1-2md\varepsilon
    \end{align*}
    and therefore the complementary event is bounded by $2md\varepsilon$. Here, we use \Cref{assumption:ising} to assert that $\vert \mathcal{N}(i)\cup \mathcal{N}(j)\vert\leq 2d$, as well as again the simple inequality $\exp(-x)\geq 1-x$ for all $x\geq 0$. We conclude that
    \begin{equation}
    \label{eq:update_error}
        \Pr\left(\mathcal{E}^t_{\bm{\ell}}\cap \mathcal{A}^c\vert X^t\right)\leq 2md\varepsilon^{m+1}.
    \end{equation}

    We now turn to the main term applying similar reasoning. On the event $\mathcal{A}$, no neighbor of either $i$ or $j$ even attempts to update, and since the Markov chain transitions are conditionally independent of update times, it follows by the Markov property that
    \begin{equation}
    \label{eq:event_prod}
        \Pr(\mathcal{E}^t_{\bm{\ell}}\vert X^t,\mathcal{A}) = \prod_{k=1}^m \Pr\left(\vert \Pi'_{\ell_k}(I_k)\vert=1, \vert \Pi'_{\overline{\ell_k}}(I_k)\vert=0\bigg\vert \mathcal{A},X^{t,\oplus \ell_1\ldots\ell_{k-1}}\right).
    \end{equation}

We now claim the following bounds for each term in \Cref{eq:event_prod} showing that the event is the same as the event that there was only a single update attempt for $\ell_k$ and none for $\overline{\ell_k}$ up to higher-order erro, which follows analogous reasoning:
\begin{claim}
\label{claim:claim_1}
For each $k\leq m$, it holds that
    \begin{align*}
        \Pr\left(\vert \Pi'_{\ell_k}(I_k)\vert=1, \vert \Pi'_{\overline{\ell_k}}(I_k)\vert=0\bigg\vert \mathcal{A},X^{t,\oplus \ell_1\ldots\ell_{k-1}}\right) &= \Pr\left(\vert \Pi_{\ell_k}(I_k)\vert = \vert \Pi'_{\ell_k}(I_k)\vert=1\bigg\vert \vert \Pi_{\overline{\ell_k}}(I_k)\vert=0,\mathcal{A},X^{t,\oplus \ell_1\ldots\ell_{k-1}}\right)\\
        &+O(\varepsilon^2)
    \end{align*}
    \end{claim}
    Informally, this holds because the most likely way for the desired event to hold is for site $\ell_k$ to update exactly once and flip while site $\overline{\ell_k}$ never updates. We defer the proof until after the main statement.

    Given \Cref{claim:claim_1}, we can now directly evaluate the product in \Cref{eq:event_prod}. By the independence of site updates, we have
    \begin{align*}
        \Pr&\left(\vert \Pi_{\ell_k}(I_k)\vert = \vert \Pi'_{\ell_k}(I_k)\vert=1\bigg\vert \vert \Pi_{\overline{\ell_k}}(I_k)\vert=0,\mathcal{A},X^{t,\oplus \ell_1\ldots\ell_{k-1}}\right)\\
        &= \Pr\left( \vert \Pi'_{\ell_k}(I_k)\vert=1\bigg\vert \vert \Pi_{\ell_k}(I_k)\vert=1,\vert \Pi_{\overline{\ell_k}}(I_k)\vert=0,\mathcal{A},X^{t,\oplus \ell_1\ldots\ell_{k-1}}\right)\\
        &\cdot \Pr\left( \vert \Pi_{\ell_k}(I_k)\vert=1\bigg\vert \vert \Pi_{\overline{\ell_k}}(I_k)\vert=0,\mathcal{A},X^{t,\oplus \ell_1\ldots\ell_{k-1}}\right)\\
        &=\mathsf{P}_{\ell_k}(X^{t,\oplus \ell_1\ldots \ell_{k-1}},X^{t,\oplus \ell_1\ldots \ell_{k}})\cdot (1-\exp(-\varepsilon)+O(\varepsilon^2))\\
        &=\varepsilon\left(\mathsf{P}_{\ell_k}(X^{t,\oplus \ell_1\ldots \ell_{k-1}},X^{t,\oplus \ell_1\ldots \ell_{k}})+O(\varepsilon)\right).
    \end{align*}
    In the last step, we use the fact that given there is exactly one update of site $\ell_k$ and no updates by neighbors or $\overline{\ell}_k$, the probability of a flip is precisely given by the transition kernel. We also use \Cref{lem:update_bounds} to write the probability of there being exactly one update by $\ell_k$, which is independent of the conditioning. Combining the previously display with \Cref{claim:claim_1} and \Cref{eq:event_prod}, we obtain
    \begin{equation}
    \label{eq:eps_sequence}
        \Pr(\mathcal{E}^t_{\bm{\ell}}\vert X^t,\mathcal{A}) = \varepsilon^m \prod_{k=1}^m \left(\mathsf{P}_{\ell_k}(X^{t,\oplus \ell_1\ldots \ell_{k-1}},X^{t,\oplus \ell_1\ldots \ell_{k}})+O(\varepsilon)\right).
    \end{equation}
    Since we have assumed that $\varepsilon<1/Cm$ for a sufficiently large constant, it follows that each term in the product is at most $(1+1/m)$ as transitions are at most $1$. Since $(1+1/m)^m\leq e$, applying \Cref{lem:seq_diff} yields
    \begin{equation}
        \prod_{k=1}^m \left(\mathsf{P}_{\ell_k}(X^{t,\oplus \ell_1\ldots \ell_{k-1}},X^{t,\oplus \ell_1\ldots \ell_{k}})+O(\varepsilon)\right) = \prod_{k=1}^m \mathsf{P}_{\ell_k}(X^{t,\oplus \ell_1\ldots \ell_{k-1}},X^{t,\oplus \ell_1\ldots \ell_{k}}) + Cm\varepsilon,
    \end{equation}
    Combining \Cref{eq:decomp}, \Cref{eq:update_error}, and \Cref{eq:eps_sequence} with the previous display proves the claim.
\end{proof}

We now return to the proof of \Cref{claim:claim_1}, which follows essentially identical reasoning to \Cref{prop:event_probs} to argue about the most likely update sequences on short intervals.
\begin{proof}[Proof of \Cref{claim:claim_1}]
    First, we rewrite

    \begin{align*}
        \Pr&\left(\vert \Pi'_{\ell_k}(I_k)\vert=1, \vert \Pi'_{\overline{\ell_k}}(I_k)\vert=0\bigg\vert \mathcal{A},X^{t,\oplus \ell_1\ldots\ell_{k-1}}\right)=\Pr\left(\vert \Pi'_{\ell_k}(I_k)\vert=1, \vert \Pi_{\overline{\ell_k}}(I_k)\vert=0\bigg\vert \mathcal{A},X^{t,\oplus \ell_1\ldots\ell_{k-1}}\right)\\
        &+\Pr\left(\vert \Pi'_{\ell_k}(I_k)\vert=1, \vert \Pi'_{\overline{\ell_k}}(I_k)\vert=0,\vert \Pi_{\overline{\ell_k}}(I_k)\vert\geq 1 \bigg\vert \mathcal{A},X^{t,\oplus \ell_1\ldots\ell_{k-1}}\right).
    \end{align*}
    The same logic as before implies that the latter term has probability $O(\varepsilon^2)$, as it is implied by the event that both $i$ and $j$ attempt to update in the interval $I_k$ which can be bounded similarly as before by \Cref{lem:update_bounds} using the independence of site updates. Similarly,
    \begin{align*}
        \Pr\left(\vert \Pi'_{\ell_k}(I_k)\vert=1, \vert \Pi_{\overline{\ell_k}}(I_k)\vert=0\bigg\vert \mathcal{A},X^{t,\oplus \ell_1\ldots\ell_{k-1}}\right)&=\Pr\left(\vert \Pi'_{\ell_k}(I_k)\vert=1\bigg\vert \vert \Pi_{\overline{\ell_k}}(I_k)\vert=0,\mathcal{A},X^{t,\oplus \ell_1\ldots\ell_{k-1}}\right)\\
        &\cdot \Pr\left(\vert \Pi_{\overline{\ell_k}}(I_k)\vert=0\bigg\vert\mathcal{A},X^{t,\oplus \ell_1\ldots\ell_{k-1}}\right)\\
        &=\Pr\left(\vert \Pi'_{\ell_k}(I_k)\vert=1\bigg\vert \vert \Pi_{\overline{\ell_k}}(I_k)\vert=0,\mathcal{A},X^{t,\oplus \ell_1\ldots\ell_{k-1}}\right)(1-O(\varepsilon))\\
        &=\Pr\left(\vert \Pi'_{\ell_k}(I_k)\vert=1\bigg\vert \vert \Pi_{\overline{\ell_k}}(I_k)\vert=0,\mathcal{A},X^{t,\oplus \ell_1\ldots\ell_{k-1}}\right)+O(\varepsilon^2),
    \end{align*}
    using again the fact that the probability that there are any update events for a given site is bounded by $\varepsilon$ and independence across site with similar logic as before. Finally, similar logic implies that
    \begin{align*}
        \Pr\left(\vert \Pi'_{\ell_k}(I_k)\vert=1\bigg\vert \vert \Pi_{\overline{\ell_k}}(I_k)\vert=0,\mathcal{A},X^{t,\oplus \ell_1\ldots\ell_{k-1}}\right)&=\Pr\left(\vert \Pi_{\ell_k}(I_k)\vert=\vert \Pi'_{\ell_k}(I_k)\vert=1\bigg\vert \vert \Pi_{\overline{\ell_k}}(I_k)\vert=0,\mathcal{A},X^{t,\oplus \ell_1\ldots\ell_{k-1}}\right)\\
        &+\Pr\left(\vert \Pi_{\ell_k}(I_k)\vert>1,\vert \Pi'_{\ell_k}(I_k)\vert=1\bigg\vert \vert \Pi_{\overline{\ell_k}}(I_k)\vert=0,\mathcal{A},X^{t,\oplus \ell_1\ldots\ell_{k-1}}\right)\\
        &\leq \Pr\left(\vert \Pi_{\ell_k}(I_k)\vert=\vert \Pi'_{\ell_k}(I_k)\vert=1\bigg\vert \vert \Pi_{\overline{\ell_k}}(I_k)\vert=0,\mathcal{A},X^{t,\oplus \ell_1\ldots\ell_{k-1}}\right)\\
        &+\Pr\left(\vert \Pi_{\ell_k}(I_k)\vert>1\bigg\vert \vert \Pi_{\overline{\ell_k}}(I_k)\vert=0,\mathcal{A},X^{t,\oplus \ell_1\ldots\ell_{k-1}}\right).
    \end{align*}
    The latter term is again $O(\varepsilon^2)$ by identical reasoning via \Cref{lem:update_bounds}. Collecting these inequalities yields \Cref{claim:claim_1}.
\end{proof}

\subsection{Distinguishing Cycle Statistics}
\label{sec:cycles}

We can now give our main result that gives a \emph{nonnegative} lower bound on the difference in probabilities of suitably defined flip sequences, which we will later argue must be \emph{strictly} positive so long as certain local fields are nondegenerate. 
Fix the \emph{ordered} pair $(i,j)$ where $i\neq j$, a time $t\geq 0$, and define:
\begin{equation}
\label{eq:cycle_def}
    Z^{i,j}_t := \mathbf{1}\{\mathcal{E}^t_{iijjiijj}\}-2\cdot \mathbf{1}\{\mathcal{E}_{iijjjiij}\}+\mathbf{1}\{\mathcal{E}_{jiijjiij}\}.
\end{equation}
As discussed in \Cref{sec:overview}, this can be viewed as the ``square'' of the difference between the cycles $iijj$ and $ijji$, which we should thus expect to be nonnegative.

\begin{proposition}
\label{prop:main_stat}
    There is an absolute constant $c_{\ref{prop:main_stat}}>0$ such that the following holds under \Cref{assumption:ising} and \Cref{assumption:mc}. For any time $t>0$, and $\varepsilon<c_{\ref{prop:main_stat}}$,
    \begin{equation*}
        \mathbb{E}\left[Z^{i,j}_t\vert \mathcal{F}_t\right] = g_j^2\left(\frac{\pi(X^{t,j\mapsto +1})}{\pi(X^{t,j\mapsto -1})}\right)\varepsilon^8\left(g_i\left(\frac{\pi(X^{t,i\mapsto +1})}{\pi(X^{t,i\mapsto -1})}\right)-g_i\left(\frac{\pi(X^{t,\oplus j,i\mapsto +1})}{\pi(X^{t,\oplus j,i\mapsto -1})}\right)\right)^2+ O( d\varepsilon^9).
    \end{equation*}
\end{proposition}
\begin{proof}
    We appeal to \Cref{prop:event_probs} to derive the stated result. For each event in the definition of $Z^{i,j}_t$, observe that each flip of $j$ occurs precisely when the value of $i$ is set to the initial configuration. Therefore, each product corresponding to a $j$ transition in the conclusion of \Cref{prop:event_probs} occurs at the initial configuration, and there are precisely two flips in each direction. By \Cref{def:site_consistent}, these four factors thus contribute exactly 
    \begin{equation*}
        g^2_j\left(\frac{\pi(X^{t,j\mapsto +1})}{\pi(X^{t,j\mapsto -1})}\right).
    \end{equation*}

    We now consider what happens for the $i$ flips for each event. In the first event $\mathcal{E}^t_{iijjiijj}$, all $i$ events occur when $j$ is set to be the initial configuration, and there are again precisely two flips in each direction. The factors thus become
    \begin{equation*}
        g^2_i\left(\frac{\pi(X^{t,i\mapsto +1})}{\pi(X^{t,i\mapsto -1})}\right).
    \end{equation*}

    For the middle event $\mathcal{E}_{iijjjiij}$, there are two flips of $i$ from the initial configuration, and two flips of $i$ when $j$ is reversed. Therefore, the product of the transitions becomes
    \begin{equation*}
        g_i\left(\frac{\pi(X^{t,i\mapsto +1})}{\pi(X^{t,i\mapsto -1})}\right)g_i\left(\frac{\pi(X^{t,\oplus j,i\mapsto +1})}{\pi(X^{t,\oplus j,i\mapsto -1})}\right).
    \end{equation*}

    Analogous reasoning for the event $\mathcal{E}_{jiijjiij}$ gives that there are two flips in each direction for $i$, all occurring when $j$ is flipped from the initial configuration. \Cref{def:site_consistent} again implies that the product of transitions becomes
    \begin{equation}
        g_i^2\left(\frac{\pi(X^{t,\oplus j,i\mapsto +1})}{\pi(X^{t,\oplus j,i\mapsto -1})}\right).
    \end{equation}

    Therefore, applying \Cref{prop:event_probs},linearity of expectation, and factoring the square yields that
    \begin{align*}
        \mathbb{E}\left[Z^{i,j}_t\vert \mathcal{F}_t\right] = g_j^2\left(\frac{\pi(X^{t,j\mapsto +1})}{\pi(X^{t,j\mapsto -1})}\right)\varepsilon^8\left(g_i\left(\frac{\pi(X^{t,i\mapsto +1})}{\pi(X^{t,i\mapsto -1})}\right)-g_i\left(\frac{\pi(X^{t,\oplus j,i\mapsto +1})}{\pi(X^{t,\oplus j,i\mapsto -1})}\right)\right)^2+ O( d\varepsilon^9),
    \end{align*}
    as claimed since $m=O(1)$.
\end{proof}

\begin{corollary}
\label{cor:non_neighbors}
    There is an absolute constant $c_{\ref{prop:main_stat}}$ such that the following holds under \Cref{assumption:ising} and \Cref{assumption:mc}. If $i\not\sim j$, then for any time $t>0$ and $\varepsilon<c_{\ref{prop:main_stat}}$,
    \begin{equation*}
        \mathbb{E}\left[Z^{i,j}_t\vert \mathcal{F}_t\right] =O(d\varepsilon^{9}).
    \end{equation*}
\end{corollary}
\begin{proof}
    Since $i\not\sim j$ by assumption, it holds that
    \begin{equation*}
        \frac{\pi(X^{t,i\mapsto +1})}{\pi(X^{t,i\mapsto -1})}=\exp\left(2\sum_{k\neq i}A_{i,j}X^t_k\right)=\frac{\pi(X^{t,\oplus j,i\mapsto +1})}{\pi(X^{t,\oplus j,i\mapsto -1})},
    \end{equation*}
    since $j$ does not appear in the sum. Applying \Cref{prop:main_stat} completes the proof.
\end{proof}

We can now argue that whenever $i\sim j$ and there exists a distinct $k$ such that $i\sim k$ as well, then with some nonnegligible probability, the statistic will be strictly positive. As described on \Cref{sec:overview}, this follows from an anticoncentration argument showing that it is not possible for this statistic to conspire all the time to be small:

\begin{corollary}
\label{cor:bulk_neighbors}
    There is an absolute constant $c_{\ref{cor:bulk_neighbors}}>0$ such that the following holds under \Cref{assumption:ising} and \Cref{assumption:ising}. Suppose that $i\sim j,k$ where $j\neq k$. Let $0<\delta\leq \delta_0$ be such that
    \begin{equation*}
        \eta(\delta)\leq \frac{c_{\ref{cor:bulk_neighbors}}\exp(-O(\lambda))}{\alpha}.
    \end{equation*}
    Then for any times $t,t'\geq 0$ such that $t'\leq t-1$, if $\varepsilon<c_{\ref{cor:bulk_neighbors}}\kappa^8\delta^2\exp(-O(\gamma\lambda))/d$, then 
    \begin{equation}
        \mathbb{E}[Z^{i,j}_t\vert \mathcal{F}_{t'}]\geq c_{\ref{cor:bulk_neighbors}}\varepsilon^8\kappa^8\delta^2\exp(-O(\gamma\lambda)).
    \end{equation}
\end{corollary}
\begin{proof}
    First, let $\ell(\bm{x})=\sum_{\ell\neq i,j} A_{i\ell}x_k$. By our assumption, this sum is nontrivial since $i\sim k$, and the corresponding coefficient satistfies $\vert A_{ik}\vert\geq \alpha$ using \Cref{assumption:ising}. We may then apply \Cref{cor:anti} using the Markov property to deduce that for any fixed choice of $a\in \mathbb{R}$ to be chosen shortly,
    \begin{equation*}
        \Pr_{X^t}\left(\left\vert \sum_{\ell\neq i,j} A_{i\ell}x_k-a\right\vert\geq \alpha \bigg\vert \mathcal{F}_{t'}\right)\geq c_{\ref{cor:anti}}\kappa^4\exp(-O(\gamma\lambda)).
    \end{equation*}
    We will let $\mathcal{E}_a$ denote this event.

    We will now compute the conditional expectation of $Z_t^{i,j}$ given any $\mathcal{F}_{t'}$. We have
    \begin{equation}
    \label{eq:decomp_good}
        \mathbb{E}[Z_t^{i,j}\vert \mathcal{F}_{t'}]=\mathbb{E}[Z_t^{i,j}\vert \mathcal{E}_a,\mathcal{F}_{t'}]\Pr\left(\mathcal{E}_a\vert \mathcal{F}_{t'}\right)+\mathbb{E}[Z_t^{i,j}\vert \mathcal{E}^c_a,\mathcal{F}_{t'}]\Pr\left(\mathcal{E}^c_a\vert \mathcal{F}_{t'}\right).
    \end{equation}
    By the choice of $\varepsilon$ and noting that the main term of \Cref{prop:main_stat} is nonnegative for any conditioning at time $t$, the second term can be lower bounded by at most $-O(d\varepsilon^9)$. We now show that for a suitable choice of $a\in \mathbb{R}$, the first term is noticeably positive of order $\varepsilon^8$, which in particular can  be made the dominant term under our choice of $\varepsilon>0$.

    To do so, suppose that for this choice of $\delta>0$, it holds that
    \begin{equation*}
        \left\vert g_i\left(\frac{\pi(X^{t,i\mapsto +1})}{\pi(X^{t,i\mapsto -1})}\right)-g_i\left(\frac{\pi(X^{t,\oplus j, i\mapsto +1})}{\pi(X^{t,\oplus j,i\mapsto -1})}\right)\right\vert\leq \delta.
    \end{equation*}
    In that case, if we define
\begin{equation*}
    z = \exp\left(2\sum_{k\neq i} A_{ik}X^t_k + 2h_i -2\vert A_{ij}\vert \right),
\end{equation*}
then we directly calculate the ratios using reversibility to see that this is equivalent to:
\begin{equation*}
    \left\vert g_i\left(z\right)-g_i\left(z\exp(4\vert A_{ij}\vert\right)\right\vert\leq \delta,
\end{equation*}
By applying stability as in \Cref{assumption:mc}, we may then  conclude that
    \begin{equation}
    \label{eq:lin_close}
\left\vert \exp\left(2\sum_{k\neq i} A_{ik}X^t_k+2h_i-2\vert A_{ij}\vert\right) - z^*\left(4\vert A_{ij}\vert\right)\right\vert\leq \eta(\delta)\leq \underbrace{\frac{c_{\ref{cor:bulk_neighbors}}\exp(-O(\lambda))}{\alpha}}_{\eta^*},
    \end{equation}
    where we use the assumption on $\delta$ in the second inequality. We will now claim that for a suitable choice of $a$, this does not occur on $\mathcal{E}_a$.

    To that end, we may assume first that $z^*(4\vert A_{ij}\vert)\geq \exp(-O(\lambda))$: since the first term is itself bounded below by $\exp(-2\lambda)$ under \Cref{assumption:ising}, the error bound of \Cref{eq:lin_close} would be violated if this failed. Rewriting \Cref{eq:lin_close}, this occurs only if 
    \begin{equation}
    \label{eq:approx_equality_int}
        2\sum_{k\neq i} A_{ik}X^t_k = -2h_i+2\vert A_{ij}\vert +\xi,
    \end{equation}
    where 
    \begin{equation*}
        \xi \in \left[\ln\left(z^*-\eta^*\right),\ln (z^*+\eta^*)\right]:=I.
    \end{equation*}
    The length of this interval is bounded by
    \begin{align*}
        \ln (z^*+\eta^*)-\ln (z^*-\eta^*)&=\ln\left(\frac{2\eta^*}{z^*-\eta^*}\right)\\
        &\leq c'\alpha,
    \end{align*}
    for some constant $c'>0$ that can be taken to zero with $c_{\ref{cor:bulk_neighbors}}>0$ (using our assumed lower bound on $z^*$), we conclude that if $c_{\ref{cor:bulk_neighbors}}>0$ is small enough, then this interval has length at most $\alpha$. Therefore if we define
    \begin{equation*}
        a = -2h_i+2\vert A_{ij}\vert,
    \end{equation*}
    the deviation event $\mathcal{E}_a$ by at least $2\alpha$ (after scaling) implies that
    \begin{equation*}
        2\sum_{k\neq i} A_{ik}X^t_k + 2h_i-2\vert A_{ij}\vert\not\in I,
    \end{equation*}
    contradicting \Cref{eq:approx_equality_int}. We can therefore conclude that on $\mathcal{E}_{a}$,
    \begin{equation}
        \left\vert g_i\left(\frac{\pi(X^{t,i\mapsto +1})}{\pi(X^{t,i\mapsto -1})}\right)-g_i\left(\frac{\pi(X^{t,\oplus j, i\mapsto +1})}{\pi(X^{t,\oplus j,i\mapsto -1})}\right)\right\vert\geq \delta.
    \end{equation}

    Returning to \Cref{eq:decomp_good}, we may conclude that
    \begin{align*}
        \mathbb{E}[Z_t^{i,j}\vert \mathcal{F}_{t'}]&\geq \varepsilon^8\kappa^4\delta^2\cdot \Pr\left(\mathcal{E}_a\vert \mathcal{F}_{t'}\right)- O(d\varepsilon^9)\\
        &\geq c_{\ref{cor:anti}}\varepsilon^8\kappa^8\delta^2\exp(-O(\gamma\lambda)) - O(d\varepsilon^9)\\
        &\geq c'\varepsilon^8\kappa^8\delta^2\exp(-O(\gamma\lambda))
    \end{align*}
    where we apply the probability lower bound and our choice of $\varepsilon$ to ensure the error term is dominated by the main term. 
\end{proof}

Putting \Cref{cor:non_neighbors} and \Cref{cor:bulk_neighbors}, we may conclude the following: there is an absolute constant $c_{\mathsf{ALG}}>0$ such that, if we set $\delta>0$ such that

\begin{equation*}
    \eta(\delta)\leq \frac{c_{\ref{cor:bulk_neighbors}}\exp(-O(\lambda))}{\alpha},
\end{equation*}
then for any $\varepsilon>0$ satisfying
\begin{equation*}
    \varepsilon\leq c_{\mathsf{ALG}}\kappa^8\delta^2\exp(-O(\gamma\lambda)),
\end{equation*}
it will hold that if $i\sim j,k$ for $j\neq k$, then for any $t,t'$ such that $t\geq t'-1$,
\begin{equation}
\label{eq:final_good}
    \mathbb{E}[Z^{i,j}_t\vert \mathcal{F}_{t'}]\geq c_{\mathsf{ALG}}\varepsilon^8\kappa^8\delta^2\exp(-O(\gamma\lambda)),
\end{equation}
while if $i\not \sim j$, then
\begin{equation}
\label{eq:final_bad}
    \mathbb{E}[Z^{i,j}_t\vert \mathcal{F}_{t'}]\leq \frac{1}{2}c_{\mathsf{ALG}}\varepsilon^8\kappa^8\delta^2\exp(-O(\gamma\lambda)).
\end{equation}

\subsection{Identifying Dense Edges}
\label{sec:bulk}

With these results in order, we may now turn to our main algorithm that will be able to efficiently identify the dense edges of the dependency graph. Our algorithm proceeds by evaluating the degree-8 cycle statistic as defined in \Cref{eq:cycle_def} at each time $\tau_{\ell}:=2\ell$ for $\ell\in \mathbb{N}$.

\begin{algorithm}[H]
    \caption{$\widehat{E}=$ FindBulkEdges$(\alpha,d,\lambda,\kappa,\gamma,\beta)$}
    \label{alg:cycle_stats}
    \LinesNumbered

    Let $\alpha, d,\lambda,\kappa,\gamma$ be as in \Cref{assumption:ising} and \Cref{assumption:mc}.
    
    Set
    \begin{gather*}
        \delta = \min\left\{\eta^{-1}\left(\frac{c_{\ref{cor:bulk_neighbors}}\exp(-O(\lambda))}{\alpha}\right),\delta_0\right\}\\
        \varepsilon = \frac{c_{\mathsf{ALG}}\kappa^8\delta^2\exp(-O(\gamma\lambda))}{d}\\
        T = 2\cdot \left\lceil\frac{2000\exp(O(\lambda\gamma))\log(n/\beta)}{c_{\mathsf{ALG}}^2\varepsilon^{16}\kappa^{16}\delta^4} \right\rceil
    \end{gather*}
   
    Observe random process $(X_t)_{t=0}^T$ and $\Pi'_k(T)$ for all $k\in [n]$.

    \For{each ordered pair $(i,j)\in [n]^2$}
        {Add $(i,j)$ to $\widehat{E}$ if 
        \begin{equation*}
            \frac{1}{(T/2)}\sum_{\ell=1}^{T/2} Z^{i,j}_{2t}\geq \frac{3}{4}c_{\mathsf{ALG}}\varepsilon^8\kappa^8\delta^2.
        \end{equation*}}
\end{algorithm}

While we have stated this result in an abstract form depending on the parameters of \Cref{assumption:ising} and \Cref{assumption:mc}, note that if $\eta(a)\geq a^{\Omega(1)}$, then the runtime of this algorithm is 
\begin{equation*}
    O(Tn^2) = \mathsf{poly}\left(\exp(\lambda\gamma),\frac{1}{\kappa},d,\frac{1}{\alpha},\frac{1}{\delta_0}\right)\cdot n^2\log(n/\beta).
\end{equation*}
As a consequence of \Cref{prop:gd_assumption} and \Cref{prop:metro_assumption}, this is indeed the case for both the Glauber dynamics and the site-consistent Metropolis chain, and the bounds reduce to 
\begin{equation*}
    \mathsf{poly}\left(\exp(\lambda),d,\frac{1}{\alpha}\right)\cdot n^2\log(n/\beta)
\end{equation*}
and 
\begin{equation*}
    \mathsf{poly}\left(\exp(\lambda),\frac{1}{r_+r_-},d,\frac{1}{\alpha}\right)\cdot n^2\log(n/\beta),
\end{equation*}
respectively.

The algorithm has the following guarantees:
\begin{theorem}
\label{thm:alg_main}
    Under \Cref{assumption:ising} and \Cref{assumption:mc}, with probability at least $1-\beta$, the following holds for all pairs $(i,j)\in [n]^2$:
    \begin{itemize}
    \item Suppose that $(i,j)$ is a dense edge as in \Cref{def:dense_edges}. Then \Cref{alg:cycle_stats} correctly outputs $(i,j)\in \widehat{E}$.
    \item Suppose that $i\not\sim j$. \Cref{alg:cycle_stats} correctly does not output $(i,j)\in \widehat{E}$.
    \end{itemize}
In particular, $\widehat{E}\subseteq E$, and moreover, the set of edges in $E\setminus \widehat{E}$ must form a (not necessarily perfect) matching among the set $\mathcal{O}$ of isolated sites in $\widehat{E}$ i.e. $\mathcal{O}=\{i\in[n]:\mathsf{deg}_{\widehat{E}}(i)=0\}$.
\end{theorem}
\begin{proof}
    The first part is a consequence of the Azuma-Hoeffding inequality applied to the martingale difference sequence
    \begin{equation*}
        Z^{i,j}_{2t} - \mathbb{E}\left[Z^{i,j}_{2t}\vert \mathcal{F}_{2t-1}\right].
    \end{equation*}
    Note that this random variable lies in the interval $[-4,4]$ surely and this is indeed a martingale difference since $Z_{2(t-1)}^{i,j}$ is measurable with respect to $\mathcal{F}_{2t-1}$ so long as $8\varepsilon<1$.

    Suppose that $(i,j)$ is a dense edge, and that there exists some $k\neq j$ such that $i\sim k$. By \Cref{eq:final_good}, we know each conditional expectation is at least $c_{\mathsf{ALG}}\varepsilon^8\kappa^8\delta^2\exp(-O(\gamma\lambda))$. We may thus apply the Azuma-Hoeffding inequality with error probability $\delta/n^2$ and deviation $\frac{1}{4}c_{\mathsf{ALG}}\varepsilon^8\kappa^8\delta^2\exp(-O(\gamma\lambda))$ to deduce that with probability at least $1-\delta/n^2$,
    \begin{align*}
        \frac{1}{T/2}\sum_{t=1}^{T/2}Z^{i,j}_{2t}&\geq c_{\mathsf{ALG}}\varepsilon^8\kappa^8\delta^2\exp(-O(\gamma\lambda))-\frac{1}{4}c_{\mathsf{ALG}}\varepsilon^8\kappa^8\delta^2\exp(-O(\gamma\lambda)\\
        &\geq \frac{3}{4}c_{\mathsf{ALG}}\varepsilon^8\kappa^8\delta^2\exp(-O(\gamma\lambda)),
    \end{align*}
    and therefore \Cref{alg:cycle_stats} will correctly identify the adjacency $i\sim j$. The same holds true for the statistics $Z^{j,i}_t$ if instead $j$ has degree at least $2$ in the dependency graph. Since any dense edge must have a vertex of degree $2$, we deduce that $i\sim j$ will correctly be outputted.

    An identical argument using the Azuma-Hoeffding inequality holds for the case $i\not\sim j$, but instead using \Cref{eq:final_bad} to upper bound the conditional probabilities. Therefore, with probability at least $1-\delta/n^2$, the algorithm again correctly does not output $(i,j)$ in this case. By a union bound over the $n^2$ pairs of sites, the algorithm thus recovers the stated edges and never incorrectly outputs an adjacency.

    For the final claim, observe that an adjacency between $i$ and $j$ is always detected in the case $i\sim j$ and \emph{either} $i$ or $j$ has degree at least $2$ in $E$. Since the algorithm does not output any false edges, any dependencies in $E$ that the algorithm fails to identify must be between two sites that are of degree-one in $E$, and therefore isolated in $\widehat{E}$. This exactly means that the remaining dependencies must form a (not necessarily perfect) matching among sites in $\mathcal{O}$, the set of isolated sites in $\widehat{E}$.
\end{proof}

\subsection{Recovering Matchings}
\label{sec:matchings}

To finish the structure learning algorithm, recall from \Cref{thm:alg_main} and \Cref{fact:dense_edge_structure} that the only unidentified dependencies in $E$ must form a (not necessarily perfect) matching among sites that are \emph{isolated} in the current dependency graph $\widehat{G} = ([n],\widehat{E})$; moreover, these sites are independent of any site adjacent to an edge in $\widehat{E}$ since all dense edges are found. In this section, we provide a simple recovery algorithm that computes all remaining edges in $E$ that must form a matching.

The main idea of this algorithm is quite natural: since the remaining edges must form a matching among sites with no other dependencies (including to sites outside $\mathcal{O}$), the set $\mathcal{O}$ must form an Ising model with isolated edges. In particular, this Ising model is simply a product of independent subsets that each has size either $1$ or $2$, depending on whether the site belongs to a (unique) edge or not. We will first argue that this system trivially has a noticeable spectral gap depending mildly on \Cref{assumption:ising} and \Cref{assumption:mc} (with no dependence on $n$). For all $i,j\in \mathcal{O}$, we can then compute good estimates of the stationary probabilities in $\pi$ for each possible value of $(x_i,x_j)$ using the empirical time-averages; these will concentrate well for all pairs thanks to the Chernoff-type bound of \Cref{thm:lezaud}~\cite{lezaud2001chernoff}. We can therefore obtain good estimators of the conditional probability that $X_i=+1$ in $\pi$ depending on the value of $X_j\in \{-1,1\}$. If these variables do not form an edge in the matching, these structural results imply they are independent and so the conditional probabilities will not differ; if they do, then \Cref{assumption:ising} and \Cref{fact:km_sigmoid} will establish an explicit quantitative separation. We can thus threshold the difference in these empirical approximations.

We now carry out this plan. First, we can easily see that the (independent) sub-system of $\pi$ induced by $\mathcal{O}$ has a large spectral gap:

\begin{lemma}
\label{lem:spectral_gap}
    Suppose that the assumptions and conclusions of \Cref{thm:alg_main} holds. Then the distribution on $\mathcal{O}$ is simply the restriction of the Ising model to $\mathcal{O}$ by independence, and moreover, the spectral gap of the generator of the induced Markov chain restricted to $\mathcal{O}$ has spectral gap at least $c\kappa\exp(-O(\lambda))$ for some constant $c>0$.
\end{lemma}
\begin{proof}
    The independence statement has already been shown by \Cref{thm:alg_main} since there are no edges in $E$ between $\mathcal{O}$ and $[n]\setminus \mathcal{O}$. Moreover, since the dependence structure of $\mathcal{O}$ is simply a matching, the induced single-site Markov chain restricted to $\mathcal{O}$ is a product chain with independent sub-systems of size at most $2$. Each of these sub-systems has spectral gap at least $c\kappa\exp(-O(\lambda))$ by \Cref{fact:spectral_gap} since they are of constant size and \Cref{assumption:ising} and \Cref{assumption:mc} furnishes the lower bounds on transition probabilities and stationary probabilities of the subsystem. By standard facts about product chains, the spectral gap of the product chain is simply the minimum of the spectral gaps of each component (see e.g. Corollary 12.13 of~\cite{levin2017markov}), completing the proof.
\end{proof}

Next, we show that we can accurately compute all conditional probabilities to high-accuracy of the spin-spin probabilities of $\pi$ restricted to $\mathcal{O}$ using the time-average along a small trajectory.

\begin{lemma}
\label{lem:emp_ests}
    Suppose that the assumptions and conclusions of \Cref{thm:alg_main} holds. Then for any $\varepsilon>0,\beta<1$, with probability at least $1-\beta$, it holds simultaneously for all $i,j\in \mathcal{O}$ and $x_i,x_j\in \{-1,1\}$ that
    \begin{equation*}
        \left\vert \frac{1}{T}\int_0^T\mathbf{1}\{X^t_i=x_i,X^t_j=x_j\}\mathrm{d}t-\pi\left(X_i=x_i,X_j=x_j\right)\right\vert\leq \varepsilon,
    \end{equation*}
    so long as 
    \begin{equation*}
        T\geq \frac{C_{\ref{lem:emp_ests}}\exp(O(\lambda))(\lambda+\log(n/\beta))}{\kappa \varepsilon^2}.
    \end{equation*}
\end{lemma}
\begin{proof}
    For each pair $(i,j)\in \mathcal{O}^2$ and values of $x_i,x_j\in \{-1,1\}^2$, let $f_{(i,j),x_i,x_j}(X)=\mathbf{1}\{X_i=x_i,X_j=x_j\}$. Consider the sub-system given by $\{i,j\}\cup \mathcal{N}(i)\cup \mathcal{N}(j)$, which we know has size at most four. This independent sub-system is of size $O(1)$, so the minimum probability under $\pi$ restricted to these sites is at least $\exp(-O(\lambda))$ under \Cref{assumption:ising}. Therefore, we may directly apply \Cref{cor:leuzard_all} using the spectral gap estimate given by \Cref{lem:spectral_gap} for all of these at most $m=4n^2$ functions simultaneously to obtain the desired result.
\end{proof}

We may now use this result to analyze a simple thresholding algorithm to detect these correlations. We first provide a lower bound on spin-spin correlations when $i,j\in\mathcal{O}$ form an edge:

\begin{lemma}
\label{lem:spin_separation}
    Suppose that the assumptions and conclusions of \Cref{thm:alg_main} holds, and suppose that $i\sim j$ are unique neighbors in $\mathcal{O}$. Then
    \begin{equation*}
        \left\vert\Pr_{\pi}\left(X_i=+1\vert X_j=+1\right)-\Pr_{\pi}\left(X_i=+1\vert X_j=-1\right)\right\vert\geq c_{\ref{lem:spin_separation}}\exp(-2\lambda)\min\{1,8\alpha\}.
    \end{equation*}

    Conversely, if $i\not\sim j$, then trivially
    \begin{equation*}
        \left\vert\Pr_{\pi}\left(X_i=+1\vert X_j=+1\right)-\Pr_{\pi}\left(X_i=+1\vert X_j=-1\right)\right\vert=0.
    \end{equation*}
\end{lemma}
\begin{proof}
By our previous structural results, since $i\sim j$, the restricted Ising model satisfies
\begin{equation*}
    \pi(x_i,x_j) \propto \exp\left(A_{ij}x_ix_j+h_ix_i+h_jx_j \right),
\end{equation*}
where we know that $\vert A_{ij}\vert\geq \alpha$.
As a result, the conditional probabilities of $X_i$ given the value of $X_j$ under $\pi$ are given by 
\begin{equation*}
    \Pr_\pi\left(X_i=+1\vert X_j\right) = \sigma\left(2A_{ij}X_{ij}+2h_i\right).
\end{equation*}
We can now use \Cref{fact:km_sigmoid} to deduce that
\begin{align*}
    \left\vert\Pr_{\pi}\left(X_i=+1\vert X_j=+1\right)-\Pr_{\pi}\left(X_i=+1\vert X_j=-1\right)\right\vert&=\vert \sigma\left(2A_{ij}+2h_i\right)-\sigma\left(-2A_{ij}+2h_i\right)\vert\\
    &\geq c\exp(-2\lambda)\min\{1,4\vert A_{ij}\vert\}\\
    &\geq c\exp(-2\lambda)\min\{1,4\alpha\}.\qedhere
\end{align*}
\end{proof}

\begin{corollary}
\label{cor:emp_accurate}
    There is a small enough constant $c_{\mathsf{THR}}>0$ such that the following holds. Suppose that the assumptions and conclusions of \Cref{thm:alg_main} holds. If the good event of \Cref{lem:emp_ests} holds with $\varepsilon=c_{\mathsf{THR}}\exp(-O(\lambda))\min\{1,8\alpha\}$, then for all $i,j\in\mathcal{O}$, and each value of $x_j\in \{-1,+1\}$,
    \begin{align*}
        \left\vert\frac{\frac{1}{T}\int_0^T\mathbf{1}\{X^t_i=+1,X^t_j=x_j\}\mathrm{d}t}{\sum_{x_i}\frac{1}{T}\int_0^T\mathbf{1}\{X^t_i=x_i,X^t_j=x_j\}\mathrm{d}t}-\Pr_\pi\left(X_i=+1\vert X_j=x_j\right)\right\vert&\\
        \leq \frac{c_{\ref{lem:spin_separation}}}{4}\exp(-2\lambda)\min\{1,8\alpha\}&.
    \end{align*}

    In particular, if $i\sim j$, then
    \begin{align*}
\left\vert\frac{\frac{1}{T}\int_0^T\mathbf{1}\{X^t_i=+1,X^t_j=+1\}\mathrm{d}t}{\sum_{x_i}\frac{1}{T}\int_0^T\mathbf{1}\{X^t_i=x_i,X^t_j=+1\}\mathrm{d}t}-\frac{\frac{1}{T}\int_0^T\mathbf{1}\{X^t_i=+1,X^t_j=-1\}\mathrm{d}t}{\sum_{x_i}\frac{1}{T}\int_0^T\mathbf{1}\{X^t_i=x_i,X^t_j=-1\}\mathrm{d}t}\right\vert&\\
\geq \frac{3c_{\ref{lem:spin_separation}}}{4}\exp(-2\lambda)\min\{1,8\alpha\}&,
    \end{align*}
    while if $i\not\sim j$,
    \begin{align*}
\left\vert\frac{\frac{1}{T}\int_0^T\mathbf{1}\{X^t_i=+1,X^t_j=+1\}\mathrm{d}t}{\sum_{x_i}\frac{1}{T}\int_0^T\mathbf{1}\{X^t_i=x_i,X^t_j=+1\}\mathrm{d}t}-\frac{\frac{1}{T}\int_0^T\mathbf{1}\{X^t_i=+1,X^t_j=-1\}\mathrm{d}t}{\sum_{x_i}\frac{1}{T}\int_0^T\mathbf{1}\{X^t_i=x_i,X^t_j=-1\}\mathrm{d}t}\right\vert&\\
\leq \frac{c_{\ref{lem:spin_separation}}}{4}\exp(-2\lambda)\min\{1,8\alpha\}&.
\end{align*}
\end{corollary}
\begin{proof}
    Note that each empirical ratio provides the natural empirical estimator of each conditional probability. The first inequality follows by observing that since we obtain $\varepsilon$-accurate estimates to the numerator and denominator on the good event of \Cref{lem:emp_ests}, straightforward algebra using the fact that the true ratios under $\pi$ are lower bounded by $\exp(-O(\lambda))$ yields the desired deviation. The last two inequalities are a consequence of the first by using the triangle inequality along with \Cref{lem:spin_separation}.
\end{proof}

\begin{algorithm}[H]
    \caption{$\widehat{E}'=$ FindMatching$(\alpha,\lambda,\kappa,\beta,\mathcal{O})$}
    \label{alg:matching}
    \LinesNumbered

    Let $\alpha,\lambda,\kappa$ be as in \Cref{assumption:ising} and \Cref{assumption:mc} and $\mathcal{O}$ be the set of isolated vertices from $\widehat{E}$, the output of \Cref{thm:alg_main}.
    
    Set
    \begin{gather*}
        \varepsilon=c_{\mathsf{THR}}\exp(-O(\lambda))\min\{1,8\alpha\}\\
        T= \frac{C_{\ref{lem:emp_ests}}\exp(O(\lambda))(\lambda+\log(n/\beta))}{\kappa \varepsilon^2}
    \end{gather*}

    Observe random process $(X_t)_{t=0}^T$ and $\Pi'_k(T)$ for all $k\in \mathcal{O}$.

    \For{each pair $(i,j)\in \mathcal{O}^2$}
        {For each value of $x_i,x_j\in \{-1,1\}$, compute
        \begin{equation*}
            p^{i,j}_{x_1,x_2} = \frac{1}{T}\int_0^T\mathbf{1}\{X^t_i=x_i,X^t_j=x_j\}\mathrm{d}t.
        \end{equation*}

        Add $(i,j)$ to $\widehat{E}'$ if 
        \begin{equation*}
            \left\vert\frac{p^{i,j}_{1,1}}{p^{i,j}_{1,1}+p^{i,j}_{-1,1}}-\frac{p^{i,j}_{1,-1}}{p^{i,j}_{1,-1}+p^{i,j}_{-1,-1}}\right\vert\geq \frac{3c_{\ref{lem:spin_separation}}}{4}\exp(-2\lambda)\min\{1,8\alpha\}.
        \end{equation*}}
\end{algorithm}

\begin{theorem}
\label{thm:main_match}
    Under the assumptions and conclusions of \Cref{thm:alg_main}, with probability at least $1-\beta$, the output $\widehat{E}'$ of \Cref{alg:matching} is precisely $E\setminus \widehat{E}$. Moreover, the running time of the algorithm is at most
    \begin{equation*}
        O(Tn^2) = \mathsf{poly}\left(\exp(\lambda),1/\alpha,1/\kappa\right)\cdot n^2\log(n/\beta).
    \end{equation*}
\end{theorem}
\begin{proof}
    The statistics in \Cref{alg:matching} can be computed for each fixed $i,j$ and $x_i,x_j$ in time $O(T)$ by a linear scan of $\Pi'(i)$ and $\Pi'(j)$ which will each have length $O(T)$, since the integrals are piecewise constant except at these flip times. Since there are at most $4n^2$ such statistics to compute , the claim follows. The correctness of the algorithm is an immediate consequence of \Cref{cor:emp_accurate}.
\end{proof}

\section{Parameter Learning}
\label{sec:parameter}

By the results of the previous section, we may assume that we have access to the true dependency graph $E$ of the Ising model. In this section, we provide an algorithm that observes a trajectory for time $T=\widetilde{O}(2^d\log(n))$ and that runs in time  $n\cdot T$ time, hiding parameter dependencies, that gives additive approximations of the actual coefficients $A_{ij}$.

In \Cref{sec:config_stats}, we define the natural empirical estimators for each $\mathsf{P}_i(\bm{x},\bm{x}^{\oplus i})$, where we may now assume that $\bm{x}\in \{-1,1\}^{\{i\}\cup \mathcal{N}(i)}$ since the dependence graph is known. We will use our previous structural results to show that after $T$ time as above, for all $i$ and $j\in \mathcal{N}(i)$ with high probability, we will obtain a large number of samples for \emph{some} configuration $\bm{x}\in \{-1,1\}^{\mathcal{N}(i)\setminus \{j\}}$ with all four possible values of $x_i,x_j$. Moreover, these estimates will be fairly accurate with high probability, and therefore we can back out $A_{i,j}$ by reversibility as described in \Cref{sec:overview}. We give the final construction in \Cref{sec:parameter_guarantees}.

\subsection{Moments and Concentration of Local Configurations}
\label{sec:config_stats}

Before beginning, we will define some convenient notation. First, we will define
\begin{equation*}
    S_i:=\{i\}\cup \mathcal{N}(i).
\end{equation*}
As alluded to, in a slight abuse of notation, for any configuration $X\in\{-1,1\}^n$ such that $X_{S_i}=\bm{x}\in \{-1,1\}^{S_i}$, we will write
\begin{equation*}
\mathsf{P}_i(\bm{x},\bm{x}^{\oplus i})= \mathsf{P}_i(X,X^{\oplus i}),
\end{equation*}
since this transition probability only depends on the neighbors of $i$. We will also slightly abuse notation by writing for $X\in \{-1,1\}^n$ and $\bm{x}\in \{-1,1\}^{S_i}$
\begin{equation*}
    H_t(X,\bm{x})=\sum_{Y\in \{-1,1\}^n:Y_{S_i}=\bm{x}} H_t(X,Y),
\end{equation*}
to denote the probability that the coordinates of $S_i$ of the final configuration of the Markov chain started at $X$ for $t$ units of time is equal to $\bm{x}$. Note that clearly
\begin{equation*}
    \sum_{\bm{x}\in \{-1,1\}^{S_i}} H_t(X,\bm{x})=1.
\end{equation*}

Given $\varepsilon>0$, we now define the following two statistics for any $t\geq 0$ and any $\bm{x}\in \{-1,1\}^{S_i}$:
\begin{gather*}
    Z^{\bm{x}}_t:=\mathbf{1}\left\{X^{2t}_{S_i}=\bm{x}\right\}\\
    Z^{\bm{x},i}_{t}=\mathbf{1}\left\{X^{2t}_{S_i}=\bm{x} \text{ and } \vert\Pi'_i(t,t+\varepsilon)\vert=1\right\}.
\end{gather*}
In words, $Z^{\bm{x}}_t$ denotes the event that at time $t$, the configuration on the sites in $S_i$ equal $\bm{x}$. Similarly, $Z^{\bm{x},i}_{t}$ is the indicator that the same event holds \emph{and} site $i$ flips exactly once in the interval $[t,t+\varepsilon]$.

We now establish the following simple moment identities:
\begin{lemma}
\label{lem:first_moment}
    For any $t\in \mathbb{N}$ and $\bm{x}\in \{-1,1\}^{S_i}$,
    \begin{equation*}
        \mathbb{E}[Z^{\bm{x}}_{2t}\vert \mathcal{F}_{2t-1}]= H_1(X^{2t-1},\bm{x}).
    \end{equation*}
    Moreover, if $\varepsilon<c$ for some small constant, then
    \begin{equation*}
        \mathbb{E}[Z^{\bm{x},i}_{2t}\vert \mathcal{F}_{2t-1}]=\varepsilon H_1(X^{2t-1},\bm{x})\left(\mathsf{P}_i(\bm{x},\bm{x}^{\oplus i})+O(\varepsilon d)\right).
    \end{equation*}
\end{lemma}
\begin{proof}
    The first identity is just a restatement of the definition of $H_1(X,\bm{x})$ after applying the Markov property to start the chain at $X^{2t-1}$.

    For the second identity, we can write
    \begin{align*}
        \mathbb{E}[Z^{\bm{x},i}_{2t}\vert \mathcal{F}_{2t-1}]&=\mathbb{E}\left[\mathbb{E}\left[Z^{\bm{x},i}_{2t}\vert \mathcal{F}_{2t}\right]\bigg\vert \mathcal{F}_{2t-1}\right]\\
        &=\mathbb{E}\left[\mathbf{1}\{X^{2t}_{S_i}=\bm{x}\}\cdot \Pr\left(\mathcal{E}^{2t}_i\vert X^{2t}\right)\vert \mathcal{F}_{2t-1}\right],
    \end{align*}
    where we use the same notation as in \Cref{eq:flip_event} and applied the Markov property. But on the event that $X^{2t}_{S_i}=\bm{x}$, we may apply \Cref{prop:event_probs} to assert that
    \begin{equation*}
        \mathbf{1}\{X^{2t}_{S_i}=\bm{x}\}\cdot \Pr\left(\mathcal{E}^{2t}_i\vert X^{2t}\right)=\mathbf{1}\{X^{2t}_{S_i}=\bm{x}\}\cdot \varepsilon\left(\mathsf{P}_i(\bm{x},\bm{x}^{\oplus i})+O(\varepsilon d)\right).
    \end{equation*}
    We may pull out this factor and then take expectations over the indicator function, applying the first identity.
\end{proof}

We can now turn to our main statistics. For $0<\varepsilon<c/d$ for a small enough constant, and $T\in \mathbb{N}$ to be chosen later, and for all $\bm{x}\in \{-1,1\}^{S_i}$, we define the following random variables:
\begin{gather*}
    N_{\bm{x}} = \sum_{t=1}^T Z^{\bm{x}}_{2t},\\
    N_{\bm{x},i} = \sum_{t=1}^T Z^{\bm{x},i}_{2t}.
\end{gather*}

We also define the empirical estimates for flip rates by
\begin{equation*}
    \widehat{p(\bm{x},i)} = \frac{N_{\bm{x},i}}{\varepsilon N_{\bm{x}}}.
\end{equation*}
Our goal will be to show that when $T$ is chosen suitably, this empirical estimator for flip rates will be a good for \emph{some values of $\bm{x}$} that we can determine.

We now establish a suitable form of \emph{pathwise} concentration for all of these simultaneously:

\begin{proposition}
\label{prop:concentration}
    There is an absolute constant $C>0$ such that the following holds. For any $\beta>0,\varepsilon<c$ and $T\in \mathbb{N}$ as above, the following holds with probability at least $1-\beta$: let 
    \begin{equation}
    \label{eq:xi_def}
        \xi=C\sqrt{d\log(\log(T)/\beta))}.
    \end{equation}
 For all $\bm{x}\in \{-1,1\}^{S_i}$ simultaneously:

    \begin{gather}
        \left\vert N_{\bm{x}}-\sum_{t=1}^T H_1(X^{2t-1},\bm{x})\right\vert\leq \max\left\{\sqrt{\sum_{t=1}^T H_1(X^{2t-1},\bm{x})},\xi\right\}\cdot \xi\\
        \left\vert N_{\bm{x},i}-\varepsilon\left(\mathsf{P}_i(\bm{x},\bm{x}^{\oplus i})+O(\varepsilon d)\right) \sum_{t=1}^T  H_1(X^{2t-1},\bm{x})\right\vert\leq \max\left\{\sqrt{\sum_{t=1}^T H_1(X^{2t-1},\bm{x})},\xi\right\}\cdot \xi.
    \end{gather}
\end{proposition}
\begin{proof}
    For each $\bm{x}$, define the martingale difference sequences for $\ell=1,\ldots, T$:
    \begin{gather*}
        \sum_{t=1}^{\ell} Z^{\bm{x}}_{2t}-\sum_{t=1}^{\ell} \mathbb{E}[Z^{\bm{x}}_{2t}\vert\mathcal{F}_{2t-1}],\\
        \sum_{t=1}^{\ell} Z^{\bm{x},i}_{2t}-\sum_{t=1}^{\ell} \mathbb{E}[Z^{\bm{x},i}_{2t}\vert\mathcal{F}_{2t-1}].
    \end{gather*}
    Note that these indeed form martingale difference sequences with respect to appropriate filtrations since $Z_{2(t-1)}$ is $\mathcal{F}_{2t-1}$-measurable as $\varepsilon<1$. Moreover, note that
    \begin{gather*}
        \mathsf{Var}(Z^{\bm{x}}_{2t}\vert \mathcal{F}_{2t-1})\leq \mathbb{E}[Z^{\bm{x}}_{2t}\vert\mathcal{F}_{2t-1}],\\
        \mathsf{Var}(Z^{\bm{x},i}_{2t}\vert \mathcal{F}_{2t-1})\leq \mathbb{E}[Z^{\bm{x},i}_{2t}\vert\mathcal{F}_{2t-1}],
    \end{gather*}
    since each random variable lies in $\{0,1\}$, so the conditional variance is bounded by the conditional mean. Therefore, the sum of conditional variances is bounded by the sum of conditional means, which are given by \Cref{lem:first_moment} (dropping the $\varepsilon$ factor in the latter for simplicity).

    We may then directly apply a version of Freedman's martingale inequality as stated in \Cref{prop:freedman} with error probability $\beta/2^{d+2}$ and take a union bound over $\bm{x}$ and whether or not $i$ flips; note there are at most $2^{d+2}$ such events we are computing. The desired concentration inequalities then follow immediately from \Cref{prop:freedman} using our definition of $\xi$.
\end{proof}

\begin{corollary}
\label{cor:mult_ests}
    Let $\varepsilon\leq c\delta\kappa/d$ for a small enough constant $c>0$ and any constant $\delta<1$. Under the conditions and good event of \Cref{prop:concentration}, suppose that $\bm{x}\in \{-1,1\}^{S_i}$ is such that 
    \begin{equation}
    \label{eq:num_samples}
        \sum_{t=1}^T H_1(X^{2t-1},\bm{x})\geq \frac{C\xi^2}{\varepsilon^2\delta^2\kappa^2},
    \end{equation}
    where $\xi$ is as in \Cref{eq:xi_def}
    Then it holds that
    \begin{equation*}
        \frac{\widehat{p(\bm{x},i)}}{\mathsf{P}_i(\bm{x},\bm{x}^{\oplus i})} \in [1-\delta,1+\delta].
    \end{equation*}

    The same conclusion holds for any $\bm{x}$ such that  $N_{\bm{x}}$ is at least the same quantity, up to a change of constants.
\end{corollary}
\begin{proof}
    For the choice of $\bm{x}$ satisfying the conditions, write
    \begin{equation*}
        A :=\sum_{t=1}^T H_1(X^{2t-1},\bm{x}).
    \end{equation*}
    Note that by assumption, $A\geq \xi^2$, so the maximum in the conclusion of \Cref{prop:concentration} is attained by $A$.

    First, note that by the choice of $\varepsilon>0$, we can assume that
    \begin{equation*}
        \mathsf{P}_i(\bm{x},\bm{x}^{\oplus i})+O(\varepsilon d) = (1\pm \delta/100) \mathsf{P}_i(\bm{x},\bm{x}^{\oplus i}),
    \end{equation*}
    since the transition probability is at least $\kappa$.
    
    By the conclusion of \Cref{prop:concentration}, we can compute that
    \begin{align*}
        \widehat{p(\bm{x},i)} &= \frac{N_{\bm{x},i}}{\varepsilon N_{\bm{x}}}\\
        &=\frac{(1\pm \delta/100)\mathsf{P}_i(\bm{x},\bm{x}^{\oplus i})A + O(\sqrt{A}\xi/\varepsilon)}{ A+ O(C\sqrt{A}\xi)}\\
        &=\frac{(1\pm \delta/100)\mathsf{P}_i(\bm{x},\bm{x}^{\oplus i}) + O(\xi/(\varepsilon\sqrt{A}))}{1+O(C\xi/\sqrt{A})}\\
        &=\left((1\pm \delta/100)\mathsf{P}_i(\bm{x},\bm{x}^{\oplus i}) + O(\xi/(\varepsilon\sqrt{A}))\right)\cdot \left(1+ O\left(\xi/\sqrt{A})\right)\right)\\
        &=(1\pm \delta/100)\mathsf{P}_i(\bm{x},\bm{x}^{\oplus i})+  O\left(\frac{\xi}{\varepsilon\sqrt{A}}\right).
    \end{align*}

    Therefore, if $A$ is such that
    \begin{equation*}
        A\geq O\left(\frac{\xi^2}{\delta^2\varepsilon^2\kappa^2}\right),
    \end{equation*}
    as assumed,
    then the error term can be bounded by $\delta\kappa/2 $, completing the proof. The same argument holds for $N_{\bm{x}}$ noting that if $N_{\bm{x}}$ satisfies the bound with a slightly larger constant, then so does $A$ by the deviation bounds of \Cref{prop:concentration}. 
\end{proof}

Finally, we can show that so long as $T=\widetilde{O}\left(\frac{\exp(O(\gamma\lambda))2^d\log(1/\beta)}{\delta^2\kappa^6}\right)$, then we can ensure that we have enough samples with probability $1$:

\begin{corollary}
    For $\delta<1$, let $\varepsilon=c\delta\kappa/d$ for a small enough constant $c$. For all $j\in \mathcal{N}(i)$ so long as 
    \begin{equation} 
    \label{eq:total_samples}
    T=\widetilde{O}\left(\frac{\exp(O(\gamma\lambda))2^d\log(1/\beta)}{\delta^4\kappa^8}\right),
    \end{equation}
    then there exists a $\bm{x}\in \{-1,1\}^{S_i\setminus\{i,j\}}$ such that for any setting of $x_i,x_j\in \{-1,1\}$, $\sum_{t=1}^T H_1(X^{2t-1},(\bm{x},x_i,x_j))$ exceeds the bound of \Cref{eq:num_samples}. Under the conditions and good event of \Cref{prop:concentration}, for any such $\bm{x}$ where $N_{\bm{x}}$ exceeds this bound at this time $T$, it holds that for each $\bm{z} = (\bm{x},x_i,x_j)$ that 
    \begin{equation}
    \label{eq:all_vals_good}
        \frac{\widehat{p(\bm{z},i)}}{\mathsf{P}_i(\bm{z},\bm{z}^{\oplus i})} \in [1-\delta,1+\delta].
    \end{equation}
\end{corollary}
\begin{proof}
    Fix any $j\in \mathcal{N}$. We first show that for the stated value of $T$ in \Cref{eq:total_samples}, there must exist such a $\bm{x}\in \{-1,1\}^{S_i\setminus\{i,j\}}$ such that \Cref{eq:num_samples} holds for each setting of $x_i,x_j$. But this is an immediate consequence of \Cref{cor:stable_random} by lower bounding by the distribution $Q_{ij}'$ as given there: since we know that
    \begin{equation*}
        \sum_{\bm{y}\in \{-1,1\}^{S_i\setminus\{i,j\}}}Q_{ij}'(\bm{y})\geq c,
    \end{equation*}
    and by construction the conditional sub-probabilities of each setting of $x_i,x_j$ are within a factor of $\exp(-O(\gamma\lambda))\kappa^4$ of each other in $Q'_{i,j}$ for any $\bm{x}$, it follows that if
    \begin{equation*}
T=\widetilde{O}\left(\frac{\exp(O(\gamma\lambda))2^d\log(1/\beta)}{\delta^4\kappa^8}\right),
    \end{equation*}
    an averaging argument implies that there must exist the desired $\bm{x}\in \{-1,1\}^{S_i\setminus\{i,j\}}$ satisfying \Cref{eq:num_samples} for each setting of $x_i,x_j$. 

    In particular, we can then apply the guarantee of \Cref{cor:mult_ests} to conclude that there exists such a $\bm{x}$, and for any such $\bm{x}$ and values of $x_i,x_j$, the desired ratio bound holds for the corresponding $\bm{z}$. Moreover, by \Cref{cor:mult_ests}, any such pair corresponding to $\bm{z}$ with $N_{\bm{z}}$ exceeding \Cref{eq:num_samples} will satisfy the ratio bound since the corresponding sum of conditional probabilities will be sufficiently large.
\end{proof}

\subsection{Final Algorithmic Guarantees}
\label{sec:parameter_guarantees}

We can now conclude our parameter learning results directly. Recall \Cref{eq:parameter}, which asserts that
\begin{equation}
\label{eq:parameter_final}
    \exp\left(4A_{ij}\right) = \frac{\mathsf{P}_i(\bm{x}^{i\mapsto -1,j\mapsto -1},\bm{x}^{i\mapsto +1,j\mapsto -1})/\mathsf{P}_i(\bm{x}^{i\mapsto +1,j\mapsto -1},\bm{x}^{i\mapsto -1,j\mapsto -1}) }{\mathsf{P}_i(\bm{x}^{i\mapsto -1,j\mapsto +1},\bm{x}^{i\mapsto +1,j\mapsto +1})/\mathsf{P}_i(\bm{x}^{i\mapsto +1,j\mapsto +1},\bm{x}^{i\mapsto -1,j\mapsto +1})}.
\end{equation} 

If $\bm{y}\in \{-1,1\}^{S_i\setminus \{i,j\}}$ and each $x_i,x_j\in \{-1,1\}$ satisfying the guarantee of \Cref{eq:all_vals_good}, then the right side of \Cref{eq:parameter_final} can be estimated to multiplicative accuracy $1+O(\delta)$. Taking natural logs and dividing by 4 thus obtains an estimate of $A_{ij}$ that has \emph{additive error} at most
\begin{equation*}
    \frac{1}{4}\ln(1+O(\delta))=O(\delta).
\end{equation*}
Therefore, by setting $\delta$ appropriately, we obtain a parameter learning algorithm. This is stated as \Cref{alg:parameter} and \Cref{thm:main_paramter}:

\begin{algorithm}[H]
    \caption{$\widehat{A}_i=$ FindParameters$(i,\mathcal{N}(i),\delta,\lambda,\gamma,\kappa,\beta)$}
    \label{alg:parameter}
    \LinesNumbered

    Let $\lambda,\kappa,\gamma$ be as in \Cref{assumption:ising} and \Cref{assumption:mc} and $\mathcal{N}(i)$ denote the set of neighbors of $i$.
    
    Set
    \begin{gather*}
        d=\vert \mathcal{N}(i)\vert\\
        \varepsilon=c_{\mathsf{EST}}\delta\kappa/d\\
        T=\widetilde{O}\left(\frac{\exp(O(\gamma\lambda))2^d\log(1/\beta)}{\delta^4\kappa^8}\right)
    \end{gather*}

    Observe random process $(X_t)_{t=0}^T$ and compute, for all $\bm{x}\in \{-1,1\}^{S_i}$,
    \begin{gather*}
    N_{\bm{x}} = \sum_{t=1}^T Z^{\bm{x}}_{2t}\\
    N_{\bm{x},i} = \sum_{t=1}^T Z^{\bm{x},i}_{2t}.
\end{gather*}

    \For{each $j\in \mathcal{N}(i)$}
        {Find $\bm{y}\in \{-1,1\}^{S_i\setminus \{i,j\}}$ such that for each setting of $x_{i},x_j\in \{-1,1\}$, and defining $\bm{z}=(\bm{y},x_i,x_j)\in \{-1,1\}^{S_i}$,
        \begin{equation*}
            N_{\bm{z}}\geq O\left(\frac{d\log(1/\beta)}{\delta^2\kappa^2}\right)
        \end{equation*}
        
        Estimate rates for each such $\bm{z}$ via 
        \begin{equation*}
            \widehat{p(\bm{z},i)}=\frac{N_{\bm{z},i}}{\varepsilon N_{\bm{z}}}.
        \end{equation*}

        Estimate 
        \begin{equation*}
            \widehat{A}_{ij} = \frac{1}{4}\ln\left(\frac{\widehat{p(\bm{z}^{-1,-1},i)}/\widehat{p(\bm{z}^{+1,-1},i)}}{\widehat{p(\bm{z}^{-1,+1},i)}/\widehat{p(\bm{z}^{+1,+1},i)}}\right).
        \end{equation*}
        }
\end{algorithm}

\begin{theorem}
\label{thm:main_paramter}
Let $i\in [n]$ and suppose that $\delta<c\kappa$. Then with probability at least $1-\beta$, \Cref{alg:parameter} yields estimates $\widehat{A}_{ij}$ for each $j\in \mathcal{N}(i)$ such that $\vert \widehat{A}_{ij}-A_{ij}\vert \leq \delta$.

The runtime of the algorithm is
\begin{equation*}
\widetilde{O}\left(\frac{\exp(O(\gamma\lambda))2^d\log(1/\beta)}{\delta^4\kappa^8}\right).
\end{equation*}

In particular, by setting $\beta\to \beta/n$ and applying this for each $i\in [n]$ with a union bound, we can obtain a $\delta$-additive approximation to $A$ with probability $1-\beta$ in time
\begin{equation*}
    n\cdot \widetilde{O}\left(\frac{\exp(O(\gamma\lambda))2^d\log(n/\beta)}{\delta^4\kappa^8}\right).
\end{equation*}
\end{theorem}

\begin{remark}
\label{rmk:external}
    Note that if each $A_{ij}$ is learned to $\ll 1/d$ additive accuracy, then one can directly also estimate each $h_i$ using a simpler technique via
\begin{equation*}
    \exp\left(2\sum_{k\neq i} A_{ij}\bm{x}+2h_i\right)=\frac{\mathsf{P}_i(\bm{x}^{i\mapsto -1}, \bm{x}^{i\mapsto +1})}{\mathsf{P}_i(\bm{x}^{i\mapsto +1}, \bm{x}^{i\mapsto -1})}.
\end{equation*}
Since we can ensure the absolute error of estimates for the sum on the left-hand side is $\ll 1$, and we have multiplicative estimates of the right-hand side, one can similarly recover each $h_i$. We leave the details to the interested reader. 
\end{remark}

\bibliographystyle{alpha}
\bibliography{bibliography}

\newcommand{\etalchar}[1]{$^{#1}$}
\begin{thebibliography}{MRR{\etalchar{+}}53}

\bibitem[ACKP13]{DBLP:conf/kdd/AbrahaoCKP13}
Bruno~D. Abrahao, Flavio Chierichetti, Robert Kleinberg, and Alessandro Panconesi.
\newblock Trace complexity of network inference.
\newblock In {\em The 19th {ACM} {SIGKDD} International Conference on Knowledge Discovery and Data Mining, {KDD}}, pages 491--499. {ACM}, 2013.

\bibitem[BBK21]{DBLP:conf/focs/Boix-AdseraBK21}
Enric Boix{-}Adser{\`{a}}, Guy Bresler, and Frederic Koehler.
\newblock Chow-{L}iu++: {O}ptimal {P}rediction-{C}entric {L}earning of {T}ree {I}sing {M}odels.
\newblock In {\em 62nd {IEEE} Annual Symposium on Foundations of Computer Science, {FOCS} 2021}, pages 417--426. {IEEE}, 2021.

\bibitem[BDH{\etalchar{+}}08]{DBLP:conf/colt/BartlettDHKRT08}
Peter~L. Bartlett, Varsha Dani, Thomas~P. Hayes, Sham~M. Kakade, Alexander Rakhlin, and Ambuj Tewari.
\newblock High-probability regret bounds for bandit online linear optimization.
\newblock In Rocco~A. Servedio and Tong Zhang, editors, {\em 21st Annual Conference on Learning Theory - {COLT} 2008, Helsinki, Finland, July 9-12, 2008}, pages 335--342. Omnipress, 2008.

\bibitem[BdH16]{bovier2016metastability}
A.~Bovier and F.~den Hollander.
\newblock {\em Metastability: A Potential-Theoretic Approach}.
\newblock Grundlehren der mathematischen Wissenschaften. Springer International Publishing, 2016.

\bibitem[BFH02]{DBLP:journals/iandc/BartlettFH02}
Peter~L. Bartlett, Paul Fischer, and Klaus{-}Uwe H{\"{o}}ffgen.
\newblock Exploiting random walks for learning.
\newblock {\em Inf. Comput.}, 176(2):121--135, 2002.

\bibitem[BGP{\etalchar{+}}23]{DBLP:journals/siamcomp/BhattacharyyaGPTV23}
Arnab Bhattacharyya, Sutanu Gayen, Eric Price, Vincent Y.~F. Tan, and N.~V. Vinodchandran.
\newblock Near-optimal learning of tree-structured distributions by chow and liu.
\newblock {\em {SIAM} J. Comput.}, 52(3):761--793, 2023.

\bibitem[BGS18]{DBLP:journals/tit/BreslerGS18}
Guy Bresler, David Gamarnik, and Devavrat Shah.
\newblock Learning {G}raphical {M}odels from the {G}lauber {D}ynamics.
\newblock {\em {IEEE} Trans. Inf. Theory}, 64(6):4072--4080, 2018.

\bibitem[BK20]{bresler_karzand}
Guy Bresler and Mina Karzand.
\newblock {Learning a tree-structured Ising model in order to make predictions}.
\newblock {\em The Annals of Statistics}, 48(2):713 -- 737, 2020.

\bibitem[BKM19]{DBLP:conf/stoc/BreslerKM19}
Guy Bresler, Frederic Koehler, and Ankur Moitra.
\newblock Learning restricted {B}oltzmann machines via influence maximization.
\newblock In {\em Proceedings of the 51st Annual {ACM} {SIGACT} Symposium on Theory of Computing, {STOC} 2019}, pages 828--839. {ACM}, 2019.

\bibitem[BLMY23]{DBLP:conf/stoc/BakshiLMY23}
Ainesh Bakshi, Allen Liu, Ankur Moitra, and Morris Yau.
\newblock A new approach to learning linear dynamical systems.
\newblock In {\em Proceedings of the 55th Annual {ACM} Symposium on Theory of Computing, {STOC} 2023}, pages 335--348. {ACM}, 2023.

\bibitem[Blu93]{BLUME1993387}
Lawrence~E. Blume.
\newblock The statistical mechanics of strategic interaction.
\newblock {\em Games and Economic Behavior}, 5(3):387--424, 1993.

\bibitem[BMOS05]{DBLP:journals/jcss/BshoutyMOS05}
Nader~H. Bshouty, Elchanan Mossel, Ryan O'Donnell, and Rocco~A. Servedio.
\newblock Learning {DNF} from random walks.
\newblock {\em J. Comput. Syst. Sci.}, 71(3):250--265, 2005.

\bibitem[BMS13]{DBLP:journals/siamcomp/BreslerMS13}
Guy Bresler, Elchanan Mossel, and Allan Sly.
\newblock Reconstruction of {M}arkov {R}andom {F}ields from {S}amples: {S}ome {O}bservations and {A}lgorithms.
\newblock {\em {SIAM} J. Comput.}, 42(2):563--578, 2013.

\bibitem[BMV08]{DBLP:conf/approx/BogdanovMV08}
Andrej Bogdanov, Elchanan Mossel, and Salil~P. Vadhan.
\newblock The {C}omplexity of {D}istinguishing {M}arkov {R}andom {F}ields.
\newblock In {\em Approximation, Randomization and Combinatorial Optimization. Algorithms and Techniques, 11th International Workshop, {APPROX} 2008, and 12th International Workshop, {RANDOM} 2008}, volume 5171 of {\em Lecture Notes in Computer Science}, pages 331--342. Springer, 2008.

\bibitem[Bre15]{DBLP:conf/stoc/Bresler15}
Guy Bresler.
\newblock Efficiently {L}earning {I}sing {M}odels on {A}rbitrary {G}raphs.
\newblock In {\em Proceedings of the Forty-Seventh Annual {ACM} on Symposium on Theory of Computing, {STOC} 2015}, pages 771--782. {ACM}, 2015.

\bibitem[CE22]{DBLP:conf/focs/ChenE22}
Yuansi Chen and Ronen Eldan.
\newblock Localization {S}chemes: {A} {F}ramework for {P}roving {M}ixing {B}ounds for {M}arkov {C}hains.
\newblock In {\em 63rd {IEEE} Annual Symposium on Foundations of Computer Science, {FOCS} 2022}, pages 110--122. {IEEE}, 2022.

\bibitem[CK25]{DBLP:conf/stoc/ChandrasekaranK25}
Gautam Chandrasekaran and Adam~R. Klivans.
\newblock {Learning the Sherrington-Kirkpatrick Model Even at Low Temperature}.
\newblock In Michal Kouck{\'{y}} and Nikhil Bansal, editors, {\em Proceedings of the 57th Annual {ACM} Symposium on Theory of Computing, {STOC} 2025, Prague, Czechia, June 23-27, 2025}, pages 1774--1784. {ACM}, 2025.

\bibitem[CL68]{chow_liu}
Chao-Kong Chow and Chao-Ning Liu.
\newblock Approximating discrete probability distributions with dependence trees.
\newblock {\em IEEE Transactions on Information Theory}, 14(3):462--467, 1968.

\bibitem[DDDK21]{DBLP:conf/stoc/DaganDDK21}
Yuval Dagan, Constantinos Daskalakis, Nishanth Dikkala, and Anthimos~Vardis Kandiros.
\newblock Learning {I}sing models from one or multiple samples.
\newblock In {\em {STOC} '21: 53rd Annual {ACM} {SIGACT} Symposium on Theory of Computing}, pages 161--168. {ACM}, 2021.

\bibitem[DKSS21]{DBLP:conf/colt/DiakonikolasKSS21}
Ilias Diakonikolas, Daniel~M. Kane, Alistair Stewart, and Yuxin Sun.
\newblock {Outlier-Robust Learning of Ising Models Under Dobrushin's Condition}.
\newblock In {\em Conference on Learning Theory, {COLT} 2021}, volume 134 of {\em Proceedings of Machine Learning Research}, pages 1645--1682. {PMLR}, 2021.

\bibitem[DLVM21]{DBLP:conf/icml/DuttLVM21}
Arkopal Dutt, Andrey~Y. Lokhov, Marc Vuffray, and Sidhant Misra.
\newblock Exponential {R}eduction in {S}ample {C}omplexity with {L}earning of {I}sing {M}odel {D}ynamics.
\newblock In {\em Proceedings of the 38th International Conference on Machine Learning, {ICML} 2021}, volume 139 of {\em Proceedings of Machine Learning Research}, pages 2914--2925. {PMLR}, 2021.

\bibitem[DMR20]{devroye2020minimax}
Luc Devroye, Abbas Mehrabian, and Tommy Reddad.
\newblock The minimax learning rates of normal and {I}sing undirected graphical models.
\newblock {\em Electronic Journal of Statistics}, 14:2338--2361, 2020.

\bibitem[EKZ22]{eldan2022spectral}
Ronen Eldan, Frederic Koehler, and Ofer Zeitouni.
\newblock A {S}pectral {C}ondition for {S}pectral {G}ap: {F}ast {M}ixing in {H}igh-{T}emperature {I}sing {M}odels.
\newblock {\em Probability {T}heory and {R}elated {F}ields}, 182(3-4):1035--1051, 2022.

\bibitem[GKK19]{DBLP:conf/colt/GoelKK19}
Surbhi Goel, Daniel~M. Kane, and Adam~R. Klivans.
\newblock Learning {I}sing {M}odels with {I}ndependent {F}ailures.
\newblock In {\em Conference on Learning Theory, {COLT} 2019}, volume~99 of {\em Proceedings of Machine Learning Research}, pages 1449--1469. {PMLR}, 2019.

\bibitem[GKK20]{DBLP:conf/nips/GoelKK20}
Surbhi Goel, Adam~R. Klivans, and Frederic Koehler.
\newblock From {B}oltzmann {M}achines to {N}eural {N}etworks and {B}ack {A}gain.
\newblock In {\em Advances in Neural Information Processing Systems 33: Annual Conference on Neural Information Processing Systems 2020, NeurIPS 2020}, 2020.

\bibitem[Gla63]{glauber1963time}
Roy~J Glauber.
\newblock {Time-dependent statistics of the Ising model}.
\newblock {\em Journal of mathematical physics}, 4(2):294--307, 1963.

\bibitem[GM24]{unified}
Jason Gaitonde and Elchanan Mossel.
\newblock {A Unified Approach to Learning Ising Models: Beyond Independence and Bounded Width}.
\newblock In {\em Proceedings of the 56th Annual {ACM} Symposium on Theory of Computing, {STOC} 2024}, pages 503--514. {ACM}, 2024.

\bibitem[GMM25]{gaitonde2024bypassing}
Jason Gaitonde, Ankur Moitra, and Elchanan Mossel.
\newblock {Bypassing the Noisy Parity Barrier: Learning Higher-Order Markov Random Fields from Dynamics}.
\newblock In {\em Proceedings of the 57th Annual {ACM} Symposium on Theory of Computing, {STOC} 2025}, pages 348--359. {ACM}, 2025.

\bibitem[GS22]{DBLP:conf/stoc/GheissariS22}
Reza Gheissari and Alistair Sinclair.
\newblock {Low-temperature Ising dynamics with random initializations}.
\newblock In {\em {STOC} '22: 54th Annual {ACM} {SIGACT} Symposium on Theory of Computing}, pages 1445--1458. {ACM}, 2022.

\bibitem[GSS25]{gheissari2025rapid}
Reza Gheissari, Allan Sly, and Youngtak Sohn.
\newblock {Rapid phase ordering for Ising and Potts dynamics on random regular graphs}.
\newblock {\em arXiv preprint arXiv:2505.15783}, 2025.

\bibitem[Has70]{10.1093/biomet/57.1.97}
W.~K. Hastings.
\newblock {Monte Carlo sampling methods using Markov chains and their applications}.
\newblock {\em Biometrika}, 57(1):97--109, 04 1970.

\bibitem[HC19]{DBLP:journals/pomacs/HoffmannC19}
Jessica Hoffmann and Constantine Caramanis.
\newblock Learning graphs from noisy epidemic cascades.
\newblock {\em Proc. {ACM} Meas. Anal. Comput. Syst.}, 3(2):40:1--40:34, 2019.

\bibitem[HKM17]{DBLP:conf/nips/HamiltonKM17}
Linus Hamilton, Frederic Koehler, and Ankur Moitra.
\newblock Information {T}heoretic {P}roperties of {M}arkov {R}andom {F}ields, and their {A}lgorithmic {A}pplications.
\newblock In {\em Advances in Neural Information Processing Systems 30: Annual Conference on Neural Information Processing Systems 2017}, pages 2463--2472, 2017.

\bibitem[JLMV24]{jayakumar}
Abhijith Jayakumar, Andrey~Y. Lokhov, Sidhant Misra, and Marc Vuffray.
\newblock Discrete distributions are learnable from metastable samples.
\newblock {\em CoRR}, abs/2410.13800, 2024.

\bibitem[Kal60]{Kalman}
R.~E. Kalman.
\newblock {A New Approach to Linear Filtering and Prediction Problems}.
\newblock {\em Journal of Basic Engineering}, 82(1):35--45, 03 1960.

\bibitem[KDDC23]{DBLP:conf/colt/KandirosDDC23}
Anthimos~Vardis Kandiros, Constantinos Daskalakis, Yuval Dagan, and Davin Choo.
\newblock Learning and {T}esting {L}atent-{T}ree {I}sing {M}odels {E}fficiently.
\newblock In {\em The Thirty Sixth Annual Conference on Learning Theory, {COLT} 2023}, volume 195 of {\em Proceedings of Machine Learning Research}, pages 1666--1729. {PMLR}, 2023.

\bibitem[KHR23]{DBLP:conf/iclr/KoehlerHR23}
Frederic Koehler, Alexander Heckett, and Andrej Risteski.
\newblock Statistical {E}fficiency of {S}core {M}atching: {T}he {V}iew from {I}soperimetry.
\newblock In {\em The Eleventh International Conference on Learning Representations, {ICLR} 2023}, 2023.

\bibitem[KLV24]{klv}
Frederic Koehler, Holden Lee, and Thuy{-}Duong Vuong.
\newblock Efficiently learning and sampling multimodal distributions with data-based initialization.
\newblock {\em CoRR}, abs/2411.09117, 2024.

\bibitem[KM17]{DBLP:conf/focs/KlivansM17}
Adam~R. Klivans and Raghu Meka.
\newblock Learning {G}raphical {M}odels {U}sing {M}ultiplicative {W}eights.
\newblock In {\em 58th {IEEE} Annual Symposium on Foundations of Computer Science, {FOCS} 2017}, pages 343--354. {IEEE} Computer Society, 2017.

\bibitem[KMR93]{kandori}
Michihiro Kandori, George~J. Mailath, and Rafael Rob.
\newblock Learning, mutation, and long run equilibria in games.
\newblock {\em Econometrica}, 61(1):29--56, 1993.

\bibitem[Lez01]{lezaud2001chernoff}
Pascal Lezaud.
\newblock {Chernoff and Berry--Ess{\'e}en inequalities for Markov processes}.
\newblock {\em ESAIM: Probability and Statistics}, 5:183--201, 2001.

\bibitem[LMR{\etalchar{+}}24]{DBLP:conf/focs/LiuMRRW24}
Kuikui Liu, Sidhanth Mohanty, Prasad Raghavendra, Amit Rajaraman, and David~X. Wu.
\newblock {Locally Stationary Distributions: {A} Framework for Analyzing Slow-Mixing Markov Chains}.
\newblock In {\em 65th {IEEE} Annual Symposium on Foundations of Computer Science, {FOCS} 2024}, pages 203--215. {IEEE}, 2024.

\bibitem[LP17]{levin2017markov}
David~A. Levin and Yuval Peres.
\newblock {\em Markov {C}hains and {M}ixing {T}imes}, volume 107.
\newblock American Mathematical Soc., 2017.

\bibitem[MRR{\etalchar{+}}53]{metropolis1953equation}
Nicholas Metropolis, Arianna~W. Rosenbluth, Marshall~N. Rosenbluth, Augusta~H. Teller, and Edward Teller.
\newblock Equation of state calculations by fast computing machines.
\newblock {\em The Journal of Chemical Physics}, 21(6):1087--1092, 1953.

\bibitem[MS09]{DBLP:conf/focs/MontanariS09}
Andrea Montanari and Amin Saberi.
\newblock Convergence to {E}quilibrium in {L}ocal {I}nteraction {G}ames.
\newblock In {\em 50th Annual {IEEE} Symposium on Foundations of Computer Science, {FOCS} 2009}, pages 303--312. {IEEE} Computer Society, 2009.

\bibitem[NS12]{DBLP:conf/sigmetrics/NetrapalliS12}
Praneeth Netrapalli and Sujay Sanghavi.
\newblock Learning the graph of epidemic cascades.
\newblock In {\em {ACM} {SIGMETRICS/PERFORMANCE} Joint International Conference on Measurement and Modeling of Computer Systems, {SIGMETRICS} '12}, pages 211--222. {ACM}, 2012.

\bibitem[PSBR20]{DBLP:conf/nips/PrasadSBR20}
Adarsh Prasad, Vishwak Srinivasan, Sivaraman Balakrishnan, and Pradeep Ravikumar.
\newblock On {L}earning {I}sing {M}odels under {H}uber's {C}ontamination {M}odel.
\newblock In {\em Advances in Neural Information Processing Systems 33: Annual Conference on Neural Information Processing Systems 2020, NeurIPS 2020}, 2020.

\bibitem[RWL10]{ravikumar2010high}
Pradeep Ravikumar, Martin~J Wainwright, and John~D Lafferty.
\newblock High-{D}imensional {I}sing {M}odel {S}election {U}sing $\ell_1$-{R}egularized {L}ogistic {R}egression.
\newblock {\em The Annals of Statistics}, pages 1287--1319, 2010.

\bibitem[SBR19]{DBLP:conf/colt/SimchowitzBR19}
Max Simchowitz, Ross Boczar, and Benjamin Recht.
\newblock Learning linear dynamical systems with semi-parametric least squares.
\newblock In {\em Conference on Learning Theory, {COLT} 2019}, volume~99 of {\em Proceedings of Machine Learning Research}, pages 2714--2802. {PMLR}, 2019.

\bibitem[Sly10]{sly2010computational}
Allan Sly.
\newblock Computational transition at the uniqueness threshold.
\newblock In {\em 2010 IEEE 51st Annual Symposium on Foundations of Computer Science}, pages 287--296. IEEE, 2010.

\bibitem[SS12]{DBLP:conf/focs/SlyS12}
Allan Sly and Nike Sun.
\newblock The {C}omputational {H}ardness of {C}ounting in {T}wo-{S}pin {M}odels on $d$-{R}egular {G}raphs.
\newblock In {\em 53rd Annual {IEEE} Symposium on Foundations of Computer Science, {FOCS} 2012}, pages 361--369. {IEEE} Computer Society, 2012.

\bibitem[SW12]{DBLP:journals/tit/SanthanamW12}
Narayana~P. Santhanam and Martin~J. Wainwright.
\newblock Information-theoretic limits of selecting binary graphical models in high dimensions.
\newblock {\em {IEEE} Trans. Inf. Theory}, 58(7):4117--4134, 2012.

\bibitem[VMLC16]{DBLP:conf/nips/VuffrayMLC16}
Marc Vuffray, Sidhant Misra, Andrey~Y. Lokhov, and Michael Chertkov.
\newblock Interaction {S}creening: {E}fficient and {S}ample-{O}ptimal {L}earning of {I}sing {M}odels.
\newblock In {\em Advances in Neural Information Processing Systems 29: Annual Conference on Neural Information Processing Systems 2016}, pages 2595--2603, 2016.

\bibitem[WSD19]{DBLP:conf/nips/WuSD19}
Shanshan Wu, Sujay Sanghavi, and Alexandros~G. Dimakis.
\newblock Sparse {L}ogistic {R}egression {L}earns {A}ll {D}iscrete {P}airwise {G}raphical {M}odels.
\newblock In {\em Advances in Neural Information Processing Systems 32: Annual Conference on Neural Information Processing Systems 2019, NeurIPS 2019}, pages 8069--8079, 2019.

\bibitem[You11]{young}
H.~Peyton Young.
\newblock The dynamics of social innovation.
\newblock {\em Proceedings of the National Academy of Sciences}, 108:21285--21291, 2011.

\end{thebibliography}

\newpage
\appendix

\section{Auxiliary Tools}
\begin{lemma}
\label{lem:seq_diff}
    Suppose that $x_1,\ldots,x_n$ and $y_1,\ldots,y_n$ are real-valued sequences bounded by $C$ in absolute value. Then
    \begin{equation*}
        \left\vert \prod_{i=1}^n x_i - \prod_{i=1}^n y_i\right\vert\leq C^{m-1}\sum_{i=1}^n \vert x_i-y_i\vert.
    \end{equation*}
\end{lemma}
\begin{proof}
    The proof follows a well-known hybrid argument:
    \begin{align*}
        \prod_{i=1}^n x_i - \prod_{i=1}^n y_i&=\sum_{k=0}^{n-1} \left(\prod_{i=1}^{n-k} x_i\prod_{j=n-k+1}^n y_j - \prod_{i=1}^{n-k-1} x_i\prod_{j=n-k}^n y_j\right)\\
        &=\sum_{k=0}^{n-1} (x_{n-k}- y_{n-k}) \prod_{i=1}^{n-k-1}x_i\prod_{j=n-k+1}^ny_j,
    \end{align*}
    at which point we may take absolute values and apply the triangle inequality, applying the assumption on the absolute values to bound each product.
\end{proof}

\begin{fact}
\label{fact:exist_unique}
   Let $g:[0,\infty)\to [0,1]$ be any continuous function such that $g(0)=0$ and $g$ is strictly increasing on $[0,a]$ and strictly decreasing on $[a,\infty)$ for some $a>0$. Then for any $\xi>1$, there exists a unique solution $z^*>0$ to the equation $g(z^*) = g(\xi\cdot z^*)$. 
\end{fact}

\begin{proof}
    Existence is clear from the fact that for small $z$, $g(\xi\cdot z)>g(z)$, while $g(a)>g(\xi\cdot a)$ by the assumptions on the regions they increase and decrease. The intermediate value theorem then yields a solution. For uniqueness, the same argument shows that any such solution must satisfy $z^*<a$ and $\xi z^*>a$. But if there are two such solutions, say $z'<z^*$, then we have
    \begin{equation*}
        g(z')<g(z^*)=g(\xi\cdot z^*) < g(\xi \cdot z'),
    \end{equation*}
    a contradiction.
\end{proof}

\subsection{Probability Facts}

We will repeatedly appeal to the following basic probability fact:
\begin{lemma}
\label{lem:averaging}
    Let $\mathcal{A},\mathcal{B},\mathcal{E}$ be events and suppose $X$ is a random variable on the same probability space. Let $\mathsf{supp}(X)$ denote the support of $X$ conditioned on the event $\mathcal{E}$. Then
    \begin{gather*}
        \frac{\Pr(\mathcal{A}\vert \mathcal{E})}{\Pr(\mathcal{B}\vert \mathcal{E})}\leq \sup_{\bm{x}\in \mathsf{supp}(X)}\frac{\Pr(\mathcal{A}\vert \mathcal{E},X=\bm{x})}{\Pr(\mathcal{B}\vert \mathcal{E},X=\bm{x})},\\
        \frac{\Pr(\mathcal{A}\vert \mathcal{E})}{\Pr(\mathcal{B}\vert \mathcal{E})}\leq \sup_{\bm{x}\in \mathsf{supp}(X)}\frac{\Pr(\mathcal{A},X=\bm{x}\vert \mathcal{E})}{\Pr(\mathcal{B},X=\bm{x}\vert \mathcal{E})}
    \end{gather*}
\end{lemma}
\begin{proof}
    The result follows by a simple averaging argument by the tower law:
    \begin{equation*}
        \frac{\Pr(\mathcal{A}\vert \mathcal{E})}{\Pr(\mathcal{B}\vert \mathcal{E})}=\frac{\mathbb{E}_X[\Pr(\mathcal{A}\vert \mathcal{E},X)]}{\mathbb{E}_{X}[\Pr(\mathcal{B}\vert \mathcal{E},X)]}\leq \sup_{\bm{x}\in \mathsf{supp}(X)}\frac{\Pr(\mathcal{A}\vert \mathcal{E},X=\bm{x})}{\Pr(\mathcal{B}\vert \mathcal{E},X=\bm{x})}.
    \end{equation*}
    The second inequality is equivalent to the first by Bayes' rule.
\end{proof}

The following pathwise concentration bound can be derived from Freedman's martingale inequality~\cite{DBLP:conf/colt/BartlettDHKRT08}:
\begin{proposition}[Lemma 2 of~\cite{DBLP:conf/colt/BartlettDHKRT08}]
\label{prop:freedman}
    Let $X_1,\ldots,X_T$ be a martingale difference sequence such that $\vert X_i\vert\leq b$ for some $b\geq 1$, and let\footnote{Note that $V$ is random and depends on the realization of the random variables.}
    \begin{equation*}
        V = \sum_{t=1}^T \mathsf{Var}(X_t\vert X_1,\ldots,X_{t-1}).
    \end{equation*}
    Then there is a constant $C>0$ such that for any $\delta<1/e$,
    \begin{equation*}
        \Pr\left(\left\vert\sum_{t=1}^T X_t\right\vert\geq C\min\left\{\sqrt{V},b\sqrt{ \ln(\ln(T)/\delta)}\right\}\cdot \sqrt{ \ln(\ln(T)/\delta)}\right)\leq \delta.
    \end{equation*}
    
\end{proposition}

\section{Site-Consistency of Popular Markov Chains}
In this section, we verify that both the Glauber dynamics and site-homogeneous Metroplis dynamics are site-consistent and suitably stable to apply our learning results.
\subsection{Glauber Dynamics}
\label{sec:app_gd}

Recall from \Cref{def:glauber} that the Glauber dynamics are defined via
\begin{equation*}
    \mathsf{P}_i(\bm{x},\bm{x}^{\oplus i}) = \frac{\pi(\bm{x}^{\oplus i})}{\pi(\bm{x})+\pi(\bm{x}^{\oplus i})}.
\end{equation*}
Therefore, it is immediate to see that
\begin{equation*}
    \mathsf{P}_i(\bm{x}^{i \mapsto -1},\bm{x}^{i \mapsto +1})=\frac{\pi(\bm{x}^{i \mapsto +1})/\pi(\bm{x}^{i \mapsto -1})}{1+\pi(\bm{x}^{i \mapsto +1})/\pi(\bm{x}^{i \mapsto -1})}:=f^{\mathsf{GD}}(\pi(\bm{x}^{i \mapsto +1})/\pi(\bm{x}^{i \mapsto -1})),
\end{equation*}
for the function $f^{\mathsf{GD}}(y) = y/(1+y)$, proving site-consistency. It follows that the associated function $g^{\mathsf{GD}}(y)$ for the product of the rates is given by 
\begin{equation*}
    g^{\mathsf{GD}}(y) = f^2(y)/y = y/(1+y)^2.
\end{equation*}

Under \Cref{assumption:ising}, the transition lower bounds for Glauber dynamics are classical:
\begin{fact}
\label{fact:prob_lb}
Under \Cref{assumption:ising}, given that $i\in [n]$ is chosen for updating at some time $t\geq 0$, it holds for each $\varepsilon\in \{-1,1\}$ and any $\bm{z}\in \{-1,1\}^{n-1}$ that
\begin{equation*}
    \Pr_\pi\left(X_i=\varepsilon\vert X_{-i}=\bm{z}\right)\geq \frac{\exp(-2\lambda)}{2}:=\kappa.
\end{equation*}.
\end{fact}

\begin{fact}
    The Glauber dynamics are $4$-bounded in the sense of \Cref{def:bounded}.
\end{fact}
\begin{proof}
    This is an immediate corollary of Lemma 5.6 of~\cite{unified}.
\end{proof}

Finally, we must check stability:

\begin{lemma}
    For any $\alpha_0>0$ and $\lambda\geq 1$, the transitions of Glauber dynamics are $(\lambda,\alpha_0,\delta_0,\eta)$-stable so long as 
    \begin{equation*}
        \delta_0=c\min\{\alpha_0^2,1\}\exp(-O(\lambda))
    \end{equation*}
    with the function
    \begin{equation*}
        \eta(\delta) = C\max\{1/\alpha_0^2,1\}\cdot \delta.
    \end{equation*}
\end{lemma}
\begin{proof}
    We prove this in a series of claims. For convenience, we state the derivative:
    \begin{equation*}
        g'(z) = \frac{1-z}{(1+z)^3}
    \end{equation*}
    \begin{claim}
        For any $\alpha>0$, the solution $z^*(\alpha)>0$ is given by $z^*=\exp(-\alpha/2)$.
    \end{claim}
    \begin{proof}
    We simply solve the equation $g(z)=g(\exp(\alpha)z)$. By definition, this is 
    \begin{equation*}
        \frac{z}{(1+z)^2}=\frac{\exp(\alpha)z}{(1+\exp(\alpha)z)^2}\iff \exp(\alpha)z^2=1,
    \end{equation*}
    where we use simply algebra to conclude.
    \end{proof}

    We will now show that if $\delta_0$ is small enough as a function of $\alpha_0$ and $\lambda$, then any approximate solution must lie in a narrow band around $z^*$. This band will have strictly positive derivative, so we will be able to argue that approximate solutions are nearby. We do so by showing that this choice of $\delta_0$ rules out all other intervals with a tedious, but careful calculus argument.

    \begin{claim}
        If $\delta_0$ is as stated, then there does not exist $z\in [1,\exp(2\lambda)]$ satisfying the approximate inequality
        \begin{equation*}
            \vert g(z)-g(\exp(\alpha)z)\vert\leq \delta_0.
        \end{equation*}
    \end{claim}
    \begin{proof}
        Note that the derivative is nonpositive on $z\geq 1$, and so
        \begin{equation*}
            g(z)-g(\exp(\alpha)z)\geq \int_z^{\exp(\alpha)z}\frac{s-1}{(1+s)^3}\mathrm{d}s.
        \end{equation*}
        Note that $\exp(\alpha)z\geq z+\alpha z$, and therefore this interval of integration contains the set $[z+\alpha/2,z+\alpha]$. By monotonicity, it follows that 
        \begin{equation*}
            g(z)-g(\exp(\alpha)z)\geq \frac{\alpha}{2}\frac{\alpha}{2\exp(O(\lambda))} \geq  c\alpha_0^2\exp(-O(\lambda)),
        \end{equation*}
        where we use our upper bound on $I$ and the lower bound on the integrand in this region. This exceeds $\delta_0$, so there cannot be any such approximate solutions on this region.
    \end{proof}

    \begin{claim}
    If $\delta_0$ is as stated, then there does not exist $z\in [\exp(-2\lambda),\exp(-\alpha)]$ satisfying the approximate inequality
        \begin{equation*}
            \vert g(z)-g(\exp(\alpha)z)\vert\leq \delta_0.
        \end{equation*}
\end{claim}
\begin{proof}
    In this case, we can directly calculate using positivity of the interval that
    \begin{equation*}
        g(\exp(\alpha)z)-g(z)=\int_z^{\exp(\alpha)z}\frac{1-s}{(1+s)^3}\mathrm{d}s\geq c\alpha_0^2z\geq c\alpha_0^2\exp(-O(\lambda)).
    \end{equation*}
    Here, we use a similar argument that the region of integration contains an interval of size $c\alpha_0z$ where the derivative is at least $\alpha_0$, and then lower bounding using the interval size.
\end{proof}

\begin{claim}
\label{claim:deriv_const}
    If $\delta_0$ is as stated, then there does not exist $z\in [\exp(-\alpha),z^*-C\max\{1/\alpha_0,1\}\delta]\cup [z^*+C\max\{1/\alpha_0,1\}\delta,1]$ satisfying the approximate inequality
        \begin{equation*}
            \vert g(z)-g(\exp(\alpha)z)\vert\leq \delta,.
        \end{equation*}
\end{claim}
\begin{proof}
    First, observe that 
    \begin{equation*}
        1-z^*(\alpha) = 1-\exp(-\alpha/2)\geq c\min\{\alpha_0,1\},
    \end{equation*}
    for a sufficiently small constant $c>0$. It follows that on the region $R=[z-c'\min\{\alpha_0,1\},z^*+c'\min\{\alpha_0,1\}]$, $g'(z)$ is strictly positive and lower bounded by $c''\min\{\alpha_0,1\}$, while being nonnegative in all of $[0,1]$. Since $C\max\{1/\alpha_0,1\}\delta$ is contained in the radius of this interval by the choice of $\delta_0$, It follows that for any $z\in [\exp(-\alpha),z^*-C\max\{1/\alpha_0,1\}\delta]$
    \begin{equation*}
        g(z)\leq g(z^*) - c'\min\{\alpha_0^2,1\}(C\max\{1/\alpha_0,1\}\delta)=g(\exp(\alpha)z^*)-\delta\leq g(\exp(\alpha)z)-\delta,
    \end{equation*}
    since on this region, $y\mapsto g(\exp(\alpha)y)$ is decreasing and if we chose $C$ large enough. Therefore, on the first part of this interval, there can be no solutions if $\delta_0$ is taken this small.
A symmetric argument holds for the other interval.
\end{proof}

Putting these three claims together proves stability as stated.\qedhere
\end{proof}

\subsection{Metropolis Dynamics}
\label{sec:app_metro}
We can prove that these conditions are similarly satisfied for the Metropolis dynamics given by \Cref{def:metropolis}. Site-consistency is immediate to see from \Cref{def:metropolis}, and moreover (dropping the index),
\begin{align*}
    \mathsf{P}(\bm{x}^{-1},\bm{x}^{+1})\mathsf{P}_i(\bm{x}^{+1},\bm{x}^{-1}) &= r_{+}r_-\min\left\{\frac{r_{+}\pi(\bm{x}^{-1})}{r_{-}\pi(\bm{x}^{+1})},\frac{r_{-}\pi(\bm{x}^{+1})}{r_{+}\pi(\bm{x}^{-1})}\right\}\\
    &:= g^{\mathsf{MD}}\left(\frac{\pi(\bm{x}^{+1})}{\pi(\bm{x}^{-1})}\right)
\end{align*}
for the function
\begin{equation*}
    g^{\mathsf{MD}}(z) = r_+r_-\min\left\{\frac{r_-z}{r_+},\frac{r_+}{r_-z}\right\}.
\end{equation*}

It is also immediate to see that under \Cref{assumption:ising},
\begin{equation*}
\mathsf{P}_i(\bm{x},\bm{x}^{\oplus i})\geq \min\{r_-,r_+\}\exp(-2\lambda):=\kappa.
\end{equation*}
To see this, simply note that the probability ratio in the definition is at least $\exp(-2\lambda)$, and a simple case analysis completes the claim.

Finally, it is easy to see that the Metropolis chain is also $4$-bounded since for any $z\geq z'$ and any $a,b>0$,
\begin{equation*}
    \frac{a\min\{bz,1\}}{a\min\{bz',1\}}\leq \frac{z}{z'}
\end{equation*}
by a simple case analysis. Since we can assume in \Cref{def:bounded} that the numerator exceeds the denominator without loss of generality, the claim follows from taking
\begin{equation*}
    z = \pi(\bm{x}^{i\mapsto \sigma})/\pi(\bm{x}), z' = \pi(\bm{y}^{i\mapsto \sigma})/\pi(\bm{y}),
\end{equation*}
and using reversibility to observe that all terms in the exponential cancel outside of the set $S$ such that $\bm{y}_k\neq \bm{x}_k$, which leaves $4\sum_{k\in S} \pm A_{ik}$. In particular,
\begin{equation*}
    \frac{z}{z'}\leq \exp\left(4\sum_{k\in S} \vert A_{ik}\vert\right).
\end{equation*}

\begin{lemma}
    For any $\alpha_0>0$ and $\lambda\geq 1$, the transitions of the Metropolis chain are $(\lambda,\alpha_0,\delta_0,\eta)$-stable so long as 
    \begin{equation*}
        \delta_0=c\min\{\alpha_0,1\}\exp(-O(\lambda))\min\{r_+^2,r_-^2\}
    \end{equation*}
    with the function
    \begin{equation*}
        \eta(\delta) = \frac{\delta}{pq},
    \end{equation*}
    where
    \begin{equation*}
    p:=r_+r_-, q=\frac{r_-}{r_+}.
\end{equation*}
\end{lemma}\begin{proof}
We carry out a similar plan to that of the Glauber dynamics. Observe that in the notation of the lemma statement, $g(z) = p\min\{qz,1/qz\}$.

\begin{claim}
    For any $\alpha>0$, the solution to $g(z^*)=g(\exp(\alpha)z^*)$ is given by $z^* = 1/q\exp(\alpha/2)$.
\end{claim}
\begin{proof}
    It is clear that since $\alpha>0$, the identity of the minimizer must change so
    \begin{equation*}
        qz^* = \frac{1}{qz^*\exp(\alpha)}.
    \end{equation*}
    Rearranging gives the claim.
\end{proof}

We now rule out various intervals as before:
\begin{claim}
        If $\delta_0$ is as stated, then there does not exist $z\in [\frac{1}{q},\exp(2\lambda)]$ satisfying the approximate inequality
        \begin{equation*}
            \vert g(z)-g(\exp(\alpha)z)\vert\leq \delta_0.
        \end{equation*}
    \end{claim}
    \begin{proof}
        On this interval, clearly both minimizers are the inverse terms, and
        \begin{equation*}
        \frac{p}{qz}-\frac{p}{qz\exp(\alpha)}=\frac{p}{qz}\left(1-\exp(-\alpha)\right)\geq \frac{cp\exp(-O(\lambda))\min\{\alpha,1\}}{q},
        \end{equation*}
        where we use the same exponential inequality as before.
    \end{proof}

\begin{claim}
        If $\delta_0$ is as stated, then there does not exist $z\in [\exp(-2\lambda),\frac{1}{q\exp(\alpha)}]$ satisfying the approximate inequality
        \begin{equation*}
            \vert g(z)-g(\exp(\alpha)z)\vert\leq \delta_0.
        \end{equation*}
    \end{claim}
    \begin{proof}
        On this interval, clearly both minimizers are the linear terms, and
        \begin{equation*}
        pq\exp(\alpha)z-pqz = pqz(\exp(\alpha)-1)\geq pq\alpha\exp(-O(\lambda)).
        \end{equation*}
        where we use the same exponential inequality as before.
    \end{proof}

    We now turn to the main interval as before:

    \begin{claim}
    If $\delta_0$ is as stated, then there does not exist $z\in [\frac{1}{q\exp(\alpha)},z^*-\frac{\delta}{pq}]\cup [z^*+\frac{\delta}{pq},1/q]$ satisfying the approximate inequality
        \begin{equation*}
            \vert g(z)-g(\exp(\alpha)z)\vert\leq \delta.
        \end{equation*}
\end{claim}
\begin{proof}
    Note that on the region $[\frac{1}{q\exp(\alpha)},1/q]$, the derivative of the function in $z$ is constant and simply given by $pq$, and the length of this interval is at least
    \begin{equation*}
        \frac{1}{q}(1-\exp(-\alpha))\geq \frac{c\min\{\alpha,1\}}{q}.
    \end{equation*}

    Therefore, so long as 
    \begin{equation*}
        \delta\leq pq\cdot \frac{c'\min\{\alpha,1\}}{q}=c'p\min\{\alpha,1\},
    \end{equation*}
    the interval with radius $\frac{\delta}{pq}$ about $z^*$ will lie in this interval. As in \Cref{claim:deriv_const}, any point $z\in  [\frac{1}{q\exp(\alpha)},z^*-\frac{\delta}{pq}]$ will have $g(z)\leq g(z^*)-\delta\leq g(\exp(\alpha)z^*)-\delta\leq g(\exp(\alpha)z)-\delta$, and similarly for any point $z$ above this interval.
\end{proof}

Replacing $p$ and $q$ with their actual values yields the claim.
\end{proof}

%%%%%%%%%%%%%%%%%%%%%%%%%%%%%%%%%%%%%%%%%%%%%%%%%%%%%%%%%%%%

\end{document}